\documentclass[moor,nonblindrev]{informs4}

\OneAndAHalfSpacedXI

\usepackage{amsmath,amssymb,amsfonts,mathrsfs}
\usepackage{bm,nicematrix}
\usepackage{algorithm}
\usepackage{algorithmic}
\usepackage{graphicx}
\usepackage{booktabs}
\usepackage{theorem}

\usepackage{natbib}
 \bibpunct[, ]{(}{)}{,}{a}{}{,}%
 \def\bibfont{\small}%

\usepackage{rotating}
\usepackage{fancyvrb}

\usepackage[colorlinks=true,breaklinks=true,bookmarks=true,urlcolor=blue,citecolor=blue,linkcolor=blue,bookmarksopen=false,draft=false]{hyperref}

\TheoremsNumberedThrough     
\ECRepeatTheorems

\JOURNAL{MOOR}

\EquationsNumberedThrough    

\begin{document}


\RUNTITLE{LDCSM}

\TITLE{Series Expansion of Probability of Correct Selection for Improved Finite Budget Allocation in Ranking and Selection}

\ARTICLEAUTHORS{%
\AUTHOR{Xinbo Shi, Yijie Peng}

\AFF{PKU-Wuhan Institute for Artificial Intelligence, Guanghua School of Management, Peking University; Xiangjiang Laboratory, \{\EMAIL{xshi@stu.pku.edu.cn}, \EMAIL{pengyijie@pku.edu.cn}\}}

\AUTHOR{Bruno Tuffin}

\AFF{Inria, Univ Rennes, CNRS, IRISA, Rennes, Campus de Beaulieu, 35042 RENNES Cedex, France, \EMAIL{bruno.tuffin@inria.fr}}


}

\ABSTRACT{%
This paper addresses the challenge of improving finite sample performance in Ranking and Selection by developing a Bahadur-Rao type expansion for the Probability of Correct Selection (PCS). While traditional large deviations approximations captures PCS behavior in the asymptotic regime, they can lack precision in finite sample settings. Our approach enhances PCS approximation under limited simulation budgets, providing more accurate characterization of optimal sampling ratios and optimality conditions dependent of budgets. Algorithmically, we propose a novel finite budget allocation (FCBA) policy, which sequentially estimates the optimality conditions and accordingly balances the sampling ratios. We illustrate numerically on toy examples that our FCBA policy achieves superior PCS performance compared to tested traditional methods. As an extension, we note that the non-monotonic PCS behavior described in the literature for low-confidence scenarios can be attributed to the negligence of simultaneous incorrect binary comparisons in PCS approximations. We provide a refined expansion and a tailored allocation strategy to handle low-confidence scenarios, addressing the non-monotonicity issue.}


\maketitle

\section{Introduction}\label{sec:intro}
This paper considers the problem of identifying the best design among a finite set of alternatives via Monte Carlo simulation, a challenge known as Ranking and Selection (R\&S) in simulation. The classic R\&S problem can be viewed in the framework of Ordinal Optimization \citep{ho1992ordinal}, where the primary goal is to determine the ordinal ranking of alternative performances rather than precisely estimate their output distributions. Given the common practice to select the alternative with the greatest sample mean as the best alternative, the main focus of R\&S research is on allocating a limited simulation budget to provide the best possible performance. Throughout this paper, we will assume a fixed-budget setting where the total number of available simulation samples for all alternatives is predetermined. The performance of an allocation rule will be evaluated using the well-known probability of correct selection (PCS) when the budget is exhausted.

There is an extensive body of allocation algorithms developed in the literature, many of which have proven to be empirically and/or theoretically successful given a sufficient number of simulation samples. Among these, the optimal computing budget allocation (OCBA), introduced by \cite{chen1995effective} and \cite{chen2000discrete}, is one of the most widely adopted algorithms. \cite{chen2000discrete} define the allocation decision as the sampling ratio, the proportion of simulation samples allocated to each alternative relative to the total number of samples. They derive an approximation of the PCS as a function of the sampling ratios and characterize an optimal allocation that maximizes this approximation. Recently, OCBA has been shown to achieve the optimal allocation \citep{li2023convergence} as the sample size increases. Using a large deviations principle (LDP), \cite{glynn2004large} prove that PCS converges to one at an exponential rate depending on the sampling ratios. Consequently, they define a rate optimal allocation (ROA) using a set of equality conditions characterizing the optimal sampling ratios that maximize the large deviations rate. The ROA has also been shown to asymptotically satisfy the corresponding optimal conditions \citep{li2023convergence}. The optimal conditions of \cite{glynn2004large} reduce to the OCBA formula \cite{chen2000discrete}, assuming that the sampling ratio for the best alternative far exceeds others. The large deviations rate serves as an asymptotic performance metric for allocation algorithms. For instance, the top-two Thompson sampling (TTTS) introduced by \cite{russo2020simple}, which is not derived explicitly from the large deviations rate, has been proven to asymptotically attain the optimal rate.

In recent years, there has been growing research interest in the special case of a \textit{limited budget} constraint. This interest stems from the increasing demand for applications, such as on time optimization of complex simulation systems \citep{boschert2016digital}. However, the finite sample behavior of PCS can be essentially different from its asymptotic behavior. For example, the large deviations rate for PCS, which takes the form of the minimum large deviations rate for pairwise comparisons between the best alternative and sub-optimal alternatives, cannot capture the impact of the number of sub-optimal alternatives on PCS. Moreover, the large deviations rate has nothing to do with the sample size. Consequently, the classic large deviations rate is not capable of approximating the PCS accurately and does not suffice to guide the development of efficient allocation algorithms in the finite sample case. Furthermore, \cite{peng2015nonmonotonicity} describe a low-confidence scenario where the total budget is much smaller than what is sufficient, and reveal a counter-intuitive phenomenon therein: PCS may not be monotonically increasing as simulation samples accumulate, which does not often appear when the simulation budget is sufficiently large. Although they provide a qualitative characterization of low-confidence scenarios, there have been no quantitative definitions, which suggests the necessity of further investigation into such scenarios.

In this work, we propose an asymptotically equivalent Bahadur-Rao type expansion to PCS inspired by \citet{bahadur1960deviations}, extending beyond the large deviations characterization in \cite{glynn2004large}, which can be leveraged to enhance finite sample performance of ordinal optimization. Roughly speaking, the LDP theory of \cite{glynn2004large} is established based merely on the assumption that the underlying sampling distributions have finite cumulant generating functions \citep[CGF, see][]{gartner1977large}, i.e., these distributions are light-tailed. Our approximation takes advantage of the additional assumption that the family of sampling distributions have bounded total variations (BTV), so that the Edgeworth expansion \citep{bhattacharya1978validity} of their probability density functions can be approximated by closed form functions with uniform error bounds. We show that under the additional BTV assumption, the large deviations theory developed by \cite{glynn2004large} can be deduced from our results. In other words, the LDP and the proposed approximation agree on the asymptotic behavior of PCS, which substantiates the validity of our results.

To demonstrate how the proposed approximation better captures the finite sample behavior of PCS, we notice that the large deviations rate, characterizing the exponential rate at which PCS converges, serves only as an asymptotic surrogate of PCS independent of the simulation budget. Consequently, it may not accurately represent the efficacy of a sampling ratio when the number of samples is limited. In contrast, our proposed approximation takes the form of a summation of products of an exponential function in $T$ and a power series in $T^{-1/2}$, where $T$ denotes the simulation budget. The power series can be truncated to any order while maintaining the asymptotic equivalence. Numerical experiments demonstrate that even the zeroth-order approximation is sufficiently accurate to improve the finite sample performance of sampling algorithms, provided the simulation budget is fairly large. Furthermore, for a fixed budget, our approximation can be made arbitrarily accurate by retaining additional terms in the series, allowing for further improvement of sampling algorithms even with smaller budgets. As extensions, we demonstrate two modifications of the Bahadur-Rao type expansion to provide further accurate approximations in low-confidence scenarios, and to characterize a conditional PCS which may serve as a foundation for future developments of sequential algorithms, respectively.

We explicitly derive the proposed approximation as a function of the sampling ratios for Gaussian sampling distributions. The approximation function is shown to be asymptotically concave, indicating that it is concave except within a shrinking sequence of subsets of feasible sampling ratios near the boundary, whose interiors converge to the empty set, as the order of the approximation approaches infinity. Moreover, the approximation functions of even orders are exactly concave. These findings justify our approximate characterization of the optimal sampling ratios using Karush-Kuhn-Tucker conditions. Similar to the findings that guide static R\&S algorithms \citep{chen2000discrete,gao2017a}, the approximate optimal conditions consist of two set of equations, with the first set of equations balancing the exploration of sub-optimal alternatives, and the second set balancing the exploration-exploitation tradeoff. The concavity of the proposed approximation enables us to design computationally efficient sampling algorithms, which are named after finite computation budget allocation (FCBA).

To the best of our knowledge, the literature contains two main categories of approaches for improving finite sample performance of R\&S. The first category formulates the sampling-allocation decision as a dynamic policy rather than the sampling ratios. Taking the parameter uncertainty into account, these approaches typically involve maintaining a Bayesian model for the sampling distributions and developing fully dynamic sampling policies based on the posterior distribution. For example, \cite{chick2001new} and \cite{chick2010sequential} develop an expected value of information (EVI) policy which focuses on allocating samples to maximize the EVI in a single additional stage of sampling. Two well-known sampling policies, knowledge gradient \citep[KG,][]{gupta1996bayesian,frazier2008knowledge} and expected improvement \citep[EI,][]{ryzhov2016convergence}, which solve one-step optimizations, share the similar spirit. Although KG and EI can estimate the best alternative consistently, their sampling ratios do not asymptotically attain the optimal sampling ratio defined by \cite{chen2000discrete} or that by \cite{glynn2004large}, implying that KG and EI are not asymptotically efficient \citep{peng2017myopic}. \cite{peng2016dynamic} unify dynamic sampling policies in a general dynamic framework. Furthermore, \cite{peng2018ranking} characterize the optimal dynamic policy via the corresponding Bellman equations and develop an asymptotically optimal allocation policy (AOAP). Although these dynamic policies are empirically shown to be superior to static algorithms in finite sample scenarios, they come with a high computational cost. Moreover, there is no theoretical guarantee for their finite sample performance. The second category incorporates exogenous information into the formulation. \cite{peng2019efficient} exploits information from multifidelity models to enhance the performance of R\&S. Recently, there has been a growing interest in context-dependent R\&S, which supports decision making by utilizing contextual information in simulation systems \citep{cakmak2021contextual,shi2023toptwo,li2024efficient,du2024contextual}. In this work, we contribute to the literature by initiating a new perspective for R\&S with finite budget constraint. We examine the finite sample behavior of PCS, thus enabling the sampling efficiency whenever the budget is small or large, and offering a theoretical guarantee for the performance of the proposed algorithm. Our framework exemplifies the possibility of improving finite sample performance in a computationally efficient manner without requesting exogenous information of the simulation alternatives.

Our work is also relevant to the literature on probability bound analysis. Concentration inequalities, such as Hoeffding's inequality and Chernoff's bound, have been widely used to bound the tail probability of sample mean statistics. See \cite{vershynin2018high} for a thorough discussion. However, calculating sub-Gaussian or sub-exponential norms of sampling distributions, which these probability bounds typically require, is difficult. Additionally, these bounds are not tight for small sample sizes. For Gaussian sampling distributions, the proposed approximation is a $\frac{1}{2}$-higher-order infinitesimal amount than the Hoeffding's bound. In the context of R\&S, the probability of incorrect selection (PICS) can be expressed as the probability of pairwise incorrect selection between a sub-optimal alternative and the best alternative. This is typically approximated using the upper bound from the first-order Bonferroni inequality \citep{bonferroni1936teoria}, followed by applying the LDP to the bound. A few exceptions, such as the works of \cite{shen2021covariate} and \cite{peng2018gradient}, use Slepian's inequality, which is only applicable for Gaussian alternatives \citep{slepian1962oneside}, to derive bounds for the PCS. The main drawback is that these inequalities neglects the probability of simultaneous incorrect binary comparisons (SIBC). We demonstrate in this paper that the misconception in the literature that the best alternative deserves more sampling ratios than sub-optimal alternatives might stem from using the first-order Bonferroni inequality. \cite{peng2015nonmonotonicity} find that the non-monotonicity of PCS is attributed to the \textit{induced correlation}, which could be interpreted as the intensity at which the event of SIBC happens. The induced correlation or the event of SIBC is neglected when leveraging the Bonferroni inequality to bound PCS, thus leading to a poor performance of consequent algorithms. As an extension of the proposed methodology, we show that by using the inclusion-exclusion principle followed by deriving an extended Bahadur-Rao type expansion for the probability of SIBC, a further accurate approximation to the PCS for low-confidence scenarios is available. This finer approximation takes the form of a discontinuous function in the sampling ratios, and thereby explains the non-monotonicity of PCS. Specifically, this discontinuity can be precisely characterized by a \textit{critical point} which is defined by the minimizer of a large deviations rate function. The critical point intuitively quantifies whether a probability of SIBC for multiple sub-optimal alternatives is negligible and is utilized to design allocation policies for low-confidence scenarios.

The manuscript is structured as follows. Section \ref{sec:main} introduces the main contributions and key notations. Section \ref{sec:methodology} presents the novel approximation technique. In Section \ref{sec:normalpolicy}, we propose a new static procedure for Gaussian alternatives and analyze its asymptotic properties. Section \ref{sec:experiment} empirically demonstrates the superiority of our procedure in finite sample cases. Section \ref{sec:extensions} discusses extensions of our algorithm and interprets their superiority. Finally, Section \ref{sec:conclusion} provides concluding remarks. Rigorous proofs of theorems will be presented in the online supplement. A brief summary of the main theoretical results in Section \ref{sec:main} without proofs was accepted by the \textit{Proceedings of the 2024 Winter Simulation Conference} \citep{shi2024finite}. The full theory with detailed proofs, and the extensions to the low-confidence scenarios and conditional probabilities, are presented in the paper for the first time.

\section{Setting and Main Results}\label{sec:main}
We first introduce general notations in Section \ref{sec:notation}. In Section \ref{sec:model}, we provide the problem definition for R\&S and recall the classic large deviations principle for PCS. We then present our main result, i.e., the Bahadur-Rao type expansion for PCS in Section \ref{sec:expansion}.

\subsection{Notations.}\label{sec:notation} For sequences $\{a_{n}\}_{n=1}^{\infty}$ and $\{b_{n}\}_{n=1}^{\infty}$, we denote $a_{n} = o(b_{n})$ if there exists a diminishing sequence $r_{n}\rightarrow 0$ as $n\rightarrow \infty$ such that $a_{n} = r_{n}b_{n}$, and denote $a_{n} = \mathcal{O}(b_{n})$ if there exists a constant $C$ independent of $n$ such that $\vert a_{n}\vert \leq C\vert b_{n}\vert$, $\forall~ n\geq 1$. For positive integer $k$, let $[k]:=\{1,2,\dots,k\}$ be the set of all alternatives, and let $[k]_{-}$ denote the set $[k]\backslash\{1\}$. For positive integer $n$, $n!!$ denotes the double factorial of $n$, i.e., $n!!:= \prod_{i=0}^{\lceil n/2\rceil - 1}(n-2i)$ and $\lceil n\rceil$ denotes the ceiling function of $n$. We adhere to the convention that $0!! = 1$ and $\sum_{l=1}^{0}{\cdot} = 0$. For $a, b\in\mathbb{R}$, let $a\wedge b := \max\{a,b\}$ and $a\vee b = \min\{a, b\}$. Additionally, let $F_{X}(x)$ denote the cumulative distribution function (CDF) with regard to random variable $X$. In this paper, bold mathematical symbols (e.g., the sampling ratios $\bm{p}$) denote vectors by default, unless explicitly stated otherwise. Finally, let $\operatorname{int}(\cdot)$ denote the interior of a set, and $\mathscr{D}(\cdot)$ and $\mathscr{R}(\cdot)$ denote the domain and range of a function, respectively.

\subsection{Ranking and Selection.}\label{sec:model}
Consider $X_i$ for $i\in[k]$ as the random output of $k$ alternatives. We aim to identify the alternative, from now called the best alternative, with the largest mean output $m_{i}=\mathbb{E}(X_{i})$. Consistent with the common practice in the literature, we will assume that the best alternative is unique and thus well-defined throughout the paper. For simplicity, we assume that $m_{1}>m_{2}\geq \dots\geq m_{k}$.

Let $T$ denote the fixed number of simulation replications. Let $T_{i}$ be the number of replications for the $i$-th alternative, and $p_{i} = T_{i}/T$ indicate the sampling ratio, i.e., the fraction of replications attributed to the $i$-th alternative. We will ignore the minor issue that $T_{i} = p_{i}T$ may not necessarily be an integer for all $p_{i}\in[0, 1]$. Let $X_{i}^{(t)}$ denote the $t$-th simulation sample drawn from alternative $i$ and let $\bar{X}_{i}(T_{i}) := T_{i}^{-1}\sum_{t=1}^{T_{i}}X_{i}^{(t)}$ be the sample mean of $T_{i}$ simulations from alternative $i$. The event of correct selection occurs when the sample mean of the best alternative ranks first. Thus the PCS is defined as 
$$P\{CS\} = P\left(\bar{X}_{1}(T_{1}) > \max_{j\in[k]_{-}}\bar{X}_{j}(T_{j})\right) = 1 - P\left(\bar{X}_{1}(T_{1}) \leq \max_{j\in[k]_{-}}\bar{X}_{j}(T_{j})\right).$$

\cite{glynn2004large} establish the LDP for PCS given light-tailed sampling distributions. Let $\Lambda_{j}(\lambda):=\log\mathbb{E}\exp\{\lambda X_{j}\}$ be the CGF of $X_{j}$ for $j\in[k]$. We make the following assumption of CGFs.
\begin{assumption}[Light-tailed distribution]\label{ass:light}
    The domain of $\Lambda_{i}(\cdot)$ contains a non-empty interior including the origin, i.e., $0\in\operatorname{int}(\mathscr{D}(\Lambda_{i}))$. Moreover, assume $[m_{k}, m_{1}] \subseteq \operatorname{int}(\mathscr{R}(\Lambda_{i}'))$, where $\Lambda_{i}'$ denotes the derivative function of $\Lambda_{i}$.
\end{assumption}

Since CGFs are smooth in the interior of their domain, a derivative of any order for $\Lambda_{i}$ is well-defined. Moreover, it turns out that any moment of $X_{i}$ is finite. The second statement rules out the trivial case where certain alternatives are deterministically inferior, i.e., their distributions’ supports are strictly bounded above by $m_{1}$, or vice versa. Assumption \ref{ass:light} ensures that the large deviations rate functions are well-defined. It can be justified for widely used distributions such as Normal, Geometric, Poisson and Gamma distributions. We recall their results below to facilitate comparisons.

\vspace{.4em}
\indent {\TheoremHeaderFont {\normalsize T}{\footnotesize HEOREM}}\ (\textbf{LDP, \citealp{glynn2004large}}).\quad \textit{Under Assumption \ref{ass:light}, there exist bivariate functions $G_{j}(\cdot, \cdot)$, $j\in[k]_{-}$, satisfying
$$\lim_{T\rightarrow \infty}-\frac{1}{T}\log(1 - P\{CS\}) = \min_{j\in[k]_{-}}G_{j}(p_{1}, p_{j}).$$
The functions $G_{j}(\cdot, \cdot)$ thus are referred to as rate functions.}
\newpage

The rate functions can be explicitly expressed, in the form of the infimum of Fenchel-Legendre transformations of $\bar{X}_{1}(T_{1}) - \bar{X}_{j}(T_{j})$, using the Gartner-Ellis Theorem \citep{dembo2009large}. The closed form can be derived for common distribution families or estimated in a data-driven approach \citep{chen2023datadriven}. However, the validity of the LDP rate does not imply that this seemingly correct asymptotic equivalence representation, i.e., $1 - P\{CS\} = \exp\{-T\cdot \min_{j\in [k]_{-}}G_{j}(p_{1}, p_{j})\}\cdot (1+o(1))$, holds true. To be specific, assume the following functional form of PCS
\begin{equation}\label{eq:asympeq}
    1 - P\{CS\} = a(T)\cdot \exp\{-T\cdot \min_{j\in [k]_{-}}G_{j}(p_{1}, p_{j})\}\cdot (1+o(1)).
\end{equation}
Regardless of the form of $a(T)$, as long as it is a function that decays slower than exponential functions asymptotically, i.e., $-\frac{1}{T}\log{a(T)} \rightarrow 0$, the large deviations principle may accommodate any functional form of PCS as delineated in Equation (\ref{eq:asympeq}).

\subsection{The Bahadur-Rao Type Expansion.}\label{sec:expansion}
In studying statistical inferences, for example, hypothesis testings where the goal is to compare a sample mean statistic and a real, \cite{bahadur1960deviations} proves the validity of Equation (\ref{eq:asympeq}) and shows that $a(T)$ can be written as a polynomial in $T^{-1/2}$ with coefficients explicitly expressed by moments of the underlying sampling distribution. This motivates us to find out an asymptotic equivalent representation for PCS as well. However, PCS involves comparing several sample mean statistics simultaneously and needs further treatment. We assume the following regularity condition in addition.
\begin{assumption}\label{ass:btv}
    The density functions of $X_{i}$'s exist for $i\in[k]$, and have bounded total variations.
\end{assumption}

Intuitively, Assumption \ref{ass:btv} means that the sampling distributions of alternatives are flat. A function $f:\mathbb{R}\rightarrow\mathbb{R}$ is said to have bounded total variation (BTV) if $\sup\sum_{i=1}^{n}\vert f(x_{i}) - f(x_{i-1})\vert < \infty$ where the supremum is taken over all possible $x_{0}<x_{1}<\cdots<x_{n}$. The BTV assumption ensures that the error introduced by approximating the distribution function of $X_{i}$ using the Edgeworth expansion is controllable \citep[see][]{cramer1970error}. It can be easily checked that Assumption \ref{ass:btv} holds for any unimodal distribution with bounded density functions. Our main result is as follows.
\begin{theorem}[Main]\label{thm:main}
    Under Assumptions \ref{ass:light} and \ref{ass:btv}, for any $\ell\in\mathbb{N}$ and $p_{i}>0$, $\forall  i\in[k]$, satisfying $\sum_{i\in[k]}p_{i} = 1$, we have the following expansion
    $$1 - P\{CS\} = \sum_{j\in[k]_{-}}\exp\left\{-T \cdot G_{j}(p_{1}, p_{j})\right\} \cdot \frac{1}{\sqrt{2\pi}\cdot \lambda_{j}^{*}p_{j}\tilde{\sigma}_{1,j}\sqrt{T}} \cdot \left(1 + \sum_{l=1}^{\ell}\frac{c_{j,l}}{T^{l}} + \mathcal{O}(T^{-(\ell + 1)})\right).$$
\end{theorem}

Herein $\lambda_{j}^{*}p_{j}\tilde{\sigma}_{1,j}$ and $c_{j,l}$ are constants irrelevant to $T$ and may depend on the sampling distribution of alternatives and the sampling ratios, which will soon be defined formally. Theorem \ref{thm:main} represents $1-P\{CS\}$ as a weighted sum of exponential functions with weights being polynomials in $T^{-1/2}$. In contrast with LDP, our main result captures the behavior of PCS in the finite sample case by showcasing the impact of the budget. Moreover, the parenthesized term can be expanded into any order and the approximation error is controllable even when $T$ is fixed. 

Even if we only retain the first term in the parenthesis, our approximation is still more informative than the rate functions, relating PCS to not only the number of alternatives, but also the size of simulation budget. For instance, let $a$ be an integer and $0<s\ll S$ be two reals. Consider a sampling ratio $\bm{p}^{a}$ indexed by $2\leq a< k$ such that $G_{j}(p_{1}^{a}, p_{j}^{a}) = s$ for $2\leq j\leq a$, $G_{j}(p_{1}^{a}, p_{j}^{a}) = S$ for $a+1\leq j\leq k$. Given $\bm{p}^{a}$, the ordinal relationships between the best alternative and alternatives $2$ to $a$ are harder to tell, whereas the differences in sample means between the best alternative and alternatives $a+1$ to $k$ are relatively large with a high probability. Loosely speaking, the value of $a$ can be interpreted as the number of competitive alternatives, wherein a \textit{competitive alternative} denotes an alternative with a comparatively higher likelihood of being optimal. The LDP rate does not depend on $a$ directly but depends on the value of $s$. In contrast, even if the value of $s$ stays unchanged, an increase in the number of competitive alternatives would result in a decrease in our approximation, reflecting the influence of competitive alternatives on PCS.

The LDP theorem of \cite{glynn2004large} is a direct corollary of the main result if Assumptions \ref{ass:light} and \ref{ass:btv} are both valid. It follows from Theorem \ref{thm:main}, by choosing $\ell = 0$ and letting $T$ be large enough such that $-1/2<\mathcal{O}(1/T)<1/2$, that for $j^{*} = \argmin_{j\in[k]_{-}}G_{j}(p_{1}, p_{j})$,
$$\begin{aligned}
    \frac{1}{2}\exp\{-T G_{j^{*}}(p_{1}, p_{j^{*}})\}\cdot \frac{1}{\sqrt{2\pi}\cdot\lambda_{j^{*}}^{*}p_{j^{*}}\tilde{\sigma}_{1,j^{*}}\sqrt{T}}
    &\leq 1 - P\{CS\},
\end{aligned}$$
and
$$\begin{aligned}
    1 - P\{CS\}&\leq \frac{3}{2}\sum_{j\in[k]_{-}}\exp\{-T G_{j^{*}}(p_{1}, p_{j^{*}})\}\cdot \frac{1}{\sqrt{2\pi}\cdot\lambda_{j}^{*}p_{j}\tilde{\sigma}_{1,j}\sqrt{T}} \\
    &= \frac{3}{2}\exp\{-T G_{j^{*}}(p_{1}, p_{j^{*}})\}\sum_{j\in[k]_{-}}\frac{1}{\sqrt{2\pi}\cdot\lambda_{j}^{*}p_{j}\tilde{\sigma}_{1,j}\sqrt{T}}.
\end{aligned}$$
Because $T^{-1/2}$ is a lower-order infinitesimal than any exponential function in $T$, it follows by taking logarithm and then dividing $-T$ on all sides that $-\frac{1}{T}\log(1 - P\{CS\}) \rightarrow G_{j^{*}}(p_{1}, p_{j^{*}}) = \min_{j\in[k]_{-}}G_{j}(p_{1}, p_{j})$. We illustrate our result under a Gaussian setting.
\begin{example}[Gaussian]\label{example:Gaussian}
    Suppose $X_{i}\sim N(m_{i}, \sigma_{i}^{2})$, $\forall i\in[k]$. Then we have $G_{j}(p_{1}, p_{j}) = \frac{1}{2}(m_{1} - m_{j})^{2} / \left(\sigma_{1}^{2}/p_{1} + \sigma_{j}^{2}/p_{j}\right)$, and $\lambda_{j}^{*}p_{j}\tilde{\sigma}_{1,j} = \sqrt{2G_{j}(p_{1}, p_{j})}$. See the electronic companion for details. With the zeroth order approximation, i.e., $\ell=0$, Theorem \ref{thm:main} yields the following equality
    \begin{equation}\label{ex:Gaussian}
        1 - P\{CS\} = \sum_{j\in[k]_{-}}\exp\left\{-T\cdot \frac{1}{2}\frac{(m_{1} - m_{j})^{2}}{\sigma_{1}^{2}/p_{1} + \sigma_{j}^{2}/p_{j}}\right\} \cdot \frac{\sqrt{\sigma_{1}^{2}/p_{1} + \sigma_{j}^{2}/p_{j}}}{\sqrt{2\pi T}\cdot(m_{1} - m_{j})} \cdot \left(1 + o(1)\right).
    \end{equation}
    Using the LDP theorem of \cite{glynn2004large} instead, we see that
    \begin{equation*}
        \lim_{T\rightarrow \infty}-\frac{1}{T}\log(1 - P\{CS\}) = \min_{j\in[k]_{-}}\frac{1}{2}\frac{(m_{1} - m_{j})^{2}}{\sigma_{1}^{2}/p_{1} + \sigma_{j}^{2}/p_{j}}.
    \end{equation*}
\end{example}
\begin{example}[Exponential]\label{example:exp}
    Suppose $X_{i}\sim \operatorname{Exp}(\beta_{i})$, $\forall i\in[k]$. The CDF of $X_{i}$ is $F_{X_{i}}(x) = 1 - \exp(-x/\beta_{i})$, and the mean and variance are $\beta_{i}$ and $\beta_{i}^{2}$, respectively. We have 
    $$G_{j}(p_{1}, p_{j}) =  (p_{1} + p_{j})\log\frac{p_{1}+p_{j}}{p_{1} + p_{j}\cdot\beta_{1} / \beta_{j}} + p_{j}\log\beta_{1}/\beta_{j},$$
    and
    $$\lambda_{j}^{*}p_{j}\tilde{\sigma}_{1,j} = \frac{\beta_{1}/\beta_{j} - 1}{p_{1}^{-1} + p_{j}^{-1}} \cdot \left(p_{1}^{-1} + p_{j}^{-1}\cdot\beta_{j}^{2}/\beta_{1}^{2}\right)^{\frac{1}{2}}.$$
    See the electronic companion for details. Then Theorem \ref{thm:main} yields the following equality
    $$\begin{aligned}
         & 1 - P\{CS\} \\
        = & \sum_{j\in[k]_{-}}\left(\frac{p_{1}+p_{j}\cdot\beta_{1} / \beta_{j}}{p_{1} + p_{j}}\right)^{T_{1} + T_{j}} \cdot (\beta_{j}/\beta_{1})^{T_{j}}\cdot \frac{p_{1}^{-1} + p_{j}^{-1}}{\sqrt{2\pi T}\cdot (\beta_{1}/\beta_{j} - 1)\left(p_{1}^{-1} + p_{j}^{-1}\cdot\beta_{j}^{2}/\beta_{1}^{2}\right)^{\frac{1}{2}} }\cdot(1+o(1)).
    \end{aligned}$$
    The rate function $G_{j}(p_{1}, p_{j})$ above is consistent with the calculation of \cite{gao2016exponential}. Both the rate function $G_{j}(p_{1}, p_{j})$ and the coefficient $\lambda_{j}^{*}p_{j}\tilde{\sigma}_{1,j}$ are functions of the ratio between the mean parameters of exponential random variables, i.e., $\beta_{1}/\beta_{j}$. This is consistent with the fact that $\beta_{i}$ is a scaling parameter, i.e., $a\operatorname{Exp}(\beta_{i}) = \operatorname{Exp}(a\beta_{i})$, $\forall a>0$.
\end{example}

\section{Approximation Methodology}\label{sec:methodology}
In this section, we will develop the novel approximation methodology. In Section \ref{sec:ldsm}, we recap the notion of the LDP. In Section \ref{sec:binarycomparison}, we focus on the Bahadur-Rao type expansion of PCS for binary comparison, and in Section \ref{sec:multiplecomparison}, we generalize it to general multiple comparisons.

\subsection{Large Deviations of Sample Mean}\label{sec:ldsm}
Recall that the probability of false selection equals to
$$1 - P\{CS\} = P\left(\bigcup_{j\in[k]_{-}}\bar{X}_{1}(T_{1}) \leq \bar{X}_{j}(T_{j})\right).$$
Obviously, it follows from the sub-additivity of probability measures that
\begin{equation}\label{eq:twosidebound}
    \max_{j\in[k]_{-}}P\left(\bar{X}_{1}(T_{1}) \leq \bar{X}_{j}(T_{j})\right) \leq 1-P\{CS\} \leq \sum_{j\in[k]_{-}} P\left(\bar{X}_{1}(T_{1}) \leq \bar{X}_{j}(T_{j})\right) \leq (k - 1)\max_{j\in[k]_{-}}P\left(\bar{X}_{1}(T_{1}) \leq \bar{X}_{j}(T_{j})\right).
\end{equation}
For $j\in [k]_{-}$ fixed, the exponential rate of $P(\bar{X}_{1}(T_{1}) \leq \bar{X}_{j}(T_{j}))$ can be characterized utilizing the CGF of the vector $(\bar{X}_{1}(T_{1}), \bar{X}_{j}(T_{j}))$, i.e., $\Lambda^{(T)}(\lambda_{1}, \lambda_{j}) := \log\mathbb{E}\exp\{\lambda_{1}\bar{X}_{1}(T_{1}) + \lambda_{j}\bar{X}_{j}(T_{j})\}$. Define
\begin{align*}
    I_{1,j}(x_{1}, x_{j}) &:= \sup_{\lambda_{1}, \lambda_{j}}(\lambda_{1}x_{1} + \lambda_{j}x_{j} - \lim_{T}\frac{1}{T}\Lambda^{(T)}(T\lambda_{1}, T\lambda_{j})) \\
    &= \sup_{\lambda_{1}, \lambda_{j}}(\lambda_{1}x_{1} + \lambda_{j}x_{j} - p_{1}\Lambda_{1}(\lambda_{1}/p_{1}) - p_{j}(\lambda_{j}/p_{j})) \\
    &= p_{1}I_{1}(x_{1}) + p_{j}I_{j}(x_{j}),
\end{align*}
where $I_{i}(x_{i}) = \sup_{\lambda} (\lambda x_{i} - \Lambda_{i}(\lambda))$ denotes the Legendre-Fenchel transformation of $\Lambda_{i}$. Following \cite{glynn2004large}, the LDP rate for comparing $\bar{X}_{1}(T_{1})$ and $\bar{X}_{j}(T_{j})$ is $G_{j}(p_{1}, p_{j}) = I_{1,j}(\mu_{j}, \mu_{j}) := \inf_{x_{1}\leq x_{j}}I_{1,j}(x_{1}, x_{j})$ with $\mu_{j}\in[m_{j}, m_{1}]$. It follows from inequality (\ref{eq:twosidebound}) that
\begin{equation}\label{eq:samerate}
    \lim_{T\rightarrow \infty}-\frac{1}{T}\log \sum_{j\in[k]_{-}}P(\bar{X}_{1}(T_{1}) \leq \bar{X}_{j}(T_{j})) = \min_{j\in[k]_{-}}G_{j}(p_{1}, p_{j}) = \lim_{T\rightarrow \infty}-\frac{1}{T}\log(1-P\{CS\}).
\end{equation}

\subsection{Probability of Incorrect Binary Comparison}\label{sec:binarycomparison}
The LDP theory implies that the exponential rate of PCS is equal to the exponential rate of PCS for the binary comparison between the best alternative and the sub-optimal alternative that minimizes $I_{1,j}(\mu_{j}, \mu_{j})$. Therefore, we will first focus on an expansion of the binary comparison $P(\bar{X}_{1}(T_{1}) \leq \bar{X}_{j}(T_{j}))$ for some fixed $j\in[k]_{-}$ similar to Theorem \ref{thm:main}.

We take inspiration from \cite{bahadur1960deviations} who characterize the tail probability $P(\bar{X}(T) \geq x)$ for some sample mean random variable $\bar{X}(T)$ and some real number $x>\mathbb{E}[\bar{X}(T)]$. They show that the function $a(T)$ in (\ref{eq:asympeq}) takes the form of a polynomial in $T^{-1/2}$ with an error term. An exponential tilting technique is utilized to extract the dominating term, i.e., the exponential function $\exp\{-T\cdot \min_{j\in [k]_{-}}G_{j}(p_{1}, p_{j})\}$, from the PCS. Then, the remainder is characterized using the Edgeworth expansion of the probability density function of sampling distributions. Although their result does no directly apply to our problem since the critical point $x$ in our problem is a random variable $\bar{X}_{j}(T_{j})$ rather than a real number, the techniques can be modified for our problem.

Note that $P(\bar{X}_{1}(T_{1}) \leq \bar{X}_{j}(T_{j}))$ has a large deviations rate
$$\exp\{-I_{1,j}(\mu_{j},\mu_{j})T\}  = \exp\{-(p_{1}I_{1}(\mu_{j})+p_{j}I_{j}(\mu_{j}))T\} = \exp\{-\inf_{x\in[m_{j}, m_{1}]}(p_{1}I_{1}(x)+p_{j}I_{j}(x))T\}.$$
On the other hand, note that $P\left(\bar{X}_{1}(T_{1})\leq \mu_{j}\right)$ and $P\left(\mu_{j} \leq \bar{X}_{j}(T_{j})\right)$ have rates $p_{1}I_{1}(\mu_{j})$ and $p_{j}I_{j}(\mu_{j})$, respectively. Therefore, $P\left(\bar{X}_{1}(T_{1}) \leq \mu_{j} \leq \bar{X}_{j}(T_{j})\right) = P\left(\bar{X}_{1}(T_{1}) \leq \mu_{j}\right)\cdot P\left(\mu_{j} \leq \bar{X}_{j}(T_{j})\right)$ has the same rate as $P(\bar{X}_{1}(T_{1}) \leq \bar{X}_{j}(T_{j}))$. Herein, $\mu_{j}$ can be interpreted as a critical point such that it happens with a high probability that the order of $\mu_{j}$ and the sample means of $X_{1}$ and $X_{j}$ are reversed. This motivates us to use exponential tilting to massage the probability measures of $X_{1}'s$ and $X_{j}'s$ to concentrate around $\mu_{j}$. 

Specifically, denote $\lambda_{1}^{*} := \argmax_{\lambda} (\lambda\mu_{j} - \Lambda_{1}(\lambda))$ and $\lambda_{j}^{*} := \argmax_{\lambda} (\lambda\mu_{j} - \Lambda_{j}(\lambda))$. It follows that $\lambda_{1}'(\lambda_{1}^{*}) = \mu_{j} = \lambda_{j}'(\lambda_{j}^{*})$. Now we define a couple of random variables independent of each other and anything else mentioned above by their CDFs $F_{Z_{1}}(z) = \int_{-\infty}^{z}{e^{I_{1}(\mu_{j})}e^{\lambda^{*}_{1}(x-\mu_{j})}dF_{X_{1}}(x)}$ and $F_{Z_{j}}(z) = \int_{-\infty}^{z}{e^{I_{j}(\mu_{j})}e^{\lambda^{*}_{j}(x-\mu_{j})}dF_{X_{j}}(x)}$. It follows that $\mu_{j} = \Lambda_{1}'(\lambda_{1}^{*}) = (\mathbb{E}[X_{1}\exp\{\lambda_{1}^{*}(X_{1}-\mu_{j}^{*})\}])/(\mathbb{E}[\exp\{\lambda_{1}^{*}(X_{1}-\mu_{j}^{*})\}])$, where the second equality follows from the definition $\Lambda_{1}^{*}(\lambda)=\log\mathbb{E}[\exp\{\lambda X_{1}\}]$. Hence, $$\mathbb{E}[Z_{1}] = \mathbb{E}\left[X_{1}\frac{dF_{Z_{1}}}{dF_{X_{1}}}(X_{1})\right] = e^{I_{1}(\mu_{j})}\mathbb{E}[X_{1}e^{\lambda_{1}^{*}(X_{1}-\mu_{j})}] = e^{I_{1}(\mu_{j})}\mu_{j}\mathbb{E}[e^{\lambda_{1}^{*}(X_{1}-\mu_{j})}] = \mu_{j},$$ where $dF_{Z_{1}}/dF_{X_{1}}$ is the Radon-Nikodym derivative of $F_{Z_{1}}$ with respect to $F_{X_{1}}$, and the variance of $Z_{1}$ is $\operatorname{Var}(Z_{1}) =: \sigma_{1}^{2} > 0$. Similarly, $\mathbb{E}[Z_{j}] = 0$ and define $\operatorname{Var}(Z_{j}) =: \sigma_{j}^{2} > 0$. Denote $\Omega(T_{1},T_{j}) = \left\{(x_{1}^{(1)}+\dots+x_{1}^{(T_{1})})/T_{1} \leq (x_{j}^{(1)}+\dots+x_{j}^{(T_{j})})/T_{j}\right\}$ as a fixed subset of $\mathbb{R}^{T_{1}+T_{j}}$. Then, 
\begin{equation}\label{eq:rescale}\begin{aligned}
    &P\left(\bar{X}_{1}(T_{1}) \leq \bar{X}_{j}(T_{j})\right) \\
    = &\int_{\Omega(T_{1},T_{j})}{dF_{X_{1}^{(1)},\dots,X_{1}^{(T_{1})}}(x_{1}^{(1)},\dots,x_{1}^{(T_{1})}) dF_{X_{j}^{(1)},\dots,X_{j}^{(T_{j})}}(x_{j}^{(1)},\dots,x_{j}^{(T_{j})})} \\
    = &\exp\{-TI_{1,j}(\mu_{j}, \mu_{j})\}\cdot \\
    &\qquad \int_{\Omega(T_{1},T_{j})}{dF_{Z_{1}^{(1)},\dots,Z_{1}^{(T_{1})}}(z_{1}^{(1)},\dots,z_{1}^{(T_{1})}) dF_{Z_{j}^{(1)},\dots,Z_{j}^{(T_{j})}}(z_{j}^{(1)},\dots,z_{j}^{(T_{j})})\cdot} \\
    &\qquad \exp\left\{- \left(\lambda_{1}^{*}\sum_{i=1}^{T_{1}}(z_{1}^{(i)} - \mu_{j}) + \lambda_{j}^{*}\sum_{i=1}^{T_{j}}(z_{j}^{(i)} - \mu_{j})\right)\right\} \\
    = &\exp\{-TI_{1,j}(\mu_{j}, \mu_{j})\}\cdot \int_{x \leq y}{dF_{H_{1}(T_{1})}(x)dF_{H_{j}(T_{j})}(y) \cdot \exp\{-(\lambda_{1}^{*}T_{1}x + \lambda_{j}^{*}T_{j}y)\}}.
\end{aligned}\end{equation}
Therein, $H_{1}(T_{1}) = \frac{1}{T_{1}}\sum_{i=1}^{T_{1}}(Z_{1}^{(i)} - \mu_{j})$ and $H_{j}(T_{j}) = \frac{1}{T_{j}}\sum_{i=1}^{T_{j}}(Z_{j}^{(i)} - \mu_{j})$. It follows from the definition of $\mu_{j}$ that $p_{1}I_{1}'(\mu_{j}) + p_{j}I_{j}'(\mu_{j}) = p_{1}\lambda_{1}^{*} + p_{j}\lambda_{j}^{*} = 0$, and thus $\lambda_{1}^{*}T_{1}x + \lambda_{j}^{*}T_{j}y = \lambda_{j}^{*}p_{j}T(y-x)$. Let $\tilde{\sigma}_{1,j}^{2} := \sigma_{1}^{2}/p_{1} + \sigma_{j}^{2}/p_{j}$. The mean and variance of $\tilde{H}_{j} := \sqrt{T} (H_{j}(T_{j}) - H_{1}(T_{1}))/\tilde{\sigma}_{1,j}$ are 0 and 1, respectively. The PICS can be further simplified as 
\begin{equation}\label{eq:concentrate}
    P\left(\bar{X}_{1}(T_{1}) \leq \bar{X}_{j}(T_{j})\right) = \exp\{-TI_{1,j}(\mu_{j}, \mu_{j})\}\cdot\mathbb{E}\left[\bm{1}\{0 \leq \tilde{H}_{j}\}\exp\left\{-\lambda_{j}^{*}p_{j}\tilde{\sigma}_{1,j}\sqrt{T}\tilde{H}_{j}\right\}\right].
\end{equation}

By applying the exponential tilting, we not only extract the exponential term which dominates the PICS but also characterize the remaining term as an expectation. Now, it suffices to derive a series expansion of the expectation in $T^{-1/2}$ to characterize the focal probability. Denote the random variable in the expectation as $W(\tilde{H}_{j})$ and let $f_{\tilde{H}_{j}}(x)$ be the probability density function of $\tilde{H}_{j}$. Intuitively, we can rewrite the expectation as 
\begin{equation}\label{eq:rawparseval}
    \mathbb{E}[W(\tilde{H}_{j})] = \int{W(x)f_{\tilde{H}_{j}}(x)dx} = \int{\operatorname{conj}(\mathcal{F}W(\lambda)) \mathcal{F}f_{\tilde{H}_{j}}(\lambda) d\lambda},
\end{equation} where $\mathcal{F}$ denotes the Fourier operator, i.e., $\mathcal{F}f(\lambda) = \int_{-\infty}^{\infty}{e^{i\lambda x}f(x)dx}$ for $f\in L^{2}(\mathbb{R})$, and $\operatorname{conj}(\cdot)$ denotes the conjugate operator. Therein, the last equality follows from the Parseval's identity. Note that the Fourier transformation of a probability density function can be expressed by its cumulants, and it is also straightforward to express $\mathcal{F}W(\lambda)$ as a series in $T^{-1/2}$. Therefore, we can truncate the two Fourier transformations to arrive at a Bahadur-Rao type expansion. To make it rigorous, we introduce the following lemma.

\begin{lemma}\label{lemma:expand}
    Under Assumptions \ref{ass:light} and \ref{ass:btv}, for any integer $q\geq 3$, if the sampling ratios are bounded away from zero, then the density function $F_{\tilde{H}_{j}}(x)$ has the following expansion
    $$F_{\tilde{H}_{j}}(x) = \Phi(x) + \sum_{\nu=1}^{q-3}\frac{p_{3\nu-1,T}(x)}{T^{\nu/2}}e^{-x^{2}/2} + R_{q,T}(x),$$
    where $p_{3\nu-1,T}(x)$ is a polynomial of order $(3\nu-1)$ in $x$ and $\Vert R_{q,T}\Vert_{\infty}=\mathcal{O}(T^{-(q-2)/2})$. Specially, for $\nu = 1$, we have
    $$p_{3-1,T}(x) = \frac{1}{3!\sqrt{2\pi}\tilde{\sigma}_{1,j}^{3/2}} \left(\frac{\Lambda^{(3)}_{1}(\lambda_{1}^{*})}{p_{1}^{2}}+\frac{\Lambda^{(3)}_{j}(\lambda_{j}^{*})}{p_{j}^{2}}\right) (1-x^{2}).$$
\end{lemma}

The remainder $R_{q,T}$, though not explicitly expressed above, has both an upper bound and an lower bound proportional to the $q$-th cumulants of $Z_{1}/\tilde{\sigma}_{1,j}$ and $Z_{j}/\tilde{\sigma}_{1,j}$. We note that it could be beneficial to use high-order expansions in applications when the sampling distribution behaves similar to Gaussian distributions, i.e., when the cumulants of orders higher than $3$ are close to $0$. If, on the other hand, the sampling distribution is highly skewed, then the approximation error could be large.

Denote $K_{q,T}(x) := F_{\tilde{H}_{j}}(x) - R_{q,T}(x)$. It is apparent that $K_{q,T}'$ is square integrable with regard to the Lebesgue measure. Finally, note that for $W(x) = e^{-\lambda_{j}^{*}p_{j}\tilde{\sigma}_{1,j}\sqrt{T}x} \cdot \bm{1}_{[0,\infty)}(x)$, we have $\mathcal{F}W(\lambda) = (\lambda_{j}^{*}p_{j}\tilde{\sigma}_{1,j}\sqrt{T} - i\lambda)^{-1}$. We first apply integration by part in order to truncate the density function, followed by reversing the integration by part and applying the Parseval's identity: 
\begin{equation}\label{eq:parseval}
    \begin{aligned}
        \mathbb{E}[W(\tilde{H}_{j})]=& \int_{0}^{\infty}{\exp\{-\lambda^{*}_{j}p_{j}\tilde{\sigma}_{1,j}\sqrt{T}x\}dF_{\tilde{H}_{j}}(x)} \\
        =& \lambda_{j}^{*}p_{j}\tilde{\sigma}_{1,j}\sqrt{T}\int_{0}^{\infty}{\exp\{-\lambda^{*}_{j}p_{j}\tilde{\sigma}_{1,j}\sqrt{T}x\}(F_{\tilde{H}_{j}}(x) - F_{\tilde{H}_{j}}(0))dx} \\
        =& \lambda_{j}^{*}p_{j}\tilde{\sigma}_{1,j}\sqrt{T}\int_{0}^{\infty}{\exp\{-\lambda^{*}_{j}p_{j}\tilde{\sigma}_{1,j}\sqrt{T}x\}(K_{q,T}(x) - K_{q,T}(0))dx} + \mathcal{O}(T^{-(q-2)/2}) \\
        =& \int_{-\infty}^{\infty}{\bm{1}_{[0,\infty)}(x)\exp\{-\lambda^{*}_{j}p_{j}\tilde{\sigma}_{1,j}\sqrt{T}x\}K_{q, T}'(x) dx} + \mathcal{O}(T^{-(q-2)/2}) \\
        =& \frac{1}{2\pi\cdot \lambda_{j}^{*}p_{j}\tilde{\sigma}_{1,j}\sqrt{T}}\int_{-\infty}^{\infty}{\left(1 + \frac{i\lambda}{\lambda_{j}^{*}p_{j}\tilde{\sigma}_{1,j}\sqrt{T}}\right)^{-1}\mathcal{F}K_{q,T}'(\lambda) d\lambda} + \mathcal{O}(T^{-(q-2)/2}).
    \end{aligned}
\end{equation}
The third equality follows from Lemma \ref{lemma:expand} by noticing $\lambda_{j}^{*} > 0$.

Although it is still hard to evaluate the integral term above, it has a straightforward expansion. We illustrate the expansion using $q = 3$. The Fourier transformation of $K_{3,T}'$ is
\begin{equation}\label{eq:order3}
    \mathcal{F}K_{3,T}'(\lambda) = \mathcal{F}\Phi'(\lambda) = e^{-\lambda^{2}/2}.
\end{equation}
Formally, combining $(1 + x)^{-1} = 1 - x + o(x)$, (\ref{eq:concentrate}), (\ref{eq:parseval}), and (\ref{eq:order3}) yields that 
\begin{equation*}
     P\left(\bar{X}_{1}(T_{1}) \leq \bar{X}_{j}(T_{j})\right) 
    = \exp\{-TI_{1,j}(\mu_{j}, \mu_{j})\} \cdot \frac{1}{\sqrt{2\pi}\cdot\lambda_{j}^{*}p_{j}\tilde{\sigma}_{1,j}\sqrt{T}}\cdot \mathcal{O}(1).
\end{equation*}
To justify the above equality, note that
$$\begin{aligned}
    \left\Vert \left( 1 + \frac{i\lambda}{\lambda_{j}^{*}p_{j}\tilde{\sigma}_{1,j}\sqrt{T}}\right)^{-1} - \left( 1 - \frac{i\lambda}{\lambda_{j}^{*}p_{j}\tilde{\sigma}_{1,j}\sqrt{T}}\right) \right\Vert &= \left\Vert \left( 1 - \frac{i\lambda}{\lambda_{j}^{*}p_{j}\tilde{\sigma}_{1,j}\sqrt{T}}\right) \cdot \frac{ \dfrac{\lambda^{2}}{{\lambda_{j}^{*}}^{2}p_{j}^{2}\tilde{\sigma}_{1,j}^{2} T} }{ 1 + \dfrac{\lambda^{2}}{{\lambda_{j}^{*}}^{2}p_{j}^{2}\tilde{\sigma}_{1,j}^{2} T} } \right\Vert \\
    &\leq \frac{\lambda^{2}}{{\lambda_{j}^{*}}^{2}p_{j}^{2}\tilde{\sigma}_{1,j}^{2} T} + \frac{\vert \lambda\vert^{3}}{{\lambda_{j}^{*}}^{3}p_{j}^{3}\tilde{\sigma}_{1,j}^{3} T^{3/2}}.
\end{aligned}$$
Since $\lambda^{2}e^{-\lambda^{2}/2}$ and $\vert\lambda\vert^{3}e^{-\lambda^{2}/2}$ are integrable, it follows from the dominating convergence theorem \citep{rudin1987} that
$$\begin{aligned}
    & \lim_{T\rightarrow \infty} \int_{-\infty}^{\infty}{ \left[\left( 1 + \frac{i\lambda}{\lambda_{j}^{*}p_{j}\tilde{\sigma}_{1,j}\sqrt{T}}\right)^{-1} - \left( 1 - \frac{i\lambda}{\lambda_{j}^{*}p_{j}\tilde{\sigma}_{1,j}\sqrt{T}}\right)\right]\mathcal{F}K_{3,T}'(\lambda) d\lambda} \\
    = & \int_{-\infty}^{\infty}{ \lim_{T\rightarrow \infty} \left[\left( 1 + \frac{i\lambda}{\lambda_{j}^{*}p_{j}\tilde{\sigma}_{1,j}\sqrt{T}}\right)^{-1} - \left( 1 - \frac{i\lambda}{\lambda_{j}^{*}p_{j}\tilde{\sigma}_{1,j}\sqrt{T}}\right)\right]\mathcal{F}K_{3,T}'(\lambda) d\lambda } = 0.
\end{aligned}$$
Consequently, it follows from (\ref{eq:parseval}) that
$$\begin{aligned}
    \mathbb{E}[W(\tilde{H}_{j})] = \frac{1}{2\pi\cdot \lambda_{j}^{*}p_{j}\tilde{\sigma}_{1,j}\sqrt{T}} \left[\int_{-\infty}^{\infty}{\left(1 - \frac{i\lambda}{\lambda_{j}^{*}p_{j}\tilde{\sigma}_{1,j}\sqrt{T}}\right)\mathcal{F}K_{3,T}'(\lambda) d\lambda} + o(1)\right] + \mathcal{O}(T^{-1/2}).
\end{aligned}$$
Finally, combining this with (\ref{eq:concentrate}) completes the justification.

For $q = 4$, we have 
\begin{equation}\label{eq:order4}
    \mathcal{F}K_{4,T}'(\lambda) = \left(1 + \frac{1}{\sqrt{T}}\frac{1}{3!\sqrt{2\pi}\tilde{\sigma}_{1,j}^{3/2}} \left(\frac{\Lambda^{(3)}_{1}(\lambda_{1}^{*})}{p_{1}^{2}}+\frac{\Lambda^{(3)}_{j}(\lambda_{j}^{*})}{p_{j}^{2}}\right) (1 + \lambda^{2})\right) e^{-\lambda^{2}/2} .
\end{equation}
Similarly, it follows from $(1 + x)^{-1} = 1 - x + \mathcal{O}(x^{2})$ by (\ref{eq:concentrate}), (\ref{eq:parseval}), and (\ref{eq:order4}) that
\begin{equation*}
    \begin{aligned}
         P\left(\bar{X}_{1}(T_{1}) \leq \bar{X}_{j}(T_{j})\right) 
        &= \exp\{-TI_{1,j}(\mu_{j}, \mu_{j})\} \cdot \frac{1}{\sqrt{2\pi}\cdot\lambda_{j}^{*}p_{j}\tilde{\sigma}_{1,j}\sqrt{T}}\cdot (1 +  \mathcal{O}\left(T^{-1/2})\right).
    \end{aligned}
\end{equation*}

The following proposition concludes a general result on the probability of incorrect binary comparison.
\begin{proposition}\label{prop:exponential}
    Under Assumptions \ref{ass:light} and \ref{ass:btv}, if the sampling ratios are bounded away from zero, the probability of incorrect binary comparison decays exponentially. Specifically, for any integer $\ell \geq 0$, there exist constants $c_{j,1}, c_{j,2}, \dots, c_{j,\ell}$, which depend only on cumulants and sampling ratio
    $$P\left(\bar{X}_{1}(T_{1}) \leq \bar{X}_{j}(T_{j})\right) = \exp\{-TI_{j}(\mu_{j}, \mu_{j})\} \cdot \frac{1}{\sqrt{2\pi}\cdot\lambda_{j}^{*}p_{j}\tilde{\sigma}_{1,j}\sqrt{T}} \cdot \left(1 + \frac{c_{j,1}}{T} + \dots + \frac{c_{j,\ell}}{T^{\ell}} + \mathcal{O}(T^{-(\ell+1)})\right).$$
\end{proposition}

\subsection{Probability of Incorrect Multiple Comparison}\label{sec:multiplecomparison}
Notice a lower bound of PICS follows from the inclusion-exclusion principle:
\begin{equation}\label{ex-inclusion}
    1-P\{CS\} \geq \sum_{j=2}^{k} P\left(\bar{X}_{1}(T_{1}) \leq \bar{X}_{j}(T_{j})\right) - \sum_{2\leq i\neq j \leq k} P\left(\bar{X}_{1}(T_{1}) \leq \bar{X}_{i}(T_{i}) \wedge \bar{X}_{j}(T_{j})\right).
\end{equation}
Inspired by inequality (\ref{eq:samerate}), we will show that the latter term is negligible, while the former is dominant because it shares the same exponential rate as $1-P\{CS\}$.

According to (\ref{ex-inclusion}), it suffices to deal with the terms $P( \bar{X}_{1}(T_{1})\leq \bar{X}_{i}(T_{i}) \wedge \bar{X}_{j}(T_{j}) )$ for $1<i<j\leq k$. We again apply the proposed technique applied to $P( \bar{X}_{1}(T_{1})\leq \bar{X}_{i}(T_{i}) )$, beginning from the rate function $I_{1,i,j}(x_{1}, x_{i}, x_{j}) = p_{1}I_{1}(x_{1}) + p_{i}I_{i}(x_{i}) + p_{j}I_{j}(x_{j})$. Because $I_{1,i,j}(\cdot)$ is continuous and coercive, there must exist $(x_{1}^{*}, x_{i}^{*}, x_{j}^{*})$ such that $I_{i,j}(x_{1}^{*}, x_{i}^{*}, x_{j}^{*}) = \inf_{\{(x_{1},x_{i},x_{j})\in\mathbb{R}^{3}: x_{1}\leq x_{i}\wedge x_{j}\}}I_{i,j}(x_{1}, x_{i}, x_{j})$. 
We have the following proposition.

\begin{proposition}\label{prop:exp-trivariate}
    Under the regularity conditions in Lemma \ref{lemma:expand}, the probability of simultaneous incorrect binary comparison for two alternatives decays exponentially, i.e., 
    \begin{equation*}
        P\left(\bar{X}_{1}(T_{1}) \leq \bar{X}_{i}(T_{i})\wedge \bar{X}_{j}(T_{j})\right) = \exp\{-TI_{i,j}(x_{1}^{*}, x_{i}^{*}, x_{j}^{*})\} \cdot \mathcal{O}(T^{-1/2}).
    \end{equation*}
    Moreover, we have $$I_{i,j}(x_{1}^{*}, x_{i}^{*}, x_{j}^{*}) > \min\{I_{i}(\mu_{i}, \mu_{i}), I_{j}(\mu_{j}, \mu_{j})\}.$$
\end{proposition}

The proof can be found in the online supplement. The first part of Proposition \ref{prop:exp-trivariate} is a parallel to Proposition \ref{prop:exponential}. We see that the dominating term in PICS is the summation of probability of false binary comparisons. For a given sampling ratio $\bm{p}$, the PICS decays at rate $\min_{j\in[k]_{-}} I_{j}(\mu_{j}, \mu_{j})$. Moreover, the second part allows us to truncate the expansion of the PCS. Combining all the results in this section, we shall give the proof for Theorem \ref{thm:main}.

\subsection{Proof of Theorem \ref{thm:main}}
Combining Propositions \ref{prop:exponential} and \ref{prop:exp-trivariate}, we provide the proof of Theorem \ref{thm:main}.
\begin{proof}{Proof of Theorem \ref{thm:main}.}
    It follows immediately by Proposition \ref{prop:exponential} that
    \begin{equation}\label{eq:transit}
        \sum_{j\in [k]_{-}} P\left(\bar{X}_{1}(T_{1}) \leq \bar{X}_{j}(T_{j})\right) = \sum_{j\in[k]_{-}} \exp\{-TI_{j}(\mu_{j}, \mu_{j})\} \cdot \frac{1}{\sqrt{2\pi}\cdot\lambda_{j}^{*}p_{j}\tilde{\sigma}_{1,j}\sqrt{T}} \cdot \left(1 + \sum_{l=1}^{\ell}\frac{c_{j,l}}{T^{l}} + \mathcal{O}(T^{-(\ell+1)})\right).
    \end{equation}
    And it follows from Propositions \ref{prop:exponential} and \ref{prop:exp-trivariate} that, for any $\varepsilon>0$ such that 
    $$\varepsilon < I_{i,j}(x_{1}^{*}, x_{i}^{*}, x_{j}^{*}) - \min\{I_{i}(\mu_{i}, \mu_{i}), I_{j}(\mu_{j}, \mu_{j})\},\quad \forall i\neq j\in[k]_{-},$$
    we have
    \begin{equation*}
        P\left(\bar{X}_{1}(T_{1}) \leq \bar{X}_{i}(T_{i})\wedge \bar{X}_{j}(T_{j})\right) = \max\left\{P\left(\bar{X}_{1}(T_{1}) \leq \bar{X}_{i}(T_{i})\right), P\left(\bar{X}_{1}(T_{1}) \leq \bar{X}_{i}(T_{i})\right)\right\} \cdot \mathcal{O}(\exp\{-\varepsilon T\}).
    \end{equation*}
    Consequently, we see that
    \begin{equation}\label{eq:asymptoticequivalence}
        \sum_{2\leq i\neq j \leq k} P\left(\bar{X}_{1}(T_{1}) \leq \bar{X}_{i}(T_{i}) \wedge \bar{X}_{j}(T_{j})\right)= \sum_{j\in[k]_{-}}P\left(\bar{X}_{1}(T_{1}) \leq \bar{X}_{j}(T_{j})\right) \cdot \mathcal{O}(\exp\{-\varepsilon T\}).
    \end{equation}
    Combining (\ref{ex-inclusion}), (\ref{eq:transit}) and the second inequality in (\ref{eq:samerate}) completes the proof. \halmos
\end{proof}

\section{Allocation Policy}\label{sec:normalpolicy}
In this section, we propose a new FCBA algorithm based on the proposed approximation. It takes the explicit form of $G_{j}(\cdot, \cdot)$, $\lambda_{j}^{*}p_{j}\tilde{\sigma}_{1,j}$ and $c_{j,l}$'s in Theorem \ref{thm:main} to develop allocation policies. Henceforth, we will stick to the Gaussian case in Example 1 for policy development. It is worth noting that our method can be extended to other sampling distributions.

\subsection{Optimality Conditions}
For $\ell\geq 0$, define $U_{\ell}: \mathbb{R}_{+}^{2}\rightarrow \mathbb{R}_{+}$ by \begin{equation}\label{eq:defineul}
    U_{\ell}(x) = \exp\{-\frac{1}{2}Tx - \frac{1}{2}\ln{x}\}\cdot (1 + \sum_{l=1}^{\ell}{\frac{(-1)^{l}(2l-1)!!}{x^{l}}}\frac{1}{T^{l}} )
\end{equation} and 
\begin{equation}\label{eq:definerj}
    R_{j}(p_{1}, p_{j}) = (m_{1} - m_{j})^{2}/(\sigma_{1}^{2}/p_{1} + \sigma_{j}^{2}/p_{j}),\quad \forall j\in[k]_{-}.
\end{equation} 
For Gaussian distributions, the polynomials $p_{3\nu-1,T}$ are identically equal to $0$ and simple calculation yields that $c_{j,l}=(-1)^{l}(2l-1)!!/(\lambda_{j}^{*}p_{j}\tilde{\sigma}_{1,j})^{2l}$. Recall in Example \ref{example:Gaussian}, we have $\lambda_{j}^{*}p_{j}\tilde{\sigma}_{1,j} = \sqrt{R_{j}(p_{1}, p_{j})}$. Then we see that $V_{\ell}(\bm{p}) := \sum_{j=2}^{k}{U_{\ell}(R_{j}(p_{1}, p_{j}))}$ is an approximation to $1 - P\{CS\}$ of order $\ell$ implied by Theorem \ref{thm:main}. It is worth mentioning that the constants $c_{j,l}$ are typically computationally intractable except for the Gaussian case. However, we will empirically show that the approximation of order $0$ is good enough for static allocation rules. In the following lemma, we show the convexity property of $V_{\ell}(\bm{p})$.

\begin{lemma}\label{lem:concavity}
    For any even $\ell\geq 0$, $V_{\ell}(\bm{p})$ is a strongly convex function of $\bm{p}$. For any odd $\ell\geq 1$, $V_{\ell}(\bm{p})$ is asymptotically almost convex of $\bm{p}$, in a sense that there exists a sequence of sets $(E_{T})_{T\geq 1}$, such that the interiors of the complements of $E_T$ in $\{\bm{p}\geq 0: p_{1}+\dots +p_{k} = 1\}$ converge to the empty set as $T\rightarrow \infty$, and for any $T\geq 1$, $V_{\ell}(\bm{p})$ is strongly convex in $E_{T}$.
\end{lemma}

The proof can be found in the preliminary version of this paper \citep{shi2024finite}. Actually, it is shown that $V_{\ell}(\bm{p})$ is identically non-negative only when $\ell$ is even. In contrast, $V_{\ell}(\bm{p})$ approaches infinity if any entry of $\bm{p}$ tends to zero from above, which explains its non-concavity outside $E_{T}$. Lemma \ref{lem:concavity} implies an edge of using even order approximations over odd order ones. Following the asymptotic convexity of $V_{\ell}(\bm{p})$, we can derive an asymptotically optimal allocation ratio by the minimization problem
\begin{equation} \label{opt:opt1}
    \begin{aligned}
    &\min_{\bm{p}\geq 0} && V_{\ell}(\bm{p}) \\
    &s.t. && \sum_{i\in[k]}p_{i} = 1.
\end{aligned}
\end{equation}

The following proposition establishes the optimality conditions to (\ref{opt:opt1}).
\begin{proposition}\label{prop:condition}
    For $\ell\geq 0$ fixed, the optimal solution to the minimization problem (\ref{opt:opt1}) satisfies
    \begin{equation}\label{eq:conditions}
        \begin{cases}
            U_{\ell}'(R_{i}(p_{1}, p_{i}))\cdot R_{i}(p_{1}, p_{i})\dfrac{\sigma_{i}^{2}/p_{i}^{2}}{\sigma_{1}^{2}/p_{1}+\sigma_{i}^{2}/p_{i}} = U_{\ell}'(R_{j}(p_{1}, p_{j}))\cdot R_{j}(p_{1}, p_{j})\dfrac{\sigma_{j}^{2}/p_{j}^{2}}{\sigma_{1}^{2}/p_{1}+\sigma_{j}^{2}/p_{j}},\ \forall~2\leq i, j\leq k, \\
            \dfrac{p_{1}^{2}}{\sigma_{1}^{2}} = \sum\limits_{j\in[k]_{-}}\dfrac{p_{j}^{2}}{\sigma_{j}^{2}}.
        \end{cases}
    \end{equation}
    If $\ell$ is even, then the above equations are not only necessary but also sufficient.
\end{proposition}

\begin{proof}{Proof of Proposition \ref{prop:condition}.}
    For any $\ell\geq 0$, let $L(\bm{p}, \lambda) = V_{\ell}(\bm{p}) + \lambda(1 - \sum_{i\in[k]}p_{i})$ be the Lagrangian function. Since the problem (\ref{opt:opt1}) satisfies the linearity constraint qualification, there is no duality gap for this problem. The KKT conditions imply that $\bm{p}$ is optimal if there exists $\lambda\in\mathbb{R}$, such that
    \begin{gather}
        \frac{\partial}{\partial p_{1}}V_{\ell}(\bm{p}) - \lambda = 0,\label{eq:gather1}\\
        \frac{\partial}{\partial p_{j}}V_{\ell}(\bm{p}) - \lambda = 0,\quad \forall j\in[k]_{-}, \label{eq:gather2}\\
        1 - \sum_{i\in[k]}p_{i} = 0. \label{eq:gather3}
    \end{gather}
    It follows from (\ref{eq:gather1}), (\ref{eq:gather2}), (\ref{eq:gather3}) and the definition of $V_{\ell}(\bm{p})$ that 
    \begin{equation*}
    \begin{cases}
        U_{\ell}'(R_{i}(p_{1}, p_{i}))\cdot \frac{\partial}{\partial p_{i}}R_{i}(p_{1}, p_{i}) = U_{\ell}'(R_{j}(p_{1}, p_{j}))\cdot \frac{\partial}{\partial p_{i}}R_{j}(p_{1}, p_{j}),\ \forall~2\leq i, j\leq k, \\
        \sum\limits_{j\in[k]_{-}}\left(\frac{\partial}{\partial p_{1}}R_{j}(p_{1}, p_{j}) / R_{j}(p_{1}, p_{j}) - 1\right)\Big/\left(\frac{\partial}{\partial p_{j}}R_{j}(p_{1}, p_{j}) / R_{j}(p_{1}, p_{j})\right) = 0.
    \end{cases}
    \end{equation*}
    For Gaussian alternatives, we have $R_{j}(p_{1}, p_{j}) = (m_{1} - m_{j})^{2}/(\sigma_{1}^{2}/p_{1} + \sigma_{j}^{2}/p_{j})$. Therefore, $\frac{\partial}{\partial p_{1}}R_{j}(p_{1}, p_{j}) = R_{j}(p_{1}, p_{j})\cdot \sigma_{1}^{2}/p_{1}^{2}/(\sigma_{1}^{2}/p_{1} + \sigma_{j}^{2}/p_{j})$ and $\frac{\partial}{\partial p_{j}}R_{j}(p_{1}, p_{j}) = R_{j}(p_{1}, p_{j})\cdot \sigma_{j}^{2}/p_{j}^{2}/(\sigma_{1}^{2}/p_{1} + \sigma_{j}^{2}/p_{j})$. Then the first equation is equivalent to
    $$U_{\ell}'(R_{i}(p_{1}, p_{i}))\cdot R_{i}(p_{1}, p_{i})\frac{\sigma_{i}^{2}/p_{i}^{2}}{\sigma_{1}^{2}/p_{1}+\sigma_{i}^{2}/p_{i}} = U_{\ell}'(R_{j}(p_{1}, p_{j}))\cdot R_{j}(p_{1}, p_{j})\frac{\sigma_{j}^{2}/p_{j}^{2}}{\sigma_{1}^{2}/p_{1}+\sigma_{j}^{2}/p_{j}},\ \forall~2\leq i, j\leq k,$$
    and the second equation can be further simplified into 
    $$\frac{p_{1}^{2}}{\sigma_{1}^{2}} = \sum_{j\in[k]_{-}}\frac{p_{j}^{2}}{\sigma_{j}^{2}}.$$

    Specially, for $\ell\geq 0$ even, it follows from Lemma \ref{lem:concavity} that $V_{\ell}(\bm{p})$ is a convex program. Therefore, the KKT conditions are also sufficient. If $\bm{p}$ satisfies the system of conditions (\ref{eq:conditions}), then equalities (\ref{eq:gather1}), (\ref{eq:gather2}) and (\ref{eq:gather3}) hold for $\lambda = U_{\ell}'(R_{k}(p_{1}, p_{k}))\cdot \frac{\partial}{\partial p_{k}}R_{k}(p_{1}, p_{k})$ and thus $\bm{p}$ is optimal. \halmos
\end{proof}

\begin{remark}
The optimality conditions are comparable to those of OCBA and ROA. A sampling policy is said to asymptotically achieve the ROA if the asymptotic sampling ratio $\bm{p}$ maximizes the large deviations rate $\min_{j\in[k]_{-}}G_{j}(p_{1}, p_{j})$. The ROA is unique and is equivalent to the optimality conditions 
\begin{equation}\label{eq:roacondition}
    \begin{cases}
        R_{i}(p_{1}, p_{i})\equiv \dfrac{(m_{1} - m_{i})^{2}}{\sigma_{1}^{2}/p_{1} + \sigma_{i}^{2}/p_{i}} = \dfrac{(m_{1} - m_{j})^{2}}{\sigma_{1}^{2}/p_{1} + \sigma_{j}^{2}/p_{j}} \equiv R_{j}(p_{1}, p_{j}),\ \forall~2\leq i, j\leq k, \\
        \dfrac{p_{1}^{2}}{\sigma_{1}^{2}} = \sum\limits_{j\in[k]_{-}}\dfrac{p_{j}^{2}}{\sigma_{j}^{2}}.
    \end{cases}
\end{equation}
The second condition, which balances the trade-off between the best alternative and sub-optimal alternatives, aligns with that in (\ref{eq:conditions}). In contrast, the first condition balances the allocation among sub-optimal alternatives. Both sides can be interpreted as a \textit{score function} of sampling a certain alternative, equating to the marginal improvement of the PCS with respect to the sampling ratio of the corresponding alternative. Intuitively, an undersampled alternative has a large score in absolute value and thus deserves more simulation samples. For the ROA, the score function is given as $R_{j}(p_{1}, p_{j})$ for alternative $j\in [k]_{-}$. The OCBA formula is similar to (\ref{eq:roacondition}) with the mere difference in the score function, i.e., $(m_{1}-m_{j})^{2}/\sigma_{j}^{2}/p_{j}$, which is an accurate approximation to $R_{j}(p_{1}, p_{j})$ when $p_{1}\gg p_{j}$. Note that, for the ROA, the score function $R_{j}(p_{1}, p_{j}) = 2G_{j}(p_{1}, p_{j})$ characterizes the rate at which the probability of pairwise incorrect selection decays in the long run. On the contrary, the proposed approximation provides a novel score function, i.e., $\vert U_{\ell}'(R_{j}(p_{1}, p_{j}))\vert \cdot R_{j}(p_{1}, p_{j}) \cdot \sigma_{j}^{2}/p_{j}^{2}/(\sigma_{1}^{2}/p_{1} + \sigma_{j}^{2}/p_{j})$, which discounts the long run marginal improvement using a decreasing function $U_{\ell}'(\cdot)$ that depends on the simulation budget and a weighting coefficient $\sigma_{j}^{2}/p_{j}^{2}/(\sigma_{1}^{2}/p_{1} + \sigma_{j}^{2}/p_{j})$ that concerns the variances of alternatives. Consequently, the proposed score function takes the finite sample behavior of PCS into account.
\end{remark}

A function $f(x)$ is called a rational polynomial if it can be expressed as the quotient of two polynomials. The following lemma serves as an intermediary step in establishing the uniqueness of solutions to the optimal allocation problem associated with the proposed PCS approximation.
\begin{lemma}\label{lem:extrapolation}
    Suppose $Q(x)$ and $R(x)$ are two rational polynomials and let $a<b$ be two reals. If $\exp\{Q(x)\} = R(x)$ holds for $x\in(a, b)$, then the equality holds for $x\in\mathbb{R}$.
\end{lemma}

The proof can be found in the electronic companion. In fact, we can show that $Q(x)$ and $R(x)$ must be constants. Below, the uniqueness of optimal allocations follows.

\begin{proposition}\label{prop:unique}
    For $\ell\geq 0$ even, the solution to the optimal conditions (\ref{eq:conditions}) is unique. Therefore, the optimal solution to (\ref{opt:opt1}) is also unique.
\end{proposition}

For $\ell\geq 0$ even, define $\bm{p}^{(\ell,T)}$ as the solution to (\ref{eq:conditions}) with a simulation budget of $T$. Proposition \ref{prop:unique} guarantees that $\bm{p}^{(\ell,T)}$ is well-defined. Note that the optimality condition (\ref{eq:conditions}), if taken logarithm on both sides, converge uniformly with regard to $\bm{p}$ to (\ref{eq:roacondition}) on compact sets within the interior of the feasible region. By showing that each entry of $\bm{p}^{(\ell, T)}$ is bounded away from $0$, we formalizes the following result.
\begin{theorem}\label{prop:asympallocation}
    Given $\ell\geq 0$ even, the approximate optimal allocation in (\ref{eq:conditions}) is asymptotically rate optimal as the simulation budget increases, i.e., 
    $$\lim_{T\rightarrow\infty}\bm{p}^{(\ell, T)} = \bm{p}^{*},$$
    where $\bm{p}^{*}$ is the unique ROA that solves (\ref{eq:roacondition}).
\end{theorem}

Based on Proposition \ref{prop:condition}, we propose the FCBA in Algorithm \ref{alg:ocba}. We name it FCBA($\ell$) since it utilizes the PCS approximation which retains the $\ell$-th order term in the expansion in Proposition \ref{prop:exponential}. In each step, FCBA determines a sub-optimal alternative with the largest score as a candidate for sampling. Then FCBA allocates a sample either to this candidate or to the estimated best alternative according to the second condition in (\ref{eq:conditions}).

\begin{algorithm}
\caption{Finite Computing Budget Allocation of Order $\ell$ (FCBA($\ell$))}
\label{alg:ocba}
\begin{algorithmic}[1]
\STATE Initialize: Set budget $T$, number of alternatives $k$, initial replications $T_{0}$.
\STATE Initialize replications $T_{i} \leftarrow T_{0}$ for $i = 1, 2, \dots, k$.
\STATE Simulate $T_{i}$ replications for each alternative $i$.
\STATE Total cost $N \leftarrow k\times T_{0}$.
\WHILE{Total cost $< T$}
    \STATE Estimate the best alternative $j^{*} \leftarrow \argmax_{j\in[k]}\hat{m}_{j}$.
    \STATE Use plug-in estimation of (\ref{eq:definerj}) based on sample means and sample variances to calculate \\ 
    $j' \leftarrow \argmax_{j\in[k]\backslash\{j^{*}\}} U'_{\ell}(\hat{R}_{j}(T_{j^{*}}/N, T_{j}/N))\cdot \hat{R}_{j}(T_{j^{*}}/N, T_{j}/N) \cdot \dfrac{\hat{\sigma}_{j}^{2}/T_{j}^{2}}{\hat{\sigma}_{j^{*}}^{2}/T_{j^{*}}+\hat{\sigma}_{j}^{2}/T_{j}}$. 
    \IF{$T_{j^{*}}^{2}/\hat{\sigma}_{j^{*}}^{2} > \sum_{j\in[k]\backslash\{j^{*}\}}T_{j}^{2}/\hat{\sigma}_{j}^{2}$}
    \STATE Simulate one additional replication for alternative $j^{\prime}$.
    \STATE Update $T_{j^{\prime}} \leftarrow T_{j^{\prime}} + 1$.
    \ELSE
    \STATE Simulate one additional replication for alternative $j^{*}$.
    \STATE Update $T_{j^{*}} \leftarrow T_{j^{*}} + 1$.
    \ENDIF
    \STATE Update $N\leftarrow N+1$.
\ENDWHILE
\STATE Select best-performing alternative based on sample means.
\end{algorithmic}
\end{algorithm}

\section{Experiments}
\label{sec:experiment}
In Section \ref{sec:experiment}, we present numerical experiments to evaluate the performance of FCBA policies, assuming Gaussian distributions for simulation outputs. Two different configurations for the means of the alternatives are tested: (\MakeUppercase{i}) the \textbf{stepping} configuration where $m_{i} = 0.1 * (k + 1 - i)$ for $i=1,2,\dots,k$; and (\MakeUppercase{ii}) the \textbf{noisy} configuration where $m_{i} \overset{i.i.d.}{\sim} 0.1 * k * \mathrm{Uniform}([0, 1])$. For the variances of alternatives, we consider three settings: ({i}) the \textbf{equal variance} setting where all variances are fixed, unless otherwise stated, at $4$; ({ii}) the \textbf{increasing variance} setting where the alternatives are ranked by means and divided equally into five groups, and the variances of the top 20\% of alternatives to the bottom 20\% of alternatives are 2, 3, 4, 5, and 6, respectively; and ({iii}) the \textbf{decreasing variance} setting where variances are assigned in reverse order with values of 6, 5, 4, 3, and 2 from the top 20\% of alternatives to the bottom 20\% of alternatives, respectively. We will choose a combination of configurations of means and of variances.

\subsection{FCBA($\ell$) versus OCBA}\label{sec:empproperty}
\begin{figure}
    \centering
    \includegraphics[width = .455\linewidth]{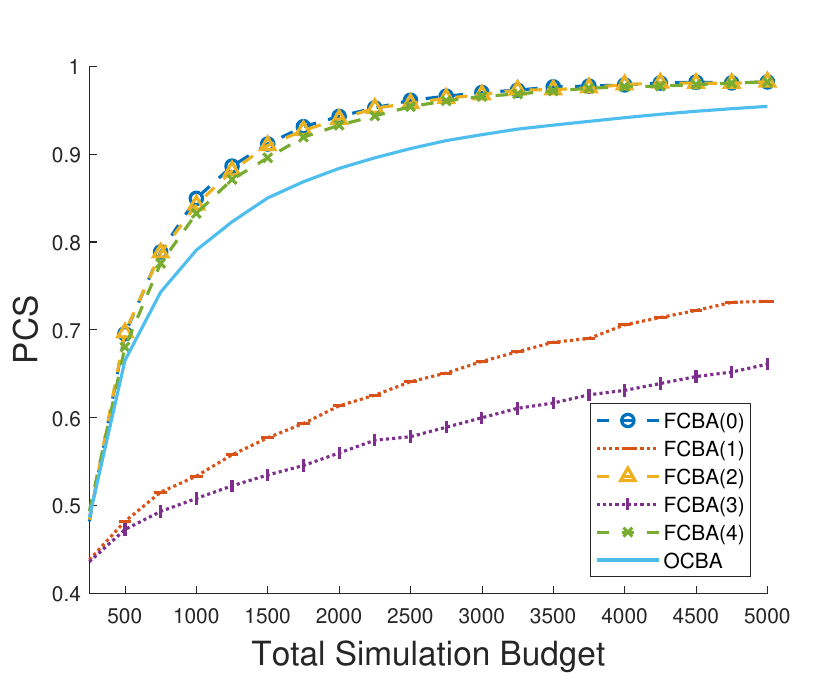}
    \includegraphics[width = .515\linewidth]{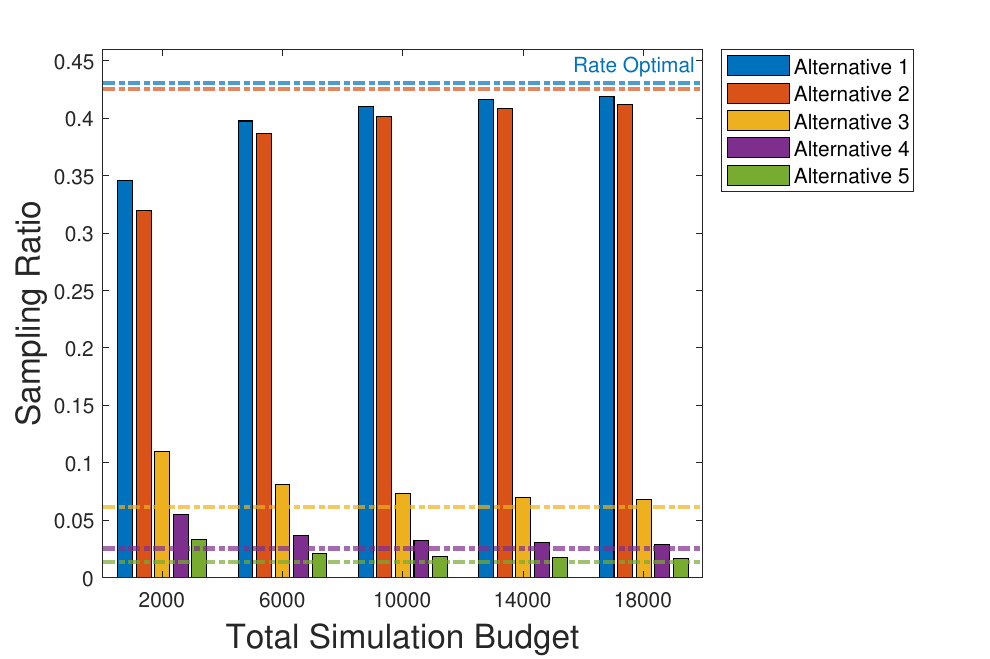}
    \caption{Left: PCS of FCBA($\ell$) versus OCBA under varying number of total simulation budget based on $5,000$ macro-replications. Right: Theoretically optimal allocation based on $V_{0}(\bm{p})$.}
    \label{fig:ex1_1}
\end{figure}

We conducted initial comparisons between FCBA($\ell$) policies and OCBA using a test case with $k=50$ alternatives. The means follow the stepping setting, and variances are set at an equal value of 1. The left panel of Figure \ref{fig:ex1_1} demonstrates the robust advantage of FCBA($\ell$) with even values of $\ell$ over OCBA. However, FCBA($\ell$) policies with odd values of $\ell$ show significantly lower performance, likely due to the non-convex nature of $V_{\ell}(\bm{p})$ when $\ell$ is odd. Since FCBA($\ell$) allocates samples by following the descending gradient of non-convex function $V_{\ell}(\bm{p})$, there is no guarantee of convergence of the sampling ratios to a global minimum. Among the policies of even orders tested, FCBA($0$) and FCBA($2$) perform similarly, with both slightly outperforming FCBA($4$) when total budgets are limited. These results suggest that FCBA($0$) is sufficient for practical applications. Hence, we will focus on the FCBA of order $\ell = 0$ in subsequent analysis. It is worth noting that, according to Proposition \ref{prop:exp-trivariate}, FCBA ignores the probability of simultaneous incorrect selection in (\ref{ex-inclusion}). However, the ignored terms, for example, $P(\bar{X}_{1}(T_{1})\leq \bar{X}_{i}(T_{i})\wedge \bar{X}_{j}(T_{j}))$, can be larger than the higher-order correction terms in absolute values under limited budgets, and it explains the superiority of low-order FCBA($\ell$) policies. Nevertheless, our approximation remains asymptotically valid. Deriving higher-order algorithms would require further characterization of this ignored probability, which we defer to future work.

We proceed by examining the behavior of FCBA($0$) through its corresponding theoretically optimal allocation, which is defined as the sampling ratio that maximizes $V_{0}(\bm{p})$ under the assumption of known distributional parameters. The right panel of Figure \ref{fig:ex1_1} presents the theoretically optimal allocation for the first five alternatives, where dashed lines indicate the rate-optimal allocation that maximizes the LDR rate. With a limited budget of $T$, FCBA($0$) adopts a more conservative approach, allocating fewer samples to the best alternative compared to OCBA. However, as $T$ tends toward infinity, the allocation strategy of FCBA($0$) converges to that of OCBA, thereby empirically confirming the asymptotic rate-optimality of FCBA($0$).

\subsection{Finite Sample Performances}\label{sec:finitesample}
The finite sample performance of FCBA($0$) consistently surpasses that of all modern R\&S procedures in comparison, including equal allocation (EA), OCBA, ROA, AOAP \citep{peng2016dynamic}, and modified complete expected improvement (mCEI; \citealp{chen2019complete}), a variant of EI tailored for rate optimality. Table \ref{tab:ex1} displays the final PCS for each combination of policies and problem instances. Each instance contains $k = 50$ alternatives, with a simulation budget set at $T = 1,000$.

\begin{table}[htbp]
\centering
\caption{PCS estimated by $100,000$ macro-replications for all instances with $k = 50$ and $T = 1,000$.}\label{tab:ex1}
\begin{tabular}{ccccccc}
\toprule
Instances & \textbf{FCBA($0$)} & AOAP   & ROA    & OCBA   & mCEI   & EA     \\
\midrule
\textbf{Stepping} + \textbf{Equal}                      & \textbf{0.5767}    & 0.5537 & 0.5487 & 0.5641 & 0.5652 & 0.3027 \\
\textbf{Stepping} + \textbf{Increasing}                & \textbf{0.7042}    & 0.6702 & 0.6659 & 0.6847 & 0.6669 & 0.3805 \\
\textbf{Stepping} + \textbf{Decreasing}                 & \textbf{0.5050}    & 0.4902 & 0.4851 & 0.4975 & 0.5045 & 0.2641 \\
\textbf{Noisy} + \textbf{Equal}                        & \textbf{0.9155}    & 0.8140 & 0.7899 & 0.8129 & 0.8744 & 0.4897 \\
\textbf{Noisy} + \textbf{Increasing}                   & \textbf{0.9698}    & 0.8965 & 0.8838 & 0.8993 & 0.9308 & 0.6819 \\
\textbf{Noisy} + \textbf{Decreasing}                   & \textbf{0.8468}    & 0.7502 & 0.7206 & 0.7459 & 0.8194 & 0.3971 \\
\bottomrule
\end{tabular}
\end{table}

It is important to note that FCBA($0$) is designed to maximize the final PCS and thus is dependent on the total budget. The following experiment evaluates the PCS of FCBA($0$) against the competing policies in an instance with means in the \textbf{stepping} and with variances in the \textbf{equal} setting, under the condition that the best alternative is identified before the simulation budget is fully utilized. The effectiveness of FCBA($0$) is illustrated in Figure \ref{fig:ex_2}, with the only exception being its performance compared to AOAP when the simulation budget is much limited. This difference arises because AOAP aims for a nearly optimal dynamic policy, whereas FCBA($0$) focuses solely on maximizing the final PCS. The good performance of FCBA in scenarios where the budget is nearly exhausted supports the robust extension of FCBA policies to applications where the total budget is not entirely predetermined, such as in fixed-precision settings. Figure \ref{fig:ex_2} also highlights the scalability of FCBA. In this experiment, three initial simulation replications are allocated to each alternative. Even with a fixed simulation budget, the difference in PCS between FCBA($0$) and other benchmarks grows as the problem size increases. To be specific, the gaps in the final PCS are $0.26\%$, $0.62\%$, $1.19\%$, $2.10\%$ for $k = 50, 100, 200, 400$, respectively. 

\begin{figure}
    \centering
    \includegraphics[width = .45\linewidth]{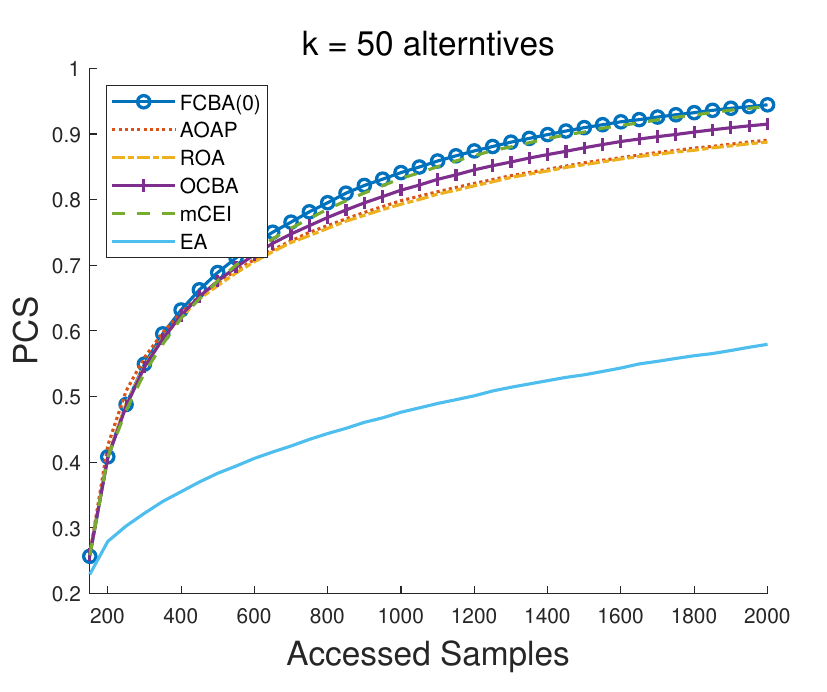}
    \includegraphics[width = .45\linewidth]{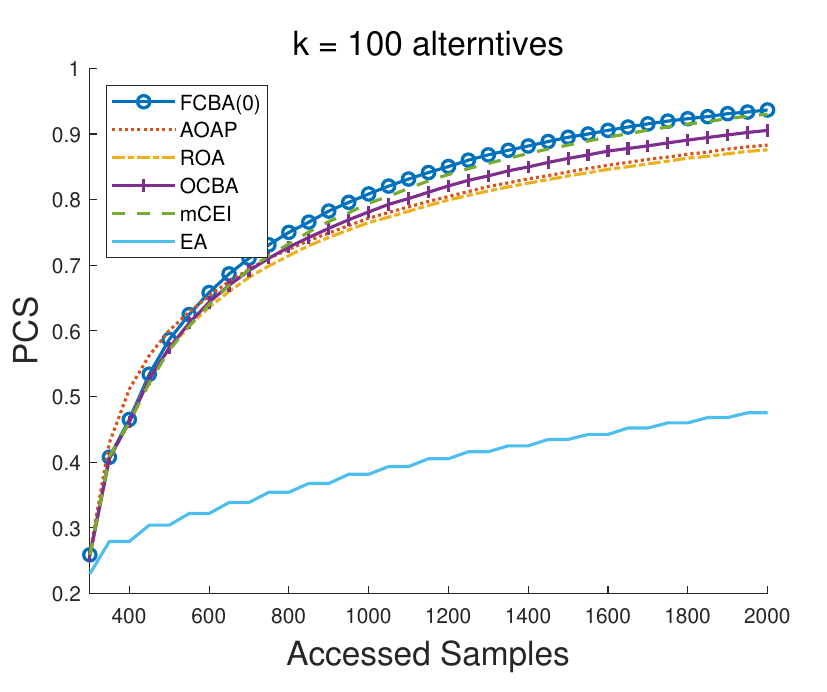}
    \includegraphics[width = .45\linewidth]{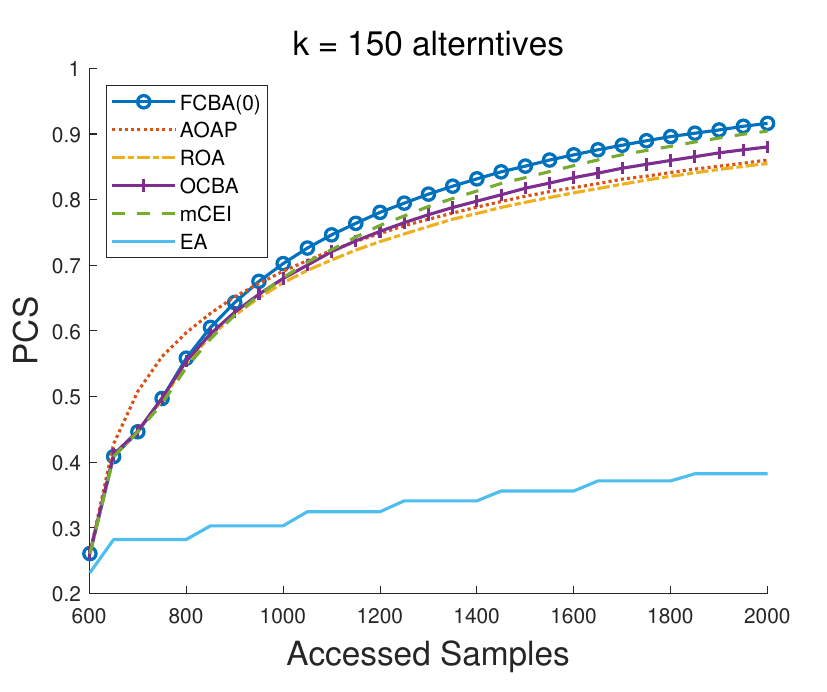}
    \includegraphics[width = .45\linewidth]{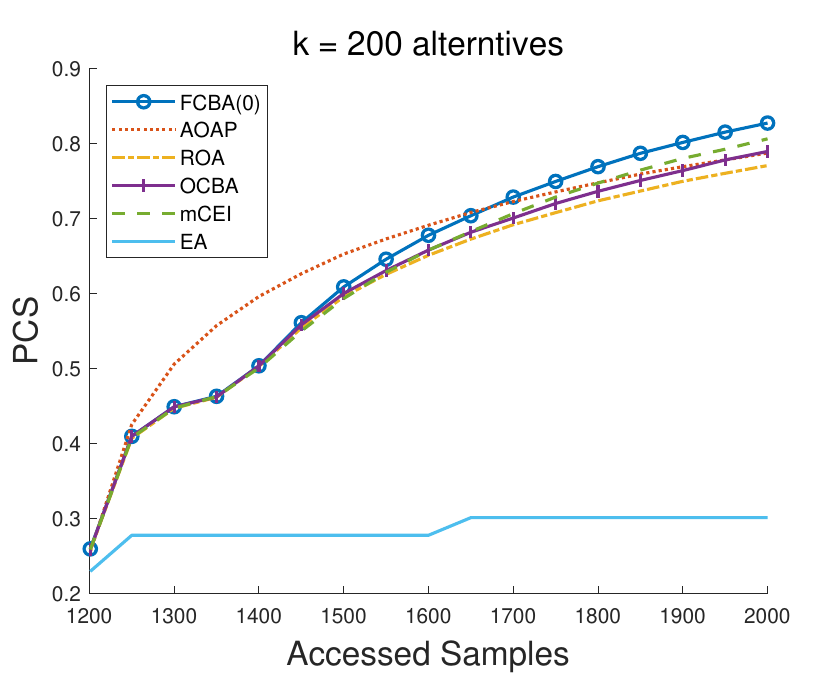}
    \caption{PCS before the budget is exhausted estimated by $100,000$ macro-replications in an instance with means in the \textbf{stepping} and with variances in the \textbf{equal} setting.}
    \label{fig:ex_2}
\end{figure}

Finally, we demonstrate the potential of our approximation as a rough estimate for PCS. In the instance with $k = 100$ mentioned earlier, we implement FCBA($0$) across various total budgets, ranging from $T = 1,000$ to $5,000$. Upon exhausting the simulation budget, we construct a post hoc estimation for the PCS using $1 - \hat{V}_{0}(\bm{p})$, where $\hat{V}_{0}(\bm{p}):= \sum_{j'\neq j^{*}}{\exp\{-\frac{1}{2}\hat{R}_{j'}(p_{j^{*}}, p_{j'})T - \frac{1}{2}\ln{\hat{R}_{j'}(p_{j^{*}}, p_{j'})}\}}$. Here, $j^{*} := \argmax_{1\leq j\leq k}\hat{m}_{i}$, $\hat{R}_{j'}(p_{j^{*}}, p_{j'}) := (m_{j^{*}} - m_{j'})^{2} / (\hat{\sigma}^{2}_{j^{*}}/p_{j^{*}} + \hat{\sigma}^{2}_{j'}/p_{j'})$. The values $\hat{m}_{i}$ and $\hat{\sigma}_{i}^{2}$ represent sample means and variances, respectively. Consistent with our findings from the initial experiment, the results presented in Table \ref{tab:estPCS} indicate that $\hat{V}_{0}(\bm{p})$ tends to be conservative, generally underestimating the PCS. However, the accuracy of this estimation improves as the sample size grows. In contrast, the estimation using LDR, i.e., $1-\hat{V}_{\mathrm{LDR}}(\bm{p}):=1-\exp\{-\frac{1}{2}\min_{j'\neq j^{*}}\hat{R}_{j'}(p_{j^{*}},p_{j'})\}$, exhibits greater mean error and higher variance than $1 - \hat{V}_{0}(\bm{p})$ for $T\geq 2,000$. For a small budget, such as $T=1,000$, $1-\hat{V}_{0}(\bm{p})$ still shows a smaller mean error but a slightly higher variance. However, for $T = 1,000$, the true PCS is not covered by the interval centered at the mean of $1 - \hat{V}_{\mathrm{LDR}}(\bm{p})$ with a radius equal to one standard error, implying a large bias of $1 - \hat{V}_{\mathrm{LDR}}(\bm{p})$. Conversely, for all tested budgets $T$, the true PCS is consistently covered by the interval centered at the mean of $1 - \hat{V}_{0}(\bm{p})$ with a radius equal to one standard error.

\begin{table}[htbp]
\centering
\caption{The true PCS and the statistical characteristics of the estimation $\hat{V}_{0}(\bm{p})$, estimated via $100,000$ macro-replications.}\label{tab:estPCS}
\begin{tabular}{cccccc}
    \toprule
    $T$ & 1,000 & 2,000 & 3,000 & 4,000 & 5,000 \\
    \midrule
    True PCS & 0.8485    & 0.9437    & 0.9706    & 0.9789    & 0.9826 \\
    Mean of $1-\hat{V}_{0}(\bm{p})$ & 0.6381    & 0.8336    & 0.9205    & 0.9605    & 0.9801 \\
    $\pm$ Standard Error  & (0.3586)  & (0.2667)  & (0.1799)  & (0.1199)  & (0.0789) \\
    Mean of $1-\hat{V}_{\mathrm{LDR}}(\bm{p})$  & 0.5119    & 0.7087    & 0.8282    & 0.8992    & 0.9415 \\
    $\pm$ Standard Error    & (0.3282)    & (0.2999)    & (0.2378)    & (0.1797)    & (0.1314) \\
    \bottomrule
\end{tabular}
\end{table}

\section{Extensions of FCBA}\label{sec:extensions}
In this section, we first discuss a specific problem setting known as the low-confidence scenario where the budget is relatively inadequate. In Section \ref{sec:lowcon}, we propose a refined approximation to PCS and design a tailored algorithm for these scenarios. In Section \ref{sec:explowcon}, we present numerical experiments to demonstrate the performance of the refined policy. Thereafter, we mention the possibility of leveraging the proposed theory to design fully sequential allocation policies by characterizing a conditional probability in Section \ref{sec:conditionalprob}. Detailed proofs can be found in the online supplement.

\subsection{Low-Confidence Scenarios}\label{sec:lowcon}
For years R\&S literature has been focused on high-confidence scenarios, where the distribution of alternatives are well-separated and the sample size is relatively large. The experience in the literature indicates the best alternative should receive a lot more samples than others. For example, the LDP yields that, for Gaussian distributions, an asymptotically optimal allocation satisfies $p_{1}^{2} = \sum_{i=2}^{k}p_{i}^{2}$ and thus $p_{1}\gg p_{i}$ for any $i\geq2$. However, the optimal sampling ratio that maximizes $P\{CS\}$ for finite sample size may not necessarily have this property. \cite{peng2015nonmonotonicity} and \cite{peng2018gradient} have brought attention to a low-confidence scenario, which is qualitatively described by three characteristics: the differences between the means of competing alternatives are small, the variances are large, and the simulation budget is small. Loosely speaking, the low-confidence scenario occurs when the budget is relatively small, which is common in the case of a fixed finite budget. \cite{peng2015nonmonotonicity} demonstrates a counterintuitive phenomenon that $P\{CS\}$ may not be monotonically increasing as more samples are given to the best alternative in these cases, implying that the best alternative should receive fewer samples than the LDP-based optimal allocation.

\subsubsection{A Refined Expansion}
Although the optimal allocation based on (\ref{eq:roacondition}) also suffers from the same property that $p_{1}^{2} = \sum_{i=2}^{k}p_{i}^{2}$, we can refine our approximation expansion to avoid it using the complete inclusion-exclusion principle. To be specific, recall that we truncate equation (\ref{ex-inclusion}) till all terms involving binary comparisons by Proposition \ref{prop:exp-trivariate}. As a result, the approximation (\ref{eq:transit}) is too conservative as it neglects the probability of simultaneous incorrect binary comparisons. To circumvent this issue, we introduce a proposition which extends Proposition \ref{prop:exp-trivariate}. Suppose $S\subseteq [k]$ is a subset of alternatives and $\bm{x}_{S}=(x_{i})_{i\in S}$ is a real vector. For subsets satisfying $1\in S$, define $I_{S}(\bm{x}_{S}) = \sum_{i\in S}{p_{i}I_{i}(x_{i})}$ and let $\bm{x}_{S}^{*}$ be given by $I_{S}(\bm{x}_{S}^{*}) = \inf_{\{\bm{x}_{S}\in\mathbb{R}^{\vert S\vert}: x_{1}\leq x_{j}, \forall j\in S\}}I_{S}(\bm{x}_{S})$.
\begin{lemma}\label{lem:critical}
    Suppose $n\geq 1$ and $2\leq i_{1}<i_{2}<\dots<i_{n}\leq k$ and let $S = \{1,i_{1}, i_{2}, \dots, i_{n}\}$. Under Assumptions \ref{ass:light} and \ref{ass:btv}, the infimum of $I_{S}(\bm{x}_{S})$ over $\{\bm{x}_{S} \in \mathbb{R}^{\vert S \vert}: x_{1} \leq x_{j},~\forall j \in S\}$ is uniquely achieved at $\bm{x}_{S}^{*}=(x_{i}^{*})_{i\in S}$, which satisfies $x_{j}^{*} = x_{1}^{*} \vee m_{j}$ for all $j \in S$.
\end{lemma}

Lemma \ref{lem:critical} indicates that the problem of finding the minimum of $I_{S}(\bm{x})$ over $\{\bm{x}_{S} \in \mathbb{R}^{\vert S \vert}: x_{1} \leq x_{j},~\forall j \in S\}$ can be reformulated as finding the minimum of $J_{S}(x_{1}) := I_{S}(x_{1}, x_{1}\vee m_{i_{1}}, \dots, x_{1}\vee m_{i_{n}})$ within the domain $x_{1}\in [m_{k}, m_{1}]$ and we have $J_{S}(x_{1}^{*})=I_{S}(\bm{x}_{S}^{*})$. On the other hand, for $j\neq 1$, $j\in S$, if $m_{j}\geq x_{1}^{*}$, then $I_{j}(x_{j}^{*})=I_{j}(x_{1}^{*}\vee m_{j})=I_{j}(m_{j})=0$ and thus the term $p_{j}I_{j}(x_{j})$ does not contribute to the rate function $I_{S}(\bm{x}_{S})$. Hence, $x_{1}^{*}$ can be reckoned a \textit{critical point}, and roughly speaking, a small critical point implies a small objective $I_{S}(\bm{x}_{S}^{*}) = \sum_{i\in S}{p_{i}I_{i}(x_{i}^{*})}$ since the number of positive terms is small.
\begin{proposition}\label{prop:anyterm}
    Suppose $n\geq 1$ and $2\leq i_{1}<i_{2}<\dots<i_{n}\leq k$ and let $S = \{1,i_{1}, i_{2}, \dots, i_{n}\}$. Under Assumptions \ref{ass:light} and \ref{ass:btv}, the probability of simultaneous incorrect binary comparisons for multiple alternatives decays exponentially, i.e., there exist an integer $l_{S}\geq 1$ and a real $c_{S}>0$ such that
    \begin{equation*}
        P\left(\bar{X}_{1}(T_{1}) \leq \bar{X}_{i_{1}}(T_{i_{1}})\wedge \bar{X}_{i_{2}}(T_{i_{2}})\wedge \dots\wedge \bar{X}_{i_{n}}(T_{i_{n}})\right) = \exp\{-TJ_{S}(x_{1}^{*})\} \cdot \frac{c_{S}}{\sqrt{T}^{l_{S}}} \cdot \left(1 + \mathcal{O}(T^{-1/2})\right).
    \end{equation*}
\end{proposition}%

Proposition \ref{prop:anyterm} extends Proposition \ref{prop:exp-trivariate} by characterizing Bahadur-Rao type expansions of probabilities of SIBC for multiple alternatives. We will see in the following proposition that the optimal objective value $I_{S}(\bm{x}_{S}^{*})$ is the exponential rate at which the probability of SIBC decays.
\begin{remark}
    In general, $c_{S}$ can be expressed as the expectation of an exponentiated quadratic form of a Gaussian vector. Although $c_{S}$ is not mathematically tractable, it can be efficiently evaluated using numeric methods. 
\end{remark}
\begin{remark}
    In the low-confidence scenario, where the $T$ budget is relatively small, $\exp\{-TJ_{S}(x_{1}^{*})\}$ is approximately equal to 1. Consequently, the probability of SIBC behaves similarly to the polynomial term $1/\sqrt{T}^{l_{S}}$. In fact, the polynomial order is given by $l_{S} = n - \#\{i\in S\backslash\{1\}: m_{i} \geq x_{1}^{*}\}$. In the case of $m_{i} \leq x_{1}^{*}$ for all $i\in S\backslash\{1\}$, the probability of SIBC is of order $\mathcal{O}(T^{-n/2})$. This order decreases exponentially as the number of alternatives increases, aligning with the intuition that more alternatives reduce the likelihood of simultaneous incorrect comparisons. Furthermore, since each alternative with a mean greater than or equal to the critical point $x_{1}^{*}$ reduces the order $l_{S}$ by 1, only alternatives with means less than $x_{1}^{*}$ contribute to the order. Roughly speaking, if $m_{i}\geq x_{1}^{*}$, then there is a relatively high likelihood that alternative $i$ will be mistakenly identified as superior to alternative 1 within the set $S$.
\end{remark}

According to the inclusion-exclusion principle, we have the following expansion of the PICS:
\begin{equation}\label{eq:in-ex}
    1 - P\{CS\} = \sum_{n=1}^{k-1}(-1)^{n-1}\sum_{2\leq i_{1}<\dots<i_{n}\leq k}P\left(\bar{X}_{1}(T_{1}) \leq \bar{X}_{i_{1}}(T_{i_{1}})\wedge \bar{X}_{i_{2}}(T_{i_{2}})\wedge \dots\wedge \bar{X}_{i_{n}}(T_{i_{n}})\right).
\end{equation}
According to Proposition \ref{prop:anyterm}, each term in the summation can be asymptotically equivalently approximated by an exponential function in $T$ and a polynomial in $T^{-1/2}$. When the budget is relatively small, the polynomial term is predominant. The following result indicates that if we retain all of the terms in the approximation with polynomials in $T^{-1/2}$ of order $\mathcal{O}(T^{-1/2})$ and truncate the remainders, the approximation is still asymptotically equivalent.%

\begin{theorem}\label{thm:extended}
    Under Assumptions \ref{ass:light} and \ref{ass:btv}, for any $p_{i}>0$, $\forall i\in[k]$, satisfying $\sum_{i\in[k]}{p_{i}} = 1$, we have the following expansion
    \begin{equation*}
        1-P\{CS\} = \left(\sum_{n=1}^{k-1}(-1)^{n+1}\sum_{S\in \mathcal{S}_{n}}\exp\{-TJ_{S}^{*}(x_{1}^{*})\}\cdot\frac{c_{S}}{\sqrt{T}}\right)\cdot\left(1 + \mathcal{O}(T^{-1/2})\right),
    \end{equation*}
    where $\mathcal{S}_{n} := \{S = \{1,i_{1},\dots, i_{n}\}: 2\leq i_{1} < \dots<i_{n}\leq k,~l_{S} = 1\}$ is a family of subsets of alternatives.
\end{theorem}
\begin{proof}{Proof.}
    It follows from (\ref{eq:in-ex}) and Proposition \ref{prop:anyterm} that
    $$1 - P\{CS\} = \sum_{n=1}^{k-1}(-1)^{n-1}\sum_{\substack{2\leq i_{1}<\dots<i_{n}\leq k\\ S=\{1,i_{1},\dots,i_{n}\}}}\exp\{-TJ_{S}(x_{1}^{*}) \} \cdot \frac{c_{S}}{\sqrt{T}^{l_{S}}}\cdot \left(1 + \mathcal{O}(T^{-1/2})\right).$$
    According to Proposition \ref{prop:exponential}, for $j\in[k]_{-}$ and $S = \{1, j\}$, we have $l_{S} = 1$ and thus $S\in \mathcal{S}_{1}$. Suppose that for some $n\geq 2$, $S=\{1,i_{1},\dots,i_{n}\}\notin \mathcal{S}_{n}$, i.e., $l_{S} > 1$. We note that 
    $$J_{S}(x_{1}^{*}) = I_{S}(\bm{x}_{S}^{*}) = \inf_{\bm{x}:x_{1}\leq x_{i_{1}}\wedge \dots\wedge x_{i_{n}}}I_{S}(\bm{x}) \geq \inf_{ \bm{x}:x_{1}\leq x_{i_{1}} }I_{S}(\bm{x}).$$
    Since $I_{S}(\bm{x}) = p_{1}I_{1}(x_{1}) + p_{i_{1}}I_{i_{1}}(x_{i_{1}}) + p_{i_{2}}I_{i_{2}}(x_{i_{2}}) + \cdots + p_{i_{n}}I_{i_{n}}(x_{i_{n}})$, it follows from $I_{i}(x_{i}) \geq I_{i}(m_{i}) = 0$, $\forall i\in\{i_{2},\dots,i_{n}\}$ that 
    \begin{equation*}
        \inf_{ \bm{x}:x_{1}\leq x_{i_{1}} }I_{S}(\bm{x}) = \inf_{ (x_{1}, x_{i_{1}}):x_{1}\leq x_{i_{1}} }p_{1}I_{1}(x_{1}) + p_{i_{1}}I_{i_{1}}(x_{i_{1}}) = J_{1,i_{1}}(x_{1}^{*}).
    \end{equation*}
    In fact, $J_{1,i_{1}}(x_{1}^{*})$ is identical to the rate $I_{1, i_{1}}(\mu_{i_{1}}, \mu_{i_{1}})$ in Proposition \ref{prop:exponential}. Therefore, $J_{S}(x_{1}^{*}) \geq J_{1,i_{1}}(x_{1}^{*})$. Consequently,
    \begin{align*}
        \dfrac{\exp\{-TJ_{S}(x_{1}^{*})\}\cdot \frac{c_{S}}{\sqrt{T}^{l_{S}}} \cdot \left(1+\mathcal{O}(T^{-1/2})\right)}{\exp\{-TJ_{1,i_{1}}(x_{1}^{*})\}\cdot \frac{c_{1,i_{1}}}{\sqrt{T}}} 
        = & \exp\{-T(J_{S}(x_{1}^{*}) - J_{1,i_{1}}(x_{1}^{*}))\}\cdot \sqrt{T}^{1 - l_{S}} \cdot \mathcal{O}(1) \\
        = & \mathcal{O}(T^{-1/2}),
    \end{align*}
    where the last line follows from $J_{S}(x_{1}^{*}) \geq J_{1,i_{1}}(x_{1}^{*})$ and $l_{S}\geq 2$. As a result,
    \begin{align*}
        1 - P\{CS\} = & \sum_{n=1}^{k-1}(-1)^{n-1}\sum_{S\in\mathcal{S}_{n}}\exp\{-TJ_{S}(x_{1}^{*}) \} \cdot \frac{c_{S}}{\sqrt{T}^{l_{S}}}\cdot \left(1 + \mathcal{O}(T^{-1/2})\right) \\
         & \qquad + \sum_{n=2}^{k-1}(-1)^{n-1}\sum_{\substack{S=\{1,i_{1},\dots,i_{n}\}\\S\notin\mathcal{S}_{n}}}\exp\{-TJ_{1,i_{1}}(x_{1}^{*}) \} \cdot \frac{c_{1,i_{1}}}{\sqrt{T}}\cdot \mathcal{O}(T^{-1/2}) \\
        = & \sum_{n=1}^{k-1}(-1)^{n-1}\sum_{S\in\mathcal{S}_{n}}\exp\{-TJ_{S}(x_{1}^{*}) \} \cdot \frac{c_{S}}{\sqrt{T}^{l_{S}}}\cdot \left(1 + \mathcal{O}(T^{-1/2})\right) \\
        = & \left(\sum_{n=1}^{k-1}(-1)^{n-1}\sum_{S\in\mathcal{S}_{n}}\exp\{-TJ_{S}(x_{1}^{*}) \} \cdot \frac{c_{S}}{\sqrt{T}^{l_{S}}} \right)\cdot \left(1 + \mathcal{O}(T^{-1/2})\right).
    \end{align*}
    \halmos
\end{proof}

For illustrative purpose, consider an instance with $k=3$ alternatives following Gaussian distribution as Example \ref{ex:Gaussian}. In this case, we have the following identity
\begin{equation}\label{eq:ex-inclusion}
    1-P\{CS\} \equiv P\left(\bar{X}_{1}(T_{1}) \leq \bar{X}_{2}(T_{2})\right) + P\left(\bar{X}_{1}(T_{1}) \leq \bar{X}_{3}(T_{3})\right) - P\left(\bar{X}_{1}(T_{1}) \leq \bar{X}_{2}(T_{2}) \wedge \bar{X}_{3}(T_{3})\right).
\end{equation}
By definition, $\bm{x}_{1,2,3} = (x_{1}^{*}, x_{2}^{*}, x_{3}^{*})$ minimizes $$I_{1,2,3}(\bm{x}_{1,2,3}) = p_{1}\cdot\frac{1}{2}\cdot\frac{(x_{1}-m_{1})^{2}}{\sigma_{1}^{2}} + p_{2}\cdot\frac{1}{2}\cdot\frac{(x_{2}-m_{2})^{2}}{\sigma_{2}^{2}} + p_{3}\cdot\frac{1}{2}\cdot\frac{(x_{3}-m_{3})^{2}}{\sigma_{3}^{2}}$$
over $\{\bm{x}_{1,2,3}\in\mathbb{R}^{3}: x_{1}\leq x_{2}\wedge x_{3}\}$. Simple calculation yields that
\begin{equation*}
    (x_{1}^{*}, x_{2}^{*}, x_{3}^{*}) = \begin{cases}
        (x_{1}^{*}, x_{1}^{*}, x_{1}^{*}) & \text{ if }\frac{\sigma_{3}^{2}}{p_{3}}(m_{1} - m_{2}) > \frac{\sigma_{1}^{2}}{p_{1}}(m_{2} - m_{3}), \\
        (\mu_{3}, m_{2}, \mu_{3}) & \text{ otherwise, }
    \end{cases}
\end{equation*}
where $\mu_{3}$ is as in Proposition \ref{prop:exponential} and
$$x_{1}^{*} = \left(\frac{p_{1}}{\sigma_{1}^{2}} + \frac{p_{2}}{\sigma_{2}^{2}} + \frac{p_{3}}{\sigma_{3}^{2}}\right)^{-1}\left(\frac{p_{1}m_{1}}{\sigma_{1}^{2}} + \frac{p_{2}m_{2}}{\sigma_{2}^{2}} + \frac{p_{3}m_{3}}{\sigma_{3}^{2}}\right).$$
Furthermore, we have
\begin{equation*}
    I_{1,2,3}(x_{1}^{*}, x_{2}^{*}, x_{3}^{*}) = \begin{cases}
        \frac{1}{2}\cdot\frac{ \frac{\sigma_{1}^{2}}{p_{1}}(m_{2} - m_{3})^{2} + \frac{\sigma_{2}^{2}}{p_{2}}(m_{3} - m_{1})^{2} + \frac{\sigma_{3}^{2}}{p_{3}}(m_{1} - m_{2})^{2} }{ \frac{\sigma_{1}^{2}}{p_{1}}\frac{\sigma_{2}^{2}}{p_{2}} + \frac{\sigma_{2}^{2}}{p_{2}}\frac{\sigma_{3}^{2}}{p_{3}} + \frac{\sigma_{3}^{2}}{p_{3}}\frac{\sigma_{1}^{2}}{p_{1}} } & \text{ if }\frac{\sigma_{3}^{2}}{p_{3}}(m_{1} - m_{2}) > \frac{\sigma_{1}^{2}}{p_{1}}(m_{2} - m_{3}), \\
        \frac{1}{2} \cdot \frac{(m_{1} - m_{3})^{2}}{ \frac{\sigma_{1}^{2}}{p_{1}} + \frac{\sigma_{3}^{2}}{p_{3}} } & \text{ otherwise, }
    \end{cases}
\end{equation*}
and
\begin{equation*}
    (\bar{\lambda}_{1}^{*}, \bar{\lambda}_{2}^{*}, \bar{\lambda}_{3}^{*}) = \begin{cases}
        \left(\frac{(x^{*} - m_{1})^{2}}{\sigma_{1}^{2}}, \frac{(x^{*} - m_{2})^{2}}{\sigma_{2}^{2}}, \frac{(x^{*} - m_{3})^{2}}{\sigma_{3}^{2}}\right) & \text{ if }\frac{\sigma_{3}^{2}}{p_{3}}(m_{1} - m_{2}) > \frac{\sigma_{1}^{2}}{p_{1}}(m_{2} - m_{3}), \\
        \left(\frac{p_{3}(m_{3} - m_{1})}{p_{1}\sigma_{3}^{2}+p_{3}\sigma_{1}^{2}}, 0, \frac{p_{1}(m_{1} - m_{3})}{p_{1}\sigma_{3}^{2}+p_{3}\sigma_{1}^{2}}\right) & \text{ otherwise. }
    \end{cases}
\end{equation*}
If $\frac{\sigma_{3}^{2}}{p_{3}}(m_{1} - m_{2}) > \frac{\sigma_{1}^{2}}{p_{1}}(m_{2} - m_{3})$, then we have
\begin{equation*}
     P\left(\bar{X}_{1}(T_{1}) \leq \bar{X}_{2}(T_{2}) \wedge \bar{X}_{3}(T_{3})\right) = \exp\{-TI_{1,2,3}(x_{1}^{*}, x_{2}^{*}, x_{3}^{*})\} \cdot \frac{1}{\bar{\lambda}_{2}^{*}\bar{\lambda}_{3}^{*}p_{2}p_{3}\sqrt{\operatorname{det}(\Sigma)}T} (1 + o(1)). 
\end{equation*}
Otherwise, we have 
\begin{equation*}
    P\left(\bar{X}_{1}(T_{1}) \leq \bar{X}_{2}(T_{2}) \wedge \bar{X}_{3}(T_{3})\right) = \exp\{-TI_{1,3}(\mu_{3},\mu_{3})\} \cdot \frac{1}{2\sqrt{2\pi}\lambda_{3}^{*}p_{3}\tilde{\sigma}_{1,3}\sqrt{T}}(1 + o(1)).
\end{equation*}

After some calculations, we have the following approximation of $1 - P\{CS\}$
$$\begin{cases}
    \begin{aligned}
        &\frac{1}{\sqrt{2\pi T}}\exp\left\{ -\frac{1}{2}T R_{2}(p_{1}, p_{2}) - \frac{1}{2}\ln{R_{2}(p_{1}, p_{2})} \right\} \\
        & \qquad + \frac{1}{\sqrt{2\pi T}}\exp\left\{ -\frac{1}{2}T R_{3}(p_{1}, p_{3}) - \frac{1}{2}\ln{R_{3}(p_{1}, p_{3})} \right\}
    \end{aligned}, & \text{ if }\frac{\sigma_{3}^{2}}{p_{3}}(m_{1} - m_{2}) > \frac{\sigma_{1}^{2}}{p_{1}}(m_{2} - m_{3}), \\
    \begin{aligned}
        &\frac{1}{\sqrt{2\pi T}}\exp\left\{ -\frac{1}{2}T R_{2}(p_{1}, p_{2}) - \frac{1}{2}\ln{R_{2}(p_{1}, p_{2})} \right\} \\
        & \qquad + \frac{1/2}{\sqrt{2\pi T}}\exp\left\{ -\frac{1}{2}T R_{3}(p_{1}, p_{3}) - \frac{1}{2}\ln{R_{3}(p_{1}, p_{3})} \right\}
    \end{aligned}, & \text{ otherwise. }
\end{cases}$$

As demonstrated in the equation above, our approximation to PICS exhibits discontinuities in sampling ratios $\bm{p}$ along the \textit{critical line} determined by the inequality: $\frac{\sigma_{3}^{2}}{p_{3}}(m_{1} - m_{2}) > \frac{\sigma_{1}^{2}}{p_{1}}(m_{2} - m_{3})$. Compared to Theorem \ref{thm:main}, the new approximation takes into account the probability of SIBC and has a potential to tackle the so-called monotonicity issue in the low-confidence scenario. To be specific, for two vectors of sampling ratios near the critical line but on opposite sides, the one with a larger $p_{1}$ undergoes a sharp decrease in the approximate PCS. Therefore, our approximation may suggest decreasing $p_{1}$ in order to improve the PCS. Moreover, the critical line $\frac{\sigma_{3}^{2}}{p_{3}}(m_{1} - m_{2}) > \frac{\sigma_{1}^{2}}{p_{1}}(m_{2} - m_{3})$ may serve as a quantitative criterion for distinguishing low- and high-confidence scenarios.

However, it is worth mentioning that although Theorem \ref{thm:extended} provides a closed-form formula, it is numerically difficult to evaluate the approximation. For one thing, in general, there are exponentially many terms in the approximation as the number of alternative increases. For another, calculating the constant $c_{S}$ involving the evaluation of the expectation of a high-dimensional Gaussian random variable with a non-trivial variance matrix. Even though we can evaluate the approximation efficiently in special cases such as Gaussian sampling distributions, the approximation as a function of sampling ratios $\bm{p}$ is no longer convex and is even not continuous, making it difficult to develop OCBA-type policies based on gradients of the approximation.

\begin{algorithm}[t]
\caption{Refined FCBA(0) for Low-Confidence Scenarios (LC-FCBA(0))}
\label{alg:modified}
\begin{algorithmic}[1]
\STATE Initialize: Set budget $T$, number of alternatives $k$, initial replications $T_{0}$.
\STATE Initialize replications $T_{i} \leftarrow T_{0}$ for $i = 1, 2, \dots, k$.
\STATE Simulate $T_{i}$ replications for each alternative $i$.
\STATE Total cost $N \leftarrow k\times T_{0}$.
\WHILE{Total cost $< T$}
    \STATE Estimate the best alternative $j^{*} \leftarrow \argmax_{j\in[k]}\hat{m}_{j}$.
    \STATE Estimate the second best alternative $j^{**} \leftarrow \argmax_{j\in[k]\backslash\{j^{*}\}}\hat{m}_{j}$.
    \STATE Use plug-in estimation of (\ref{eq:definerj}) based on sample means and sample variances to calculate \\ 
    $j' \leftarrow \argmax_{j\in[k]\backslash\{j^{*}\}} U'_{\ell}(\hat{R}_{j}(T_{j^{*}}/N, T_{j}/N))\cdot \hat{R}_{j}(T_{j^{*}}/N, T_{j}/N) \cdot \dfrac{\hat{\sigma}_{j}^{2}/T_{j}^{2}}{\hat{\sigma}_{j^{*}}^{2}/T_{j^{*}}+\hat{\sigma}_{j}^{2}/T_{j}}.$
    \IF{$T_{j^{*}}^{2}/\hat{\sigma}_{j^{*}}^{2} > \sum_{j\in[k]\backslash\{j^{*}\}}T_{j}^{2}/\hat{\sigma}_{j}^{2}$ \textbf{or} for any $j\in[k]\backslash\{j^{*}, j^{**}\}$
    $$\frac{\hat{\sigma}_{j}^{2}}{p_{j}}(\hat{m}_{j^{*}} - \hat{m}_{j^{**}}) > \frac{\hat{\sigma}_{j^{*}}^{2}}{p_{j^{*}}}(\hat{m}_{j^{**}} - \hat{m}_{j}).$$}
    \STATE Simulate one additional replication for alternative $j^{\prime}$.
    \STATE Update $T_{j^{\prime}} \leftarrow T_{j^{\prime}} + 1$.
    \ELSE
    \STATE Simulate one additional replication for alternative $j^{*}$.
    \STATE Update $T_{j^{*}} \leftarrow T_{j^{*}} + 1$.
    \ENDIF
    \STATE Update $N\leftarrow N+1$.
\ENDWHILE
\STATE Select best-performing alternative based on sample means.
\end{algorithmic}
\end{algorithm}

Nevertheless, we can intuitively take advantage of Lemma \ref{lem:critical} and Proposition \ref{prop:anyterm} to develop an efficient R\&S policy for low-confidence scenarios. Roughly speaking, Theorem \ref{thm:extended} refines Theorem \ref{thm:main} by characterizing the probability of SIBC and indicating whether such a probability is significantly large based on the order of $\sqrt{T}^{-1/2}$ in the approximation of Proposition \ref{prop:anyterm}. Although the asymptotic equivalent approximation per se is hard to evaluate, the significance of the probability of SIBC can be easily assessed using the critical point. Recall that $l_{S} = n - s$ where $s = \#\{i\in S\backslash\{1\}: m_{j} \geq x_{1}^{*}\}$. Therefore, a small critical point $x_{1}^{*}$ implies that $s$ is large and $l_{S}$ is small, which further suggests a high probability of SIBC and is indicative of the occurrence of a low-confidence scenario based on the current allocation. Under Gaussian sampling distributions, we introduce a novel refined FCBA(0) for low-confidence scenarios (LC-FCBA(0)), which simply refines FCBA(0) with a judgment of the low-confidence scenario by the critical points.

To be specific, line 9 in Algorithm \ref{alg:modified} judges whether $m_{j}$ is larger than or equal to the critical points using plug-in estimation for $j\in[k]\backslash\{j^{*}, j^{**}\}$, where $j^{*}$ and $j^{**}$ are the estimates of the best and the second best alternative, respectively. If the inequality holds for any $j$, then Algorithm \ref{alg:modified} simply stops allocating samples to the estimated best, preventing the non-monotonicity issue in the low-confidence scenario. Intuitively, the second best alternative roughly has the largest mean except for the best alternative, which will lead to a significant probability of SIBC when the critical point for three alternatives exceeds its mean. It could be reasonable to take the probability of SIBC for more than three alternatives into account when there are a few alternatives, which may cause high computational cost though if the number of alternatives is large. Thanks to Proposition \ref{prop:anyterm}, we can restrict the algorithm to only consider SIBC for three alternatives due to the fact that $J_{S}(x_{1}^{*}) \leq J_{S^{\prime}}(x_{1}^{*})$ if $S\subseteq S^{\prime}$ as shown in the proof of Theorem \ref{thm:extended}. %

\subsubsection{Experiment in Low-Confidence Scenarios}\label{sec:explowcon}
We present an experiment in a low-confidence scenario comparing ROA, OCBA, EA and FCBA(0). We fix $k=10$ alternatives following Gaussian distributions. The means and variances are given by $m_{i} = 0.001 * (k+1-i)$ and $\sigma_{i}^{2} = 1 + \bm{1}\{i\leq 5\}$, for $i=1,2,\dots,k$, respectively. We fix a total sampling budget $T=100$ aside from an initial sample of size $10$ for each alternative. The means are smaller by two orders than those in Section \ref{sec:experiment}, and the average budget per alternative is half as that in Section \ref{sec:experiment}. Moreover, the PCS using EA is only around 13\% at the end, which is slightly higher than the PCS of a random guess, i.e., 10\%. Therefore, we may think of this problem setting as a low-confidence scenario.

\begin{figure}[t]
    \centering
    \includegraphics[width = .45\linewidth]{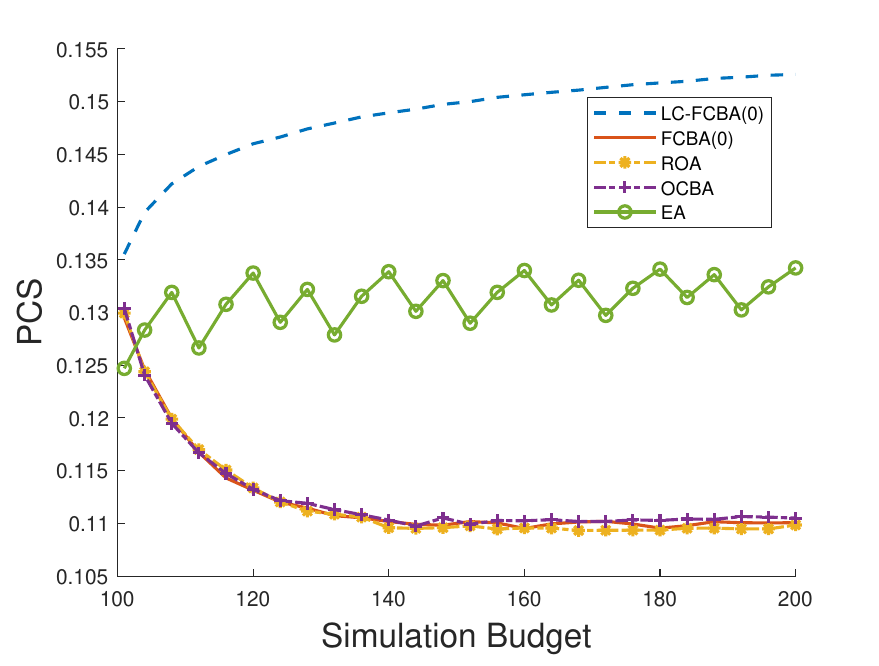}
    \includegraphics[width = .45\linewidth]{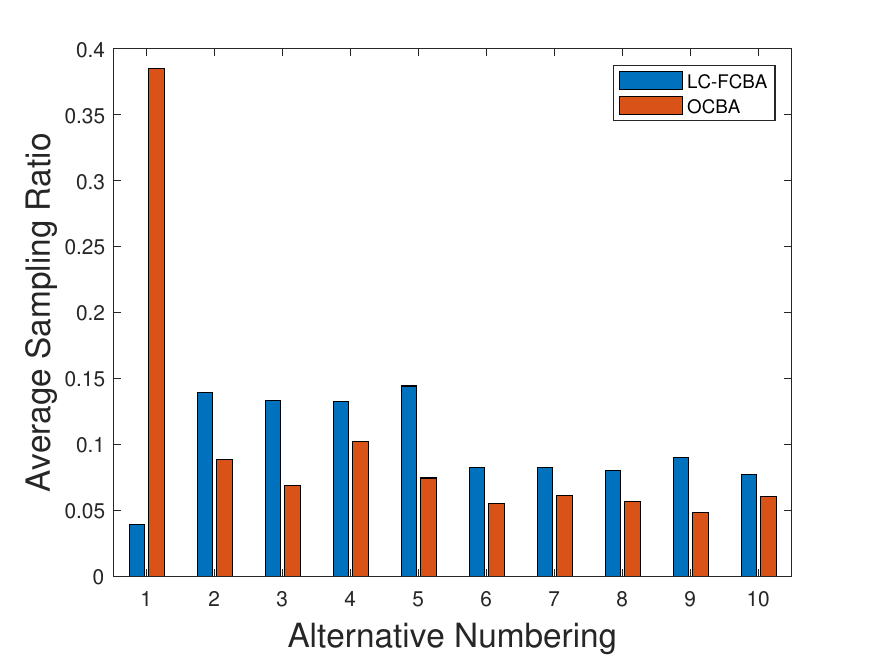}
    \caption{Left: PCS of LC-FCBA(0) versus compared algorithms based on $1,000,000$ macro-replications. Right: Average sample allocation when the best alternative is correctly identified.}
    \label{fig:ex-low}
\end{figure}

The left panel in Figure \ref{fig:ex-low} demonstrates the ability of LC-FCBA(0) tackling the non-monotonicity issue. While all other algorithms experience a decrease in PCS as samples accumulate, LC-FCBA(0) is the only one that shows an increase. The right panel depicts the allocation ratios of LC-FCBA and OCBA. Consistent with our reasoning, LC-FCBA(0) allocates less than 5\% of the budget to the best alternative, while other compared algorithms, such as OCBA, allocate more than one-third of the budget to the best alternative which results from underestimating the probability of SIBC.

\subsection{Conditional Probability}\label{sec:conditionalprob}
When making allocation decision, it is also of interest to take into account the influence of the observed data to the final PCS. To be specific, we will discuss an approximation of the PCS, i.e., $P\{CS\vert \mathcal{F}_{t}\}$, conditioned on observed samples. This approximation has a potential to instruct the sample allocation process dynamically. We shall make a clarification on notations. Now, let $t$ be the total number of samples that have been allocated and $t_{j}$ be the number of samples allocated to $j$. Hence $\sum_{j\in[k]}t_{j} = t$. Let $\mathcal{F}_{t}$ denote the natural filtration generated by observations till time $t$. Henceforth, we will let $T$ and $T_{j}$ be the total number of remaining samples and the number of samples therein to be allocated to sample $j$. We have
\begin{equation*}
    1 - P\{CS\vert \mathcal{F}_{t}\} \approx \sum_{j\in[k]_{-}} P\left(\frac{t_{1}\bar{X}_{1}(t_{1}) + T_{1}\bar{X}_{1}'(T_{1})}{t_{1} + T_{1}} \leq \frac{t_{j}\bar{X}_{j}(t_{j}) + T_{j}\bar{X}_{j}'(T_{j})}{t_{j} + T_{j}}\middle\vert \mathcal{F}_{t}\right),
\end{equation*}
where $\bar{X}_{i}(t_{i})$ is the sample mean of alternative $i$ based on $t$ allocated samples, which can be viewed as constants conditioned on $\mathcal{F}_{t}$, and $\bar{X}'_{i}(T_{i})$ is the sample mean of alternative $i$ based on $T_{i}$ samples to be allocated. This probability represents the ultimate PCS if the best alternative is estimated using sample means based on data both before and after time $t$. The large deviation rate of $P\{CS\vert \mathcal{F}_{t}\}$ is still given by $\inf_{x_{1}\leq x_{j}}I(x_{1},x_{j})$. By the same token, we have
\begin{equation}\label{eq:rescale2}\begin{aligned}
    &P\left(\frac{t_{1}\bar{X}_{1}(t_{1}) + T_{1}\bar{X}_{1}'(T_{1})}{t_{1} + T_{1}} \leq \frac{t_{j}\bar{X}_{j}(t_{j}) + T_{j}\bar{X}_{j}'(T_{j})}{t_{j} + T_{j}}\middle\vert \mathcal{F}_{t}\right) \\
    = &\int{dF_{X_{1}^{(t_{1} + 1)},\dots,X_{1}^{(t_{1} + T_{1})}}(x_{1}^{(t_{1} + 1)},\dots,x_{1}^{(t_{1} + T_{1})}) dF_{X_{j}^{(t_{j} + 1)},\dots,X_{j}^{(t_{j} + T_{j})}}(x_{j}^{(t_{j} + 1)},\dots,x_{j}^{(t_{j} + T_{j})})} \\
    &\quad \cdot \bm{1}\left\{\frac{t_{1}\bar{X}_{1}(t_{1}) + \sum_{\tau = 1}^{T_{1}}x_{1}^{(t_{1}+\tau)}}{t_{1} + T_{1}} \leq \frac{t_{j}\bar{X}_{j}(t_{j}) + \sum_{\tau = 1}^{T_{j}}x_{j}^{(t_{j}+\tau)}}{t_{j} + T_{j}}\right\}.
\end{aligned}\end{equation}
We shall change a little bit on the definition of $Z_{1}^{(\tau)}$'s and $Z_{j}^{(\tau)}$'s, and then the probability above can be rewritten as
\begin{equation}\label{eq:rescale3}\begin{aligned}
    &\exp\left\{-T\cdot\left( p_{1}\tilde{I}_{1}(\Lambda_{1}'(\tilde{\lambda}_{1}^{*})) + p_{j}\tilde{I}_{j}(\Lambda_{j}'(\tilde{\lambda}_{j}^{*}))\right)\right\}\cdot \\
    &\quad \int{dF_{Z_{1}^{(t_{1} + 1)},\dots,Z_{1}^{(t_{1} + T_{1})}}(z_{1}^{(t_{1} + 1)},\dots,z_{1}^{(t_{1} + T_{1})}) dF_{Z_{j}^{(t_{j} + 1)},\dots,Z_{j}^{(t_{j} + T_{j})}}(z_{j}^{(t_{j} + 1)},\dots,z_{j}^{(t_{j} + T_{j})})} \\
    &\quad \cdot \bm{1}\left\{\frac{t_{1}\bar{X}_{1}(t_{1}) + \sum_{\tau = 1}^{T_{1}}z_{1}^{(t_{1}+\tau)}}{t_{1} + T_{1}} \leq \frac{t_{j}\bar{X}_{j}(t_{j}) + \sum_{\tau = 1}^{T_{j}}z_{j}^{(t_{j}+\tau)}}{t_{j} + T_{j}}\right\} \\
    &\quad \cdot\exp\left\{- \lambda_{1}^{*}p_{1}T\left(\frac{1}{t_{1}+T_{1}}\sum_{i=1}^{T_{1}}\left(z_{1}^{(i)} - \Lambda_{1}'(\tilde{\lambda}_{1}^{*})\right) - \frac{1}{t_{j}+T_{j}}\sum_{i=1}^{T_{j}}\left(z_{j}^{(i)} - \Lambda_{j}'(\tilde{\lambda}_{j}^{*})\right)\right)\right\},
\end{aligned}\end{equation}
where $\tilde{\lambda}_{i}^{*} = \lambda_{i}^{*}\frac{T_{i}}{t_{i} + T_{i}}$ and $\tilde{I}_{i}(x) = \tilde{\lambda}_{i}^{*}x - \Lambda_{i}(\tilde{\lambda}_{i}^{*})$. Therein, we have $dF_{Z_{i}}/dF_{X_{i}}(x) = \exp\{-\Lambda_{i}(\tilde{\lambda}_{i}^{*}) + \tilde{\lambda}_{i}^{*}x\}$.

We use the same technique applied to binary comparisons by defining $\tilde{\Gamma}_{j} := \sqrt{T}\left(\sum_{i=1}^{T_{1}}\frac{z_{1}^{(i)} - \Lambda_{1}'(\tilde{\lambda}_{1}^{*})}{t_{1}+T_{1}} - \sum_{i=1}^{T_{j}}\frac{z_{j}^{(i)} - \Lambda_{j}'(\tilde{\lambda}_{j}^{*})}{t_{j}+T_{j}}\right)$. Then, the conditional probability in (\ref{eq:rescale3}) equals
\begin{equation}\label{eq:concentrate3}
    \exp\left\{-T\cdot\left( p_{1}\tilde{I}_{1}(\Lambda_{1}'(\tilde{\lambda}_{1}^{*})) + p_{j}\tilde{I}_{j}(\Lambda_{j}'(\tilde{\lambda}_{j}^{*}))\right)\right\}\times \mathbb{E}\left[\bm{1}\{\tilde{\Gamma}_{j}\leq \Xi_{T}\}\cdot\exp\{\Upsilon_{T}\tilde{\Gamma}_{j}\}\right],
\end{equation}
where $\Xi_{T} = \sqrt{T}(t_{j}\bar{X}_{j}(t_{j})+T_{j}\Lambda_{j}'(\tilde{\lambda}_{j}^{*})) / (t_{j} + T_{j}) - \sqrt{T}(t_{1}\bar{X}_{1}(t_{1})+T_{1}\Lambda_{1}'(\tilde{\lambda}_{1}^{*}))/(t_{1} + T_{1})$ and $\Upsilon_{T} = -\lambda_{1}^{*}p_{1}\sqrt{T} > 0$. Moreover, it follows from the Taylor's expansion of $\Lambda_{j}'(\lambda)$ at $\tilde{\lambda}_{j}^{*}$
\begin{gather*}
    \Lambda_{1}'(\tilde{\lambda}_{1}^{*}) = \mu_{j} - \Lambda_{1}''(\lambda_{1}^{*})\frac{t_{1}}{t_{1} + T_{1}}\lambda_{1}^{*} + \mathcal{O}(T^{-2}), \\
    \Lambda_{j}'(\tilde{\lambda}_{j}^{*}) = \mu_{j} - \Lambda_{j}''(\lambda_{j}^{*})\frac{t_{j}}{t_{j} + T_{j}}\lambda_{j}^{*} + \mathcal{O}(T^{-2}),
\end{gather*}
that 
\begin{equation*}
    \Xi_{T} = \sqrt{T}\left(\frac{t_{j}\bar{X}_{j}(t_{j})}{T_{j}} - \frac{t_{1}\bar{X}_{1}(t_{1})}{T_{1}}\right) + \sqrt{T}\left(\frac{t_{1}}{T_{1}}(\Lambda_{1}''(\lambda_{1}^{*})\lambda_{1}^{*} + \mu_{j}) - \frac{t_{j}}{T_{j}}(\Lambda_{j}''(\lambda_{j}^{*})\lambda_{j}^{*} + \mu_{j}) + \mathcal{O}(T^{-2})\right).
\end{equation*}
Consequently, we have
\begin{equation*}
\begin{aligned}
    &\exp\{(-i\lambda + \Upsilon_{T})\Xi_{T}\} = \left(1 + \mathcal{O}(T^{-1/2})\right)\times \\ 
    &\qquad \exp\left\{t_{j}\lambda_{j}^{*}\left(\bar{X}_{j}(t_{j}) - \mu_{j} - \Lambda_{j}''(\lambda_{j}^{*})\lambda_{j}^{*}\right) + t_{1}\lambda_{1}^{*}\left(\bar{X}_{1}(t_{1}) - \mu_{j} - \Lambda_{1}''(\lambda_{1}^{*})\lambda_{1}^{*}\right)\right\}.
\end{aligned}
\end{equation*}

Denote the exponential term above as $\exp\{A\}$. Following the Parseval's identity and the Taylor's expansion again, we come to the conclusion that the expectation in (\ref{eq:concentrate3}) takes the form
\begin{equation}\label{eq:parseval3}\begin{aligned}
     & \frac{1}{2\pi}\int_{-\infty}^{\infty}{\frac{1}{-i\lambda + \Upsilon_{T}}\exp\{(-i\lambda + \Upsilon_{T})\Xi_{T}\} \times \exp(\Lambda_{\tilde{\Gamma}_{j}}(i\lambda)) d\lambda} \\
    = & \frac{1}{2\pi\cdot\Upsilon_{T}}\int_{-\infty}^{\infty}{\left(1 + \mathcal{O}(T^{-1/2})\right)\cdot \exp\{A\} \cdot \exp\left\{-\frac{1}{2}\sigma_{1,j}^{2}\lambda^{2}\right\} d\lambda}.
\end{aligned}\end{equation}
Therefore, we have 
$$\begin{aligned}
     & 1 - P\{CS\vert\mathcal{F}_{t}\} \approx \sum_{j\in[k]_{-}}\bigg(\frac{1}{\sqrt{2\pi}\cdot \lambda_{j}^{*}p_{j}\sigma_{1,j}\sqrt{T}} \\
    &\qquad \times \exp\left\{-T\cdot\left( p_{1}\tilde{I}_{1}(\Lambda_{1}'(\tilde{\lambda}_{1}^{*})) + p_{j}\tilde{I}_{j}(\Lambda_{j}'(\tilde{\lambda}_{j}^{*}))\right)\right\} \\
    &\qquad \times \exp\left\{\frac{t_{j}}{p_{j}}\left(\bar{X}_{j}(t_{j}) - \mu_{j} - \Lambda_{j}''(\lambda_{j}^{*})\lambda_{j}^{*}\right) - \frac{t_{1}}{p_{1}}\left(\bar{X}_{1}(t_{1}) - \mu_{j} - \Lambda_{1}''(\lambda_{1}^{*})\lambda_{1}^{*}\right)\right\}\bigg).
\end{aligned}$$

In a Gaussian setting, after some tedious calculations, we see that
\begin{equation*}\begin{aligned}
    &1 - P\{CS\vert\mathcal{F}_{t}\} \approx \sum_{j\in[k]_{-}}\bigg(\frac{1}{\sqrt{2\pi T}}\exp\left\{-\frac{1}{2}\ln R_{j}(p_{1}, p_{j})\right\} \\
    &\qquad \times \exp\left\{-\frac{T}{2} R_{j}(p_{1}, p_{j}) \cdot \left(\frac{\frac{\sigma_{1}^{2}}{p_{1}}\cdot\frac{T_{1}^{2}}{(t_{1}+T_{1})^{2}} + \frac{\sigma_{j}^{2}}{p_{j}}\cdot\frac{T_{j}^{2}}{(t_{j}+T_{j})^{2}}}{\frac{\sigma_{1}^{2}}{p_{1}} + \frac{\sigma_{j}^{2}}{p_{j}}}\right)\right\} \\
    &\qquad \times \exp\left\{t_{j}\lambda_{j}^{*}\left(\bar{X}_{j}(t_{j}) + m_{j} - 2\mu_{j}\right) + t_{1}\lambda_{1}^{*}\left(\bar{X}_{1}(t_{1}) + m_{1} - 2\mu_{1}\right)\right\}\bigg).
\end{aligned}\end{equation*}
With mean and variance parameters replaced with sample means and sample variances, the approximation equals to 
\begin{equation*}\begin{aligned}
    &1 - P\{CS\vert\mathcal{F}_{t}\} \approx \sum_{j\in[k]_{-}}\Bigg(\frac{1}{\sqrt{2\pi T}}\exp\left\{-\frac{1}{2}\ln R_{j}(p_{1}, p_{j})\right\} \\
    &\qquad \times \exp\left\{-\frac{T}{2} R_{j}(p_{1}, p_{j}) \cdot \left(\frac{\frac{\sigma_{1}^{2}}{p_{1}}\cdot\frac{T_{1}^{2}}{(t_{1}+T_{1})^{2}} + \frac{\sigma_{j}^{2}}{p_{j}}\cdot\frac{T_{j}^{2}}{(t_{j}+T_{j})^{2}} + 4\frac{\sigma_{1}^{2}\sigma_{j}^{2}}{p_{1}p_{j}}\left(\frac{t_{1}}{T_{1}} + \frac{t_{j}}{T_{j}}\right)}{\frac{\sigma_{1}^{2}}{p_{1}} + \frac{\sigma_{j}^{2}}{p_{j}}} \right)\right\}\Bigg).
\end{aligned}\end{equation*}
Therefore, the above approximation reduces to the former results when $t_{i}=0$. However, if $t_{i}>0$, then the large deviations rates are re-weighted. It is worth mentioning that although our approximation is asymptotically equivalent and provides more information than LDP, the approximation error could be high if $T$ is extremely small. Therefore, in order to design a fully sequential algorithm using the approximate conditional PCS above, we will have to tackle this issue for large time steps, which is deferred to future work.

\section{Conclusions}\label{sec:conclusion}
In summary, our work introduces an advanced Bahadur-Rao type expansion for PCS that significantly enhances the finite sample performance of R\&S algorithms. By extending beyond large deviation principles, this expansion provides a precise characterization of PCS, integrating finite sample behavior that classic asymptotic metrics overlook. Through our framework, the exploration of the optimal sampling ratios and their impact on PCS allows for an efficient allocation of simulation resources, balancing between exploration and exploitation under limited budgets. Our findings show that this approach can address scenarios with low-confidence selection, thereby solving previously observed non-monotonicity in PCS as sampling budgets increase. We also discuss the possibility of leveraging the proposed theory to develop dynamic R\&S policies.

Our novel allocation policy, grounded in the proposed PCS approximation, demonstrates substantial improvements in computational efficiency and accuracy. Particularly, empirical results indicate that even lower-order approximations within our framework enable accurate selection outcomes across varied simulation budgets. By providing a theoretical foundation that accommodates both asymptotic and finite sample behavior, this approach bridges the gap between theoretical and practical applications of PCS, optimizing performance for both large and small budgets without relying on complex dynamic programming methods.

Our novel theory provides a step forward concerning finite budget allocation in the R\&S research. Future research can expand on these findings by applying our framework to non-Gaussian distributions and exploring dynamic adjustments of the sampling ratios in real-time applications. With a versatile foundation and promising results, the proposed approach offers a robust tool for advancing finite sample optimization in simulation-based decision-making, ensuring that PCS remains a reliable metric even as computational constraints shift.

\section*{Acknowledgments.}
This work was supported in part by the National Natural Science Foundation of China (NSFC) under Grants 72325007 and 72250065, and by Xiangjiang Laboratory Key Project under Grant 23XJ02004. This work was also supported in part by Wuhan East Lake High-Tech Development Zone (also known as the Optics Valley of China, or OVC) National Comprehensive Experimental Base for Governance of Intelligent Society.

\bibliographystyle{informs2014}
\bibliography{myrefs.bib}

\begin{APPENDICES}
\section{Proofs of Lemmas and Theorems}
\begin{proof}{Proof of Example \ref{example:Gaussian}.}
    Suppose $X_{i}\sim N(m_{i}, \sigma_{i}^{2})$, $\forall i\in[k]$. The CGF equates to $\Lambda_{i}(\lambda) = m_{i}\lambda + \sigma_{i}^{2}\lambda^{2}/2 < \infty$, $\forall\lambda\in\mathbb{R}$. Moreover, $\Lambda_{i}'(\lambda) = m_{i} + \sigma_{i}^{2}\lambda$ and thus $\mathscr{R}(\Lambda_{i}') = \mathbb{R}$. This justifies Assumption \ref{ass:light}. On the other hand, the density function of a normal random variable is unimodal and bounded, and therefore has bounded total variation. This justifies Assumption \ref{ass:btv}. To partially sum up, the main result holds for the Gaussian case.
    
    For $j\in[k]_{-}$, we see that $I_{j}(x) = \frac{1}{2}\cdot (x-m_{j})^{2} /\sigma_{j}^{2}$ and $\lambda_{j}^{*} = (x- m_{j})/\sigma_{j}^{2}$ takes the maximum in its definition. Recall that $p_{1}I_{1}'(\mu_{j}) + p_{j}I_{j}'(\mu_{j}) = 0$, leading to $\mu_{j} = (p_{1}m_{1}/\sigma_{1}^{2} + p_{j}m_{j}/\sigma_{j}^{2}) / (p_{1}/\sigma_{1}^{2} + p_{j}/\sigma_{j}^{2})$. Immediately, we have $I_{j}(\mu_{j}, \mu_{j}) = \frac{1}{2}(m_{1} - m_{j})^{2} / \left(\sigma_{1}^{2}/p_{1} + \sigma_{j}^{2}/p_{j}\right)$, $\lambda_{1,j}^{*} = p_{j}(m_{j} - m_{1})/(p_{1}\sigma_{j}^{2} + p_{j}\sigma_{1}^{2})$ and $\lambda_{j}^{*} = p_{1}(m_{1} - m_{j})/(p_{1}\sigma_{j}^{2} + p_{j}\sigma_{1}^{2})$. Since Gaussian distribution has any order cumulants, we can expand the series to any term. Using the first order approximation completes the proof. \halmos
\end{proof}

\begin{proof}{Proof of Example \ref{example:exp}.}
    Suppose $X_{i}\sim \operatorname{Exp}(\beta_{i})$, $\forall i\in[k]$. The CDF of $X_{i}$ is $F_{X_{i}}(x) = 1 - \exp(-x/\beta_{i})$. It follows immediately that the CGF of $X_{i}$ is $\Lambda_{i}(\lambda) = -\log(1 - \beta_{i} \lambda)$, $\forall \lambda < \beta_{i}^{-1}$. For $\lambda \geq \beta_{i}^{-1}$, the CGF diverges. Since $\Lambda_{i}'(\lambda) = \beta_{i}/(1 - \beta_{i}\lambda)$ and the domain of $\Lambda_{i}'$ is $(-\infty, \beta_{i}^{-1})$, we see that $\mathscr{R}(\Lambda_{i}') = (0, \infty)$. This justifies Assumption \ref{ass:light}. And Assumption \ref{ass:btv} holds because the density function of an exponential random variable is bounded and monotonic.\\
    \indent For $j\in[k]_{-}$ and $x>0$, we have $I_{j}(x) = \sup_{\lambda<\beta_{i}^{-1}}\lambda x + \log(1 - \beta_{j}\lambda) = \beta_{j}^{-1}x - \log(\beta_{j}^{-1}x) - 1$ and $\lambda_{j}^{*} = \beta_{j}^{-1} - x^{-1}$. If follows from the definition that $p_{1}I_{1}'(\mu_{j}) + p_{j}I_{j}'(\mu_{j}) = p_{1}(\beta_{1}^{-1} - \mu_{j}^{-1}) + p_{j}(\beta_{j}^{-1} - \mu_{j}^{-1}) = 0$. Therefore, $\mu_{j} = (p_{1} + p_{j})/(p_{1}\beta_{1}^{-1} + p_{j}\beta_{j}^{-1})$ and thus 
    $$\begin{aligned}
        I_{j}(\mu_{j}, \mu_{j}) = & p_{1}I_{1}(\mu_{j}) + p_{j}I_{j}(\mu_{j}) = p_{1}\log\frac{p_{1}\beta_{1}^{-1} + p_{j}\beta_{1}^{-1}}{p_{1}\beta_{1}^{-1} + p_{j}\beta_{j}^{-1}} + p_{j}\log\frac{p_{1}\beta_{j}^{-1} + p_{j}\beta_{j}^{-1}}{p_{1}\beta_{1}^{-1} + p_{j}\beta_{j}^{-1}} \\
        = & (p_{1} + p_{j})\log\frac{p_{1}+p_{j}}{p_{1}\beta_{1}^{-1} + p_{j}\beta_{j}^{-1}} - p_{1}\log\beta_{1} - p_{j}\log\beta_{j}.
    \end{aligned}$$
    Moreover, we see that $\lambda_{1,j}^{*} = p_{j}(\beta_{1}^{-1} - \beta_{j}^{-1})/(p_{1} + p_{j})$ and $\lambda_{j}^{*} = p_{1}(\beta_{j}^{-1} - \beta_{1}^{-1})/(p_{1} + p_{j})$. Recall that the variance of $X_{i}$ is $\beta_{i}^{2}$. Hence, $\tilde{\sigma}_{1,j}^{2} = \beta_{1}^{2}/p_{1} + \beta_{j}^{2}/p_{j}$, $\forall j\in[k]_{-}$. It turns out that
    $$\lambda_{j}^{*}p_{j}\tilde{\sigma}_{1,j} = \frac{\beta_{1}/\beta_{j} - 1}{p_{1}^{-1} + p_{j}^{-1}} \cdot \left(p_{1}^{-1} + p_{j}^{-1}\cdot\beta_{j}^{2}/\beta_{1}^{2}\right)^{\frac{1}{2}}.$$
    Using the first order approximation completes the proof. \halmos
\end{proof}

\begin{proof}{Proof of Lemma \ref{lemma:expand}.}
The proof is standard. Suppose $p_{j}>\varepsilon$, $\forall j\in[k]$. By assumption, $Z_{1}-\mu_{j}$ and $Z_{j}-\mu_{j}$ has finite absolute moments up to $q$-th order, $\forall q>3$. And let $\gamma_{1,\nu} := \frac{d^{\nu}}{d\lambda^{\nu}} \left.\log\mathbb{E}[\exp\{\lambda(Z_{1}-\mu_{j})\}]\right\vert_{\lambda=0} = \Lambda_{1}^{(\nu)}(\lambda_{1}^{*})$ and $\gamma_{j,\nu} := \frac{d^{\nu}}{d\lambda^{\nu}} \left.\log\mathbb{E}[\exp\{\lambda(Z_{j}-\mu_{j})\}]\right\vert_{\lambda=0} = \Lambda_{j}^{(\nu)}(\lambda_{j}^{*})$ denote their $\nu$-th cumulants, respectively. Then the CGF of $\tilde{H}_{j}$ is 
\begin{align*}
    \Lambda_{\tilde{H}_{j}}(\lambda) &= \sum_{\nu=2}^{q-1}{\frac{1}{\nu!}T_{1}\gamma_{1,\nu}\left(\frac{\lambda}{T_{1}\sigma_{1,j}}\right)^{\nu} + \frac{1}{\nu!}T_{j}\gamma_{j,\nu}\left(\frac{\lambda}{T_{j}\sigma_{1,j}}\right)^{\nu}} + \mathcal{O}(\vert\lambda\vert^{q}T^{-(q-1)/2}) \\
    &= \frac{1}{2}\lambda^{2} + \sum_{\nu=3}^{q-1}{\frac{1}{\nu!}\left(\frac{\gamma_{1,\nu}}{p_{1}^{\nu-1}} + \frac{\gamma_{j,\nu}}{p_{j}^{\nu-1}}\right)\frac{T^{-(\nu-2)/2}}{\tilde{\sigma}_{1,j}^{\nu}}\lambda^{\nu}} + \mathcal{O}(\vert\lambda\vert^{q}T^{-(q-2)/2}),
\end{align*}
where $\tilde{\sigma}_{1,j} = \sigma_{1,j}\sqrt{T}$. Specifically, we have $\lambda_{1,2} = \sigma_{1}^{2}$ and $\lambda_{j,2} = \sigma_{j}^{2}$ and it follows that the leading term equals $\frac{1}{2}\lambda^{2}$. The moment generating function $\Psi_{\tilde{H}_{j}}(\lambda):= \exp\{\Lambda_{\tilde{H}_{j}}(\lambda)\}$ of $\tilde{H}_{j}$ is consequently expanded by replacing $y$ with $\Lambda_{\tilde{H}_{j}}(\lambda) - \frac{1}{2}\lambda^{2}$ in $e^{y} = \sum_{i=0}^{q-3}\frac{1}{i!}y^{i} + \mathcal{O}(y^{q-2}e^{\vert y\vert})$ as
$$\Psi_{\tilde{H}_{j}}(\lambda)e^{-\lambda^{2}/2} = 1 + \sum_{\nu=1}^{q-3}\frac{P_{\nu,T}(\lambda)}{T^{\nu/2}} + \mathcal{O}((\vert\lambda\vert^{q} + \vert\lambda\vert^{3(q-2)})T^{-(q-2)/2}),$$
which holds for $\vert\lambda\vert \leq \mathcal{O}(\sqrt{T})$. Therein, $P_{\nu,T}(\lambda)$ is a polynomial of order $3\nu$ in $\lambda$. 
Notice that the polynomials $P_{\nu,T}(\lambda)$ depend on $T$ since the sampling ratios $p_{1} = T_{1}/T$ and $p_{j} = T_{j}/T$ may vary as sample size $T$ grows. It can also be checked that $P_{\nu,T}(\lambda)$ is an odd (even) function whenever $\nu$ is odd (even). 

Let $-\frac{d}{dx}$ denote the negation of the derivative operator of analytic functions. By assumption, $F_{Z_{1}}(x)$ and $F_{Z_{j}}(x)$ are absolute continuous with regard to the Lebesgue measure and have bounded total variations. Then the distribution function of $\tilde{H}_{j}$ can be expanded, using the inverse Fourier transformation of the characteristic function $\Psi_{\tilde{H}_{j}}(i\lambda)$, as follows:
\begin{equation}\label{eq:originpoly}
    F_{\tilde{H}_{j}} = \Phi + \sum_{\nu=1}^{q-3}\frac{P_{\nu,T}(-\frac{d}{dx})}{T^{\nu/2}}\Phi + R_{q,T},
\end{equation}
where $\Phi$ is the density function of a standard Gaussian, and $R_{q,n}(x)$ is a remainder function with a uniform bound $\sup_{x}\vert R_{q,T}(x)\vert \leq MT^{-(q-2)/2}$ for some positive number $M$ that is independent of $q$ and $T$, but may depend on the $q$-th cumulants of $Z_{1}$ and $Z_{j}$ and the sampling ratios. The uniform bound of the error term can be justified using Assumption \ref{ass:btv} and the same argument in \cite{cramer1970error} and is omitted. For the general case where the distribution of $Z_{1}$ or $Z_{j}$ includes a discrete or singular part, see Theorem 26 in \cite{cramer1970error}. In this expansion, $P_{\nu,T}(-\frac{d}{dx})\Phi$ denotes an operator $P_{\nu,T}(-\frac{d}{dx})$ functioning on $\Phi$. Each power $\lambda^{\tau}$ in the polynomial $P_{\nu,T}$ is replaced by $(-1)^{\tau}\frac{d^{\tau}}{dx^{\tau}}$ and the expansion can be simplified into
$$F_{\tilde{H}_{j}}(x) = \Phi(x) + \sum_{\nu=1}^{q-3}\frac{p_{3\nu-1,T}(x)}{T^{\nu/2}}e^{-x^{2}/2} + R_{q,T}(x),$$
where $p_{3\nu-1,T}(x)$ is a polynomial in $x$ of order $(3\nu-1)$. 
\halmos
\end{proof}

\begin{proof}{Proof of Proposition \ref{prop:exponential}.}
    Denote $q = 2\ell + 4$. It suffices to approximate (\ref{eq:parseval}) with residual at most $\mathcal{O}(T^{-(q-2)/2})$. Proceeding with $K_{q,T} = \Phi + \sum_{\nu=1}^{q-3}\frac{P_{\nu,T}(-\frac{d}{dx})}{T^{\nu/2}}\Phi$, we have $K_{q,T}^{'}(x) = \phi(x) + \sum_{\nu=1}^{q-3}\frac{P_{\nu,T}(-\frac{d}{dx})}{T^{\nu/2}}\phi(x)$ where $\phi(x) = \Phi'(x)$ since $\Phi$ is apparently analytical. It follows from the linearity of the Fourier transform operator $\mathcal{F}$ by $\mathcal{F}\left(-\frac{d}{dx}\right)^{n}\phi(\lambda) = (i\lambda)^{n}\mathcal{F}\phi(\lambda) = (i\lambda)^{n}e^{-\lambda^{2}/2}$ for all $n\in\mathbb{N}$ that
    \begin{equation*}
        \mathcal{F}K_{q,T}'(\lambda) = \left(1 + \sum_{\nu=1}^{q-3}\frac{P_{\nu,T}(i\lambda)}{T^{\nu/2}}\right)e^{-\lambda^{2}/2}.
    \end{equation*}
    For another term, for $q>3$, it follows from $(1 + x)^{-1} = 1 + \sum_{\nu = 1}^{q-4}(-x)^{\nu} + \mathcal{O}(\vert x\vert^{q-3})$ that 
    \begin{equation*}
        \left(1 + \frac{i\lambda}{\lambda_{j}^{*}p_{j}\tilde{\sigma}_{1,j}\sqrt{T}}\right)^{-1} = 1 + \sum_{\nu=1}^{q-4}\left(\frac{1}{i\lambda_{j}^{*}p_{j}\tilde{\sigma}_{1,j}}\right)^{\nu}\frac{\lambda^{\nu}}{T^{\nu/2}} + \mathcal{O}(T^{-(q-3)/2}).
    \end{equation*}
    
    Now, we conclude from (\ref{eq:concentrate}), (\ref{eq:parseval}) and the last two equations from the preceding paragraph that 
    \begin{equation}\label{eq:integralremainder}
        P\left(\bar{X}_{1}(T_{1}) \leq \bar{X}_{j}(T_{j})\right) = \exp\{-TI_{j}(\mu_{j}, \mu_{j})\} \cdot \frac{1}{\sqrt{2\pi}\cdot\lambda_{j}^{*}p_{j}\tilde{\sigma}_{1,j}\sqrt{T}} \cdot \left(\int_{-\infty}^{\infty}{Q_{q,T}(\lambda)d\Phi(\lambda)} + \mathcal{O}_{q}(T^{-(q-2)/2})\right),
    \end{equation}
    where 
    \begin{equation*}
        Q_{q,T}(\lambda) = \sum_{0\leq r + s< k-2}P_{r,T}(i\lambda)\cdot \frac{\lambda^{s}}{\left(i\lambda_{j}^{*}p_{j}\tilde{\sigma}_{1,j}\right)^{s}} \cdot T^{-(r+s)/2}.
    \end{equation*}
    For convenience, $P_{0,T}\equiv 1$. Recall that $P_{r,T}(i\lambda)$ involves only odd (even) power of $\lambda$ whenever $r$ is odd (even). Hence, it follows from $\int_{-\infty}^{\infty}\lambda^{i}d\Phi(\lambda) = 0$ for any odd number $i\geq 1$ that $\int_{-\infty}^{\infty}{Q_{q,T}(\lambda)d\Phi(\lambda)} = \int_{-\infty}^{\infty}{\tilde{Q}_{q,T}(\lambda)d\Phi(\lambda)}$ where
    \begin{equation*}
        \tilde{Q}_{q,T}(\lambda) = \sum_{1\leq i <q/2-1}\left(\sum_{r + s = 2i}P_{r,T}(i\lambda)\cdot \frac{\lambda^{s}}{\left(i\lambda_{j}^{*}p_{j}\tilde{\sigma}_{1,j}\right)^{s}}\right)T^{-i}.
    \end{equation*}
    Consequently, defining $c_{j,l} = \sum_{r + s = 2l}\int_{-\infty}^{\infty}{P_{r,T}(i\lambda)\cdot \frac{\lambda^{s}}{\left(i\lambda_{j}^{*}p_{j}\tilde{\sigma}_{1,j}\right)^{s}}d\Phi(\lambda)}$ completes the proof. 
    \halmos
\end{proof}

\begin{proof}{Proof of Proposition \ref{prop:exp-trivariate}.}
    The first part is direct from the more general Proposition \ref{prop:anyterm} to which we will therefore limit ourselves to prove for space reason. Now, we show the second part. Consider two cases: (i) $I_{j}(\mu_{j}, \mu_{j}) < I_{i}(\mu_{i}, \mu_{i})$; (ii) $I_{i}(\mu_{i}, \mu_{i}) \leq I_{j}(\mu_{j}, \mu_{j})$. In case (i), notice that $I_{i}(\mu_{i}, \mu_{i}) = I_{i,j}(\mu_{i}, \mu_{i}, m_{j})$. It is obvious that
    \begin{equation*}
        \begin{aligned}
            I_{i}(\mu_{i}, \mu_{i}) = & p_{1}I_{1}(\mu_{i}) + p_{i}I_{i}(\mu_{i}) + p_{j}I_{j}(m_{j}) = \inf_{\bm{x}: x_{1}\leq x_{i}}p_{1}I_{1}(x_{1}) + p_{i}I_{i}(x_{i}) + p_{j}I_{j}(x_{j}) \\
            \leq & \inf_{\bm{x}: x_{1}\leq x_{i}\wedge x_{j}}p_{1}I_{1}(x_{1}) + p_{i}I_{i}(x_{i}) + p_{j}I_{j}(x_{j}) = I_{i,j}(x_{1}^{*}, x_{i}^{*}, x_{j}^{*}).
        \end{aligned}
    \end{equation*}
    This in turn implies that $I_{j}(\mu_{j}, \mu_{j}) < I_{i,j}(x_{1}^{*}, x_{i}^{*}, x_{j}^{*})$. 
    
    In case (ii), we again notice that $I_{i}(\mu_{i}, \mu_{i}) = I_{i,j}(\mu_{i}, \mu_{i}, m_{j}) \leq I_{i,j}(x_{1}^{*}, x_{i}^{*}, x_{j}^{*})$. Actually, we claim that a strict inequality holds. Otherwise, assume that the equality holds and hence both $(\mu_{i}, \mu_{i}, m_{j})$ and $(x_{1}^{*}, x_{i}^{*}, x_{j}^{*})$ minimizes $I_{i,j}(x_{1}, x_{i}, x_{j})$ on $\{(x_{1}, x_{i}, x_{j}): x_{1}\leq x_{i}\}$. Denote $\Delta = (\delta_{1}, \delta_{i}, \delta_{j}) := (x_{1}^{*}, x_{i}^{*}, x_{j}^{*}) - (\mu_{i}, \mu_{i}, m_{j})$. Note that $\mu_{i} > m_{i} > m_{j}$ while $x_{1}^{*} \leq x_{j}^{*}$, hence $\Delta \neq \bm{0}$. Now that $I_{i,j}(x_{1}, x_{i}, x_{j})$ is convex on $\mathbb{R}^{3}$, then for any $t\in[0, 1]$,
    \begin{equation*}
        \begin{aligned}
            I_{i,j}(\mu_{i}, \mu_{i}, m_{j}) \leq & I_{i,j}(\mu_{i} + t\delta_{1}, \mu_{i} + t\delta_{i}, m_{j} + t\delta_{j}) \\
            \leq & I_{i,j}(\mu_{i}, \mu_{i}, m_{j}) + t(I_{i,j}(x_{1}^{*}, x_{i}^{*}, x_{j}^{*}) - I_{i,j}(\mu_{i}, \mu_{i}, m_{j})) \\
            = & I_{i,j}(\mu_{i}, \mu_{i}, m_{j}).
        \end{aligned}
    \end{equation*}
    It follows immediately that the equality holds throughout. Taking derivative with respect to $t$ evaluated at $t = 0$, we see that
    $$p_{1}\delta_{1} I_{1}'(\mu_{i}) + p_{i}\delta_{i} I_{i}'(\mu_{i}) = 0.$$
    Also recall that
    $$p_{1} I_{1}'(\mu_{i}) + p_{i} I_{i}'(\mu_{i}) = 0.$$
    Since $[I_{1}'(\mu_{i})\ I_{i}'(\mu_{i})] = [\lambda_{1,i}^{*}\ \lambda_{i}^{*}] \neq 0$ and $p_{1}, p_{i} > 0$, it follows that $\delta_{1} = \delta_{i}$. Therefore, $x_{1}^{*} = \mu_{i} + \delta_{1} = \mu_{i} + \delta_{i} = x_{i}^{*} \equiv x_{1}^{*}\vee m_{i}$. Moreover, it implies $x_{1}^{*} \geq m_{i} > m_{j}$ and thus $x_{j}^{*} \equiv x_{1}^{*}\vee m_{j} = x_{1}^{*} = x_{i}^{*}$. Now, we have $I_{j}(x_{j}^{*}) > 0$. Therefore,
    \begin{equation*}
        \begin{aligned}
            I_{i,j}(x_{1}^{*}, x_{i}^{*}, x_{j}^{*}) = & I_{i,j}(x_{1}^{*}, x_{1}^{*}, x_{1}^{*}) \\
            = & p_{1}I_{1}(x_{1}^{*}) + p_{i}I_{i}(x_{1}^{*}) + p_{j}I_{j}(x_{1}^{*}) \\
            > & p_{1}I_{1}(x_{1}^{*}) + p_{i}I_{i}(x_{1}^{*}) \\
            \geq & p_{1}I_{1}(\mu_{i}) + p_{i}I_{i}(\mu_{i}) = I_{i,j}(\mu_{i}, \mu_{i}, m_{j}),
        \end{aligned}
    \end{equation*}
    a contradiction to $I_{i,j}(x_{1}^{*}, x_{i}^{*}, x_{j}^{*}) = I_{i,j}(\mu_{i}, \mu_{i}, m_{j})$, which completes the proof. \halmos
\end{proof}

\begin{proof}{Proof of Lemma \ref{lem:critical}.}
    Fix any $i\in S$. It follows from \cite{glynn2004large} that $I_{i}(\cdot)$ is a strict convex function on $[m_{k}, m_{1}]$ and $I_{i}(x_{i})\geq I_{i}(m_{i}) = 0$, $\forall x_{i}\in[m_{k}, m_{1}]$. Actually, the inequality attains an equality if and only if $x_{i}=m_{i}$. Otherwise, let $x_{0} \neq m_{i}$ be another zero of $I_{i}(\cdot)$. Then, for $0<\lambda<1$, it follows the convexity that $0\leq I_{i}(\lambda x_{0} + (1-\lambda)m_{i})\leq \lambda f(x_{0}) + (1-\lambda) f(m_{i}) = 0$. Therefore, $I_{i}(x_{i}) = 0$ on a non-degenerate interval, which contradicts the strict convexity. As a result, it follows from the convexity and the fact that $m_{i}$ is a minimum of $I_{i}(\cdot)$ that $I_{i}(\cdot)$ must be strictly monotonically decreasing in $[m_{k}, m_{i}]$ and strictly monotonically increasing in $[m_{i}, m_{1}]$.

    Furthermore, $I_{S}(\bm{x}_{S}) = \sum_{i\in S}p_{i}I_{i}(x_{i})$ is a strictly convex function of $\bm{x}_{S}$ by definition. Note that the feasible region $\{\bm{x}_{S} \in \mathbb{R}^{\vert S \vert}: x_{1} \leq x_{j},~\forall j \in S\} = \bigcap_{j\in S}\{\bm{x}_{S} \in \mathbb{R}^{\vert S \vert}: x_{1} \leq x_{j}\}$ is a convex set. As a result, there exists a unique minimum $\bm{x}_{S}^{*}$ over $\{\bm{x}_{S} \in \mathbb{R}^{\vert S \vert}: x_{1} \leq x_{j},~\forall j \in S\}$. Now it remains to show the last statement.

    Fix any $j\in S$ with $j\neq 1$. By definition, $x_{j}^{*} \geq x_{1}^{*}$. For another, consider $\tilde{\bm{x}}_{S}:=(\tilde{x}_{j})_{j\in S}\in\mathbb{R}^{\vert S\vert}$ with $\tilde{x}_{i}=x_{i}^{*}$ for $i\neq j, i\in S$ and $\tilde{x}_{j} = m_{j}$. If $x_{j}^{*} < m_{j}$, then it follows from the monotonicity of $I_{j}(\cdot)$ on $[m_{k}, m_{j}]$ that $I_{S}(\tilde{\bm{x}}_{S}) = \sum_{i\in S\backslash\{1\}}{p_{i}I_{i}(\tilde{x}_{i})} + I_{j}(m_{j}) < \sum_{i\in S\backslash\{1\}}{p_{i}I_{i}(x^{*}_{i})} + I_{j}(x_{j}^{*}) = I_{S}(\bm{x}_{S}^{*})$. However, since $\tilde{x}_{i}=x_{i}^{*}\geq x_{1}^{*} =\tilde{x}_{1}$ for $i\neq j, i\in S$ and $\tilde{x}_{j} = m_{j} > x_{j}^{*} \geq x_{1}^{*} =\tilde{x}_{1}$, we see that $\tilde{\bm{x}}_{S}$ is feasible and thus $I_{S}(\bm{x}_{S}^{*}) \leq I_{S}(\tilde{\bm{x}}_{S})$ due to the optimality of $\bm{x}_{S}^{*}$. This leads to a contradiction. Therefore, we have $x_{j}^{*}\geq m_{j}$. To sum up, we have shown that $x_{j}^{*}\geq x_{1}^{*} \vee m_{j}$.

    The next goal is to establish the reverse inequality, $x_{j}^{*}\leq x_{1}^{*} \vee m_{j}$, which completes the argument. In fact, since $x_{1}^{*} \vee m_{j} \geq m_{j}$ and $I_{j}(\cdot)$ is strictly increasing on $[m_{j}, m_{1}]$, one can always decrease $x_{j}^{*}$, if greater than $x_{1}^{*}\vee m_{j}$, to this value and keep other coordinates, such that the new solution turns out a feasible solution with a smaller objective. This contradicts the fact the $\bm{x}_{S}^{*}$ is a minimum and we conclude that $x_{j}^{*}\leq x_{1}^{*} \vee m_{j}$. \halmos
\end{proof}

\begin{proof}{Proof of Proposition \ref{prop:anyterm}.}
    Suppose $n\geq 2$ and $2\leq i_{1} < i_{2} < \dots < i_{n} \leq k$ and let $S = \{1,i_{1},i_{2},\dots,i_{n}\}$. Moreover, let $\Omega_{S}(\bm{T}_{S})$ denote the set of observations that lead to SIBC, i.e.,
    $$\Omega_{S}(\bm{T}_{S}):= \left\{ (x_{i}^{(t)})_{\substack{i\in S\\ 1\leq t\leq T_{i}}} \in \mathbb{R}^{\sum_{i\in S}{T_{i}}}:  \bm{x}:\frac{x_{1}^{(1)}+\cdots+x_{1}^{(T_{1})}}{T_{1}} \leq \frac{x_{i_{1}}^{(1)}+\cdots+x_{i_{1}}^{(T_{i_{1}})}}{T_{i_{1}}} \wedge \cdots \wedge \frac{x_{i_{n}}^{(1)}+\cdots+x_{i_{n}}^{(T_{i_{n}})}}{T_{i_{n}}} \right\},$$
    where $\bm{T}_{S}$ stands for the vector $(T_{i})_{i\in S}$.

    Define $\bar{\lambda}_{i}^{*}:=\argmax_{\lambda}{(\lambda x_{i}^{*} - \Lambda_{i}(\lambda))}$ for $i\in S$. Using the exponential tilting similar to (\ref{eq:rescale}), we can rewrite the focal probability as
    \begin{equation}\label{eq:multintegral}\begin{aligned}
        & P\left(\bar{X}_{1}(T_{1}) \leq \bar{X}_{i_{1}}(T_{i_{1}})\wedge \bar{X}_{i_{2}}(T_{i_{2}})\wedge \dots\wedge \bar{X}_{i_{n}}(T_{i_{n}})\right) = \exp\{-TI_{S}(\bm{x}_{S}^{*})\} \cdot \\
        & \qquad \int_{x_{1}\leq x_{i_{1}}\wedge \cdots \wedge x_{i_{n}}}{dF_{\bar{H}_{1}(T_{1})}(x)dF_{\bar{H}_{i_{1}}(T_{i_{1}})}(x_{i_{1}})\cdots dF_{\bar{H}_{i_{n}}(T_{i_{n}})}(x_{i_{n}})\cdot \exp\left\{-\sum_{i\in S}{\bar{\lambda}_{i}^{*}T_{i}x_{i}}\right\}},
    \end{aligned}\end{equation}
    where $\bar{H}_{\ell}(T_{\ell}) = \frac{1}{T_{\ell}}\sum_{\tau = 1}^{T_{\ell}}{(\bar{Z}_{\ell}^{(\tau)} - x_{\ell}^{*})}$ for $\ell\in S$, and $\{\bar{Z}_{\ell}^{(\tau)}\}_{\tau = 1}^{T_{\ell}}$ is an independent and identically distributed sample from the population $\bar{Z}_{\ell}$ with CDF given by $F_{\bar{Z}_{\ell}}(z) = \int_{-\infty}^{z}{e^{I_{\ell}(x_{\ell}^{*})}e^{\bar{\lambda}_{\ell}^{*}(x - x_{\ell}^{*})}dF_{X_{\ell}}(x)}$ for $\ell\in S$. %

    According to Lemma \ref{lem:critical}, $I_{S}(\bm{x}_{S}^{*}) = J_{S}(x_{1}^{*}) = \inf_{x_{1}}J_{S}(x_{1})$. For one thing, we see that $\exp\{-TI_{S}(\bm{x}_{S}^{*})\} = \exp\{-TJ_{S}(x_{1}^{*})\}$ in (\ref{eq:multintegral}) and it remains to treat the integral. For another, since $x_{1}^{*}$ achieves the infimum of $J_{S}(\cdot)$, it follows that $$\begin{aligned}
        J'(x_{1}^{*}) = & p_{1}I_{1}'(x_{1}^{*}) + \sum_{j\in S\backslash\{1\}}{p_{j}I_{j}'(x_{1}^{*}\vee m_{j})\bm{1}_{[m_{j},\infty)}(x_{1}^{*})} \\
        = & p_{1}\bar{\lambda}_{1}^{*} + \sum_{j\in S\backslash\{1\}}{p_{j}\bar{\lambda}_{j}^{*}\bm{1}_{[m_{j},\infty)}(x_{1}^{*})} \\
        = & p_{1}\bar{\lambda}_{1}^{*} + \sum_{j\in S\backslash\{1\}}{p_{j}\bar{\lambda}_{j}^{*}} \\
        = & 0.
    \end{aligned}$$
    The first equality is immediate from the definition of $J_{S}(\cdot)$. To see the second equality, recall that $I_{j}'(x_{1}^{*}\vee m_{j}) = \frac{\partial}{\partial x} \sup_{\lambda}(\lambda x - \Lambda_{i}(\lambda)) \big\vert_{x=x_{1}^{*}\vee m_{j}} = \bar{\lambda}_{j}^{*}$. Finally, if $x_{1}^{*} \geq m_{j}$, then $\bm{1}_{[m_{j},\infty)}(x_{1}^{*}) = 1$ and $\bar{\lambda}_{j}^{*}\bm{1}_{[m_{j},\infty)}(x_{1}^{*}) = \bar{\lambda}_{j}^{*}$. Otherwise, if $x_{1}^{*} < m_{j}$, then $\bar{\lambda}_{j}^{*} = I_{j}'(x_{1}^{*}\vee m_{j}) = I_{j}'(m_{j}) = 0$ and $\bar{\lambda}_{j}^{*}\bm{1}_{[m_{j},\infty)}(x_{1}^{*}) = 0 = \bar{\lambda}_{j}^{*}$. In either case, the third equality holds. Consequently,
    $$\exp\left\{-\sum_{i\in S}{\bar{\lambda}_{i}^{*}T_{i}x_{i}}\right\} = \exp\left\{-\sum_{j\in S\backslash\{1\}}{\bar{\lambda}_{i}^{*}T_{i}(x_{i} - x_{1})}\right\}.$$
    To be succinct, let $w = \left(\bar{\lambda}_{j}^{*}p_{j}\right)_{j\in S\backslash\{1\}}$ be a coefficient vector and $\tilde{\bar{H}}_{S} := \left(\sqrt{T} (\bar{H}_{j}(T_{j}) - \bar{H}_{1}(T_{1}))\right)_{j\in S\backslash\{1\}}$. Then $\tilde{\bar{H}}_{S}$ has variance matrix
    $$\Sigma = \frac{\sigma_{1}^{2}}{p_{1}} \bm{1}_{n}\bm{1}_{n}^{T} + \Lambda,$$
    where $\bm{1}_{n}$ is an $n$-dimensional vector consisting of all ones and $\Lambda:=\operatorname{diag}(\sigma_{i_{1}}^{2}/p_{i_{1}}, \dots, \sigma_{i_{n}}^{2}/p_{i_{n}})$ is a diagonal matrix.
    Concluding from (\ref{eq:multintegral}) and the discussion above, we see that
    \begin{equation}\label{eq:multiexpectation}\begin{aligned}
        & P\left(\bar{X}_{1}(T_{1}) \leq \bar{X}_{i_{1}}(T_{i_{1}})\wedge \bar{X}_{i_{2}}(T_{i_{2}})\wedge \dots\wedge \bar{X}_{i_{n}}(T_{i_{n}})\right) \\
        & \qquad = \exp\{-TJ_{S}(x_{1}^{*})\} \cdot \mathbb{E}\left[\bm{1}\{0\leq \tilde{\bar{H}}_{S}\}\cdot \exp\left\{-\sqrt{T} \cdot w^{T}\tilde{\bar{H}}_{S}\right\}\right].
    \end{aligned}\end{equation}
    To proceed, we need the following Lemma \ref{lem:multivarexpand}, which is a multivariate version of Lemma \ref{lemma:expand}. The proof is similar with cumulants replaced of joint cumulants and thus is omitted. For the validity of the uniform bound, see Theorem 2(b) and Remark 1.1 in \cite{bhattacharya1978validity}.

    \begin{lemma}\label{lem:multivarexpand}
        Under the conditions of Proposition \ref{prop:anyterm}, with $\tilde{\bar{H}}_{S}$ defined as in the proof thereof, for any integer $q\geq3$, if the sampling ratios are bounded away from zero, the density function $F_{\tilde{\bar{H}}_{S}}(\bm{x})$ has the following expansion
        $$F_{\tilde{\bar{H}}_{S}}(\bm{x}) = \Phi(\Sigma^{-1/2}\bm{x}) + \sum_{\nu=1}^{q-3}{\frac{p_{3\nu-1,T}(\bm{x})}{T^{\nu/2}}}e^{-\bm{x}^{T}\Sigma^{-1}\bm{x}/2} + R_{q,T}(\bm{x}),$$
        where, with a bit abuse of notations, $\Phi$ denotes the distribution function of standard multivariate normal variable, $p_{3\nu-1}$ is a multivariate polynomial of order $(3\nu-1)$ in $\bm{x}$ and $\Vert R_{q,T}\Vert_{\infty}=\mathcal{O}(T^{-(q-2)/2})$. 
    \end{lemma}

    Now we return to the proof of Proposition \ref{prop:anyterm}. It remains to cope with the expectation in (\ref{eq:multiexpectation}). 
    Let $s$ be an integer satisfying $m_{i_{1}} \geq m_{i_{2}} \geq \dots \geq m_{i_{s}} \geq x_{1}^{*} \geq m_{i_{s+1}} \geq \dots \geq m_{i_{n}}$. Recall that, for $i\in\{i_{1},i_{2},\dots,i_{s}\}$, $\bar{\lambda}_{i}^{*} = 0$, and for $i\in\{i_{s+1}, \dots, i_{n}\}$, $\bar{\lambda}_{i}^{*} = I_{i}'(x_{1}^{*}\vee m_{i}) = I_{i}'(x_{1}^{*}) > 0$. Let $\bm{1}_{s}$ and $\bm{1}_{n-s}$ be two vectors with all ones, and $w = (w_{U}^{T}, w_{V}^{T})^{T}$ according to the partition $\{i_{1},i_{2},\dots,i_{s}\}\bigcup\{i_{s+1}, \dots, i_{n}\}$, i.e., $w_{U} = (\bar{\lambda}^{*}_{i_{1}}p_{i_{1}}, \bar{\lambda}^{*}_{i_{2}}p_{i_{2}}, \dots, \bar{\lambda}^{*}_{i_{s}}p_{i_{s}})^{T} = 0$ and $w_{V} = (\bar{\lambda}^{*}_{i_{s+1}}p_{i_{s+1}}, \dots, \bar{\lambda}^{*}_{i_{n}}p_{i_{n}})^{T}$. Similarly, let $\Lambda_{U}$ and $\Lambda_{U}$ be two principal submatrices of $\Lambda$ according to the partition. Then the variance matrix can be decomposed as
    $$\Sigma = \begin{pmatrix}
        \Sigma_{UU} & \Sigma_{UV} \\
        \Sigma_{VU} & \Sigma_{VV}
    \end{pmatrix},$$
    such that $\Sigma_{UU} = \frac{\sigma_{1}^{2}}{p_{1}}\bm{1}_{s}\bm{1}_{s}^{T} + \Lambda_{U}$, $\Sigma_{UV} = \Sigma_{VU}^{T} = \frac{\sigma_{1}^{2}}{p_{1}}\bm{1}_{s}\bm{1}_{n-1}^{T}$, and $\Sigma_{VV} = \frac{\sigma_{1}^{2}}{p_{1}}\bm{1}_{n-s}\bm{1}_{n-1}^{T} + \Lambda_{V}$.  Let $(U^{T}; V^{T})^{T}$ be a centered multivariate normal variable with variance matrix $\Sigma$, such that $\operatorname{Var}(U) = \Sigma_{UU}$, $\operatorname{Cov}(U, V) = \Sigma_{UV} = \Sigma_{VU}^{T}$, and $\operatorname{Var}(V) = \Sigma_{VV}$. 
    We repeat the integration by parts in (\ref{eq:parseval}) as follows and assume that $s = n-s = 1$ for illustration
    \begin{align*}\allowdisplaybreaks
         & \mathbb{E}\left[\bm{1}\{0\leq \tilde{\bar{H}}_{S}\}\cdot \exp\left\{ -\sqrt{T}\cdot w^{T}\tilde{\bar{H}}_{S} \right\}\right] \\
        = & \int_{[0,\infty)^{2}}{\exp\left\{-\sqrt{T}w^{T}(x_{1}, x_{2})^{T}\right\} dF_{\tilde{\bar{H}}_{S}}(x_{1}, x_{2})} \\
        = & \int_{[0,\infty)}{\exp\left\{-\sqrt{T}w_{2}x_{2}\right\} d\left(F_{\tilde{\bar{H}}_{S}}(\infty, x_{2}) - F_{\tilde{\bar{H}}_{S}}(0, x_{2})\right)} \\
        = & \int_{[0,\infty)}{\exp\left\{-\sqrt{T}w_{2}x_{2}\right\} d\left(F_{\tilde{\bar{H}}_{S}}(\infty, x_{2}) - F_{\tilde{\bar{H}}_{S}}(0, x_{2}) - F_{\tilde{\bar{H}}_{S}}(\infty, 0) + F_{\tilde{\bar{H}}_{S}}(0, 0) \right)} \\
        = & \sqrt{T}w_{2}\int_{[0,\infty)}{\exp\left\{-\sqrt{T}w_{2}x_{2}\right\} \left(F_{\tilde{\bar{H}}_{S}}(\infty, x_{2}) - F_{\tilde{\bar{H}}_{S}}(0, x_{2}) - F_{\tilde{\bar{H}}_{S}}(\infty, 0) + F_{\tilde{\bar{H}}_{S}}(0, 0) \right)dx_{2}} \\
        = & \sqrt{T}w_{2}\int_{[0,\infty)}{\exp\left\{-\sqrt{T}w_{2}x_{2}\right\} \left(K_{q,T}(\infty, x_{2}) - K_{q,T}(0, x_{2}) - K_{q,T}(\infty, 0) + K_{q,T}(0, 0) \right)dx_{2}} + \mathcal{O}(T^{-(q-2)/2}) \\
        = & \int_{[0,\infty)}{\exp\left\{-\sqrt{T}w_{2}x_{2}\right\} d\left(K_{q,T}(\infty, x_{2}) - K_{q,T}(0, x_{2}) - K_{q,T}(\infty, 0) + K_{q,T}(0, 0) \right)} + \mathcal{O}(T^{-(q-2)/2}) \\
        = & \int_{[0,\infty)}{\exp\left\{-\sqrt{T}w_{2}x_{2}\right\} d\left(K_{q,T}(\infty, x_{2}) - K_{q,T}(0, x_{2})\right)} + \mathcal{O}(T^{-(q-2)/2}) \\
        = & \int_{[0,\infty)^{2}}{\exp\left\{-\sqrt{T}w^{T}(x_{1}, x_{2})^{T}\right\} dK_{q,T}(x_{1}, x_{2})} + \mathcal{O}(T^{-(q-2)/2}) \\
        = & \int_{[0,\infty)^{n}}{\exp\left\{-\sqrt{T}w^{T}\bm{x}\right\} \frac{\partial^{2}}{\partial x_{i_{1}}\cdots\partial x_{i_{n}}}K_{q,T}(\bm{x})d\bm{x}} + \mathcal{O}(T^{-(q-2)/2}),
    \end{align*}
    where the second equality follows from Fubini's theorem and $w_{1} = \bar{\lambda}_{i_{1}}^{*}p_{i_{1}} = 0$, the fourth follows from the integration by part, the fifth follows from Lemma \ref{lem:multivarexpand} with $K_{q,T}:=F_{\tilde{\bar{H}}_{S}} - R_{q,T}$ redefined for the multivariate case, and the equalities below are valid for the same reason. According to Lemma \ref{lem:multivarexpand}, there exists polynomials $\tilde{p}_{\nu}(\bm{x})$ for $1\leq \nu\leq q-3$ such that
    \begin{align*}
        \frac{\partial^{2}}{\partial x_{i_{1}}\cdots\partial x_{i_{n}}}K_{q,T}(\bm{x}) = (2\pi)^{-\frac{n-1}{2}}\mathrm{det}(\Sigma)^{-\frac{1}{2}}\cdot\exp\{-\bm{x}^{T}\Sigma^{-1}\bm{x}/2\} \cdot \left(1 + \sum_{\nu=1}^{q-3}{\frac{\tilde{p}_{\nu}(\bm{x})}{T^{\nu/2}}}\right),
    \end{align*}
    where the polynomials $\tilde{p}_{\nu}(\bm)$ could be found explicitly by matching terms.
    
    Concluding from the discussion above, we can rewrite
    \begin{align}
         & \mathbb{E}\left[\bm{1}\{0\leq \tilde{\bar{H}}_{S}\}\cdot \exp\left\{ -\sqrt{T}\cdot w^{T}\tilde{\bar{H}}_{S} \right\}\right] \notag \\
        = & \mathbb{E}\left[\bm{1}\{0\leq U, V\}\cdot \exp\{-\sqrt{T}\cdot w^{T}(U^{T}, V^{T})^{T}\} \cdot \left(1 + \sum_{\nu=1}^{q-3}\frac{\tilde{p}_{\nu}(U, V)}{T^{\nu/2}}\right)\right] \notag \\
        = & \mathbb{E}\left[\bm{1}\{0\leq U, V\}\cdot \exp\{-\sqrt{T}\cdot w_{V}^{T}V\}\right] + \sum_{\nu=1}^{q-3}T^{-\nu/2}\mathbb{E}\left[\bm{1}\{0\leq U, V\}\cdot \exp\{-\sqrt{T}\cdot w_{V}^{T}V\} \cdot \tilde{p}_{\nu}(U, V) \right]. \label{eq:summation}
    \end{align}
    The next goal is to approximate each of these expectation. The tower's law of total expectation yields that
    $$\mathbb{E}\left[\bm{1}\{0\leq U, V\}\cdot \exp\left\{ -\sqrt{T}\cdot w_{V}^{T}V \right\}\right] = \mathbb{E}\left[\bm{1}\{0\leq U\}\mathbb{E}\left[\bm{1}\{0\leq V\}\cdot \exp\{-\sqrt{T}\cdot w_{V}^{T}V\}\middle\vert U\right]\right].$$
    The conditional distribution of $V$ given $U$ is $N(\Sigma_{VU}\Sigma_{UU}^{-1}U, \Sigma_{VV}-\Sigma_{VU}\Sigma_{UU}^{-1}\Sigma_{UV})$. Let $\tilde{W}:\mathbb{R}^{n-s}\rightarrow \mathbb{R}$ be given by $\tilde{W}(\bm{x})=e^{-\sqrt{T}\cdot w_{V}^{T}\bm{x}}\cdot \bm{1}_{[0,\infty)^{n-s}}(\bm{x})$, then the Fourier transformation of $\tilde{W}$ is $\mathcal{F}\tilde{W}(\bm{\lambda}) = \prod_{i\in\{i_{s+1},\dots,i_{n}\}}(\sqrt{T}w_{i} - i\lambda_{i})^{-1}$ for $\bm{\lambda} := (\lambda_{i})_{i\in \{i_{s+1},\dots,i_{n}\}} \in \mathbb{R}^{n-1}$. It follows from the Parseval's identity that 
    \begin{align*}\allowdisplaybreaks
         & \mathbb{E}\left[\bm{1}\{0\leq V\}\cdot \exp\{-\sqrt{T}\cdot w_{V}^{T}V\}\middle\vert U\right] \\
        = & \dfrac{1}{ (2\pi)^{-(n-s)}\cdot \prod_{i}{\bar{\lambda}_{i}^{*}p_{i} \cdot T^{(n-s)/2}} } \times \\
         & \qquad \int_{\mathbb{R}^{n-s}}d\bm{\lambda} \cdot \prod_{i}\left(1 + \frac{i\lambda_{i}}{\sqrt{T}\cdot \bar{\lambda}_{i}^{*}p_{i}}\right)^{-1} \cdot \exp\left\{i\bm{\lambda}^{T}\Sigma_{VU}\Sigma_{UU}^{-1}U - \frac{1}{2}\bm{\lambda}^{T}(\Sigma_{VV} - \Sigma_{VU}\Sigma_{UU}^{-1}\Sigma_{UV})\bm{\lambda}\right\}.
    \end{align*}
    It follows from the Lagrange's mean value theorem that, for some $\theta\in(0, 1)$ dependent on $\bm{\lambda}$, we have 
    $$\prod_{i}\left(1 + \frac{i\lambda_{i}}{\sqrt{T}\cdot \bar{\lambda}_{i}^{*}p_{i}}\right)^{-1} = 1 + \theta \left(\sum_{i}\frac{i\lambda_{i}}{\bar{\lambda}_{i}^{*}p_{i}}\right) T^{-\frac{1}{2}}.$$
    Consequently, we see that
    \begin{align*}\allowdisplaybreaks
         & \mathbb{E}\left[\bm{1}\{0\leq V\}\cdot \exp\{-\sqrt{T}\cdot w_{V}^{T}V\}\middle\vert U\right] \\
        = & \dfrac{1}{ (2\pi)^{(n-s)}\cdot \prod_{i}{\bar{\lambda}_{i}^{*}p_{i}} \cdot T^{(n-s)/2} } \times \\
         & \qquad \int_{\mathbb{R}^{n-s}}d\bm{\lambda} \cdot \exp\left\{i\bm{\lambda}^{T}\Sigma_{VU}\Sigma_{UU}^{-1}U - \frac{1}{2}\bm{\lambda}^{T}(\Sigma_{VV} - \Sigma_{VU}\Sigma_{UU}^{-1}\Sigma_{UV})\bm{\lambda}\right\} + \\
         & \dfrac{1}{ (2\pi)^{(n-s)}\cdot \prod_{i}{\bar{\lambda}_{i}^{*}p_{i} \cdot T^{(n-s)/2}} } \times  T^{-\frac{1}{2}} \cdot\\
         & \qquad \int_{\mathbb{R}^{n-s}}d\bm{\lambda} \cdot \theta \left(\sum_{i}\frac{i\lambda_{i}}{\bar{\lambda}_{i}^{*}p_{i}}\right) \cdot \exp\left\{i\bm{\lambda}^{T}\Sigma_{VU}\Sigma_{UU}^{-1}U - \frac{1}{2}\bm{\lambda}^{T}(\Sigma_{VV} - \Sigma_{VU}\Sigma_{UU}^{-1}\Sigma_{UV})\bm{\lambda}\right\}.
    \end{align*}
    On one hand, we have 
    \begin{align*}
         & \int_{\mathbb{R}^{n-s}}d\bm{\lambda} \cdot \exp\left\{i\bm{\lambda}^{T}\Sigma_{VU}\Sigma_{UU}^{-1}U - \frac{1}{2}\bm{\lambda}^{T}(\Sigma_{VV} - \Sigma_{VU}\Sigma_{UU}^{-1}\Sigma_{UV})^{-1}\bm{\lambda}\right\} \\
        = & \dfrac{(2\pi)^{(n-s)/2}}{\mathrm{det}(\Sigma_{VV} - \Sigma_{VU}\Sigma_{UU}^{-1}\Sigma_{UV})^{1/2}}\cdot \exp\left\{-\frac{1}{2}U^{T}\Sigma_{UU}^{-1}\Sigma_{UV}(\Sigma_{VV} - \Sigma_{VU}\Sigma_{UU}^{-1}\Sigma_{UV})^{-1}\Sigma_{VU}\Sigma_{UU}^{-1}U\right\}.
    \end{align*}
    On the other hand, the error term has an upper bound uniform in $U$ as follows
    \begin{align*}
         & \left\Vert \int_{\mathbb{R}^{n-s}}d\bm{\lambda} \cdot \theta \left(\sum_{i}\frac{i\lambda_{i}}{\bar{\lambda}_{i}^{*}p_{i}}\right) \cdot \exp\left\{i\bm{\lambda}^{T}\Sigma_{VU}\Sigma_{UU}^{-1}U - \frac{1}{2}\bm{\lambda}^{T}(\Sigma_{VV} - \Sigma_{VU}\Sigma_{UU}^{-1}\Sigma_{UV})\bm{\lambda}\right\} \right\Vert \\
        \leq & \int_{\mathbb{R}^{n-s}}d\bm{\lambda} \cdot \left(\sum_{i}\frac{\lambda_{i}}{\bar{\lambda}_{i}^{*}p_{i}}\right) \cdot \exp\left\{- \frac{1}{2}\bm{\lambda}^{T}(\Sigma_{VV} - \Sigma_{VU}\Sigma_{UU}^{-1}\Sigma_{UV})\bm{\lambda}\right\} < \infty.
    \end{align*}
    Combining the results in the above paragraph, we partially conclude that
    \begin{align*}
         & \mathbb{E}\left[\bm{1}\{0\leq U, V\}\cdot \exp\left\{ -\sqrt{T}\cdot w_{V}^{T}V \right\}\right] \\
        = & \dfrac{1}{\mathrm{det}(\Sigma_{VV} - \Sigma_{VU}\Sigma_{UU}^{-1}\Sigma_{UV})^{1/2} \cdot \prod_{i}{\bar{\lambda}_{i}^{*}p_{i}}} \cdot T^{-(n-s)/2} \cdot \\
         & \qquad \mathbb{E}\left[ \bm{1}\{0\leq U\}\cdot \exp\left\{-\frac{1}{2}U^{T}\Sigma_{UU}^{-1}\Sigma_{UV}(\Sigma_{VV} - \Sigma_{VU}\Sigma_{UU}^{-1}\Sigma_{UV})^{-1}\Sigma_{VU}\Sigma_{UU}^{-1}U\right\} \right]  + \mathcal{O}\left(T^{-\frac{n-s+1}{2}}\right) \\
        = & \dfrac{1}{\mathrm{det}(\Sigma_{VV} - \Sigma_{VU}\Sigma_{UU}^{-1}\Sigma_{UV})^{1/2} \cdot \prod_{i}{\bar{\lambda}_{i}^{*}p_{i}}} \cdot T^{-(n-s)/2} \cdot \\
         & \qquad \mathbb{E}\left[ \bm{1}\{0\leq U\}\cdot \exp\left\{-\frac{1}{2}U^{T}\Sigma_{UU}^{-1}\Sigma_{UV}(\Sigma_{VV} - \Sigma_{VU}\Sigma_{UU}^{-1}\Sigma_{UV})^{-1}\Sigma_{VU}\Sigma_{UU}^{-1}U\right\} \right] \times \left(1 + \mathcal{O}\left(T^{-\frac{1}{2}}\right)\right).
    \end{align*}.

    In order to complete the first statement of Proposition \ref{prop:anyterm}, it suffices to tackle the second term on the last line in (\ref{eq:summation}). We state without proof, since the argument simply repeats the procedure above, that for any non-negative integers $t, t\geq 0$ and any indices $i\in\{i_{1},i_{2},\dots,i_{s}\}$ and $j\in\{i_{s+1},\dots,i_{n}\}$,
    $$\mathbb{E}\left[\bm{1}\{0\leq U, V\}\cdot \exp\left\{ -\sqrt{T}\cdot w_{V}^{T}V \right\} \cdot U_{i}^{r} V_{j}^{t}\right] = T^{-(n-s)/2} \cdot (C + \mathcal{O}(T^{-\frac{1}{2}})),$$
    for some constant $C\geq 0$. Recall that $\tilde{p}_{\nu}(U, V)$ is a multivariate polynomial jointly in $U$ and $V$, the expectation in the second term in (\ref{eq:summation}) takes the same form as above. Putting pieces together, we conclude that
    $$P\left(\bar{X}_{1}(T_{1}) \leq \bar{X}_{i_{1}}(T_{i_{1}})\wedge \bar{X}_{i_{2}}(T_{i_{2}})\wedge \dots\wedge \bar{X}_{i_{n}}(T_{i_{n}})\right) = \exp\{-TJ_{S}(x_{1}^{*})\} \cdot \frac{c_{S}}{\sqrt{T}^{l_{S}}} \cdot \left(1 + \mathcal{O}(T^{-1/2})\right),$$
    where the constant equals to
    $$\begin{aligned}
        c_{S} = & \dfrac{1}{\mathrm{det}(\Sigma_{VV} - \Sigma_{VU}\Sigma_{UU}^{-1}\Sigma_{UV})^{1/2} \cdot \prod_{i}{\bar{\lambda}_{i}^{*}p_{i}}} \cdot \\
        & \qquad \mathbb{E}\left[ \bm{1}\{0\leq U\}\cdot \exp\left\{-\frac{1}{2}U^{T}\Sigma_{UU}^{-1}\Sigma_{UV}(\Sigma_{VV} - \Sigma_{VU}\Sigma_{UU}^{-1}\Sigma_{UV})^{-1}\Sigma_{VU}\Sigma_{UU}^{-1}U\right\} \right],
    \end{aligned}$$
    and the order $l_{S} = n-s$ equals to the number of alternatives with means smaller than the critical point $x_{1}^{*}$. \halmos
\end{proof}%
\begin{proof}{Proof of Lemma \ref{lem:extrapolation}.}
    Suppose $Q(x) = A(x)/B(x)$ and $R(x) = C(x)/D(x)$, where $A(x)$, $B(x)$, $C(x)$, $D(x)$ are polynomials with real coefficients. Without loss of generality, assume that $A(x)$ and $B(x)$ are relatively prime, and $C(x)$ and $D(x)$ are relatively prime. Since $Q(x)$ and $R(x)$ are well-defined on the interval $(a, b)$, the polynomials $B(x)$ and $D(x)$ must have no zeros on it. 
    
    For $x \in (a, b)$, we have
    $$\exp\left\{\frac{A(x)}{B(x)}\right\} = \exp\{Q(x)\} = R(x) = \frac{C(x)}{D(x)}.$$
    By taking logarithms on both sides, we see that 
    $$\frac{A(x)}{B(x)} = \ln\left(\frac{C(x)}{D(x)}\right).$$
    Then it follows from taking derivatives on both sides that
    \begin{equation}\label{eq:rationalpoly}
        \frac{A'(x)B(x) - A(x)B'(x)}{B^{2}(x)} = \frac{C'(x)D(x) - C(x)D'(x)}{C(x)D(x)}.
    \end{equation}
    By rearranging the terms in (\ref{eq:rationalpoly}), we see that if $C(x)\neq 0$, 
    \begin{equation}\label{eq:bezout}
        \left[(A'(x)B(x) - A(x)B'(x))D(x) + B^{2}(x)D'(x)\right]C(x) - B^{2}(x)C'(x)D(x) = 0.
    \end{equation}
    Since the left-hand side is a polynomial in $x$ and the equality holds for $x\in(a, b)$ except for a finite number of zeros of $C(x)$, we see that (\ref{eq:bezout}) holds for $x\in\mathbb{R}$. 

    We note that the identical equality (\ref{eq:bezout}) does not in turn imply the validity of (\ref{eq:rationalpoly}), because $B(x)$, $C(x)$ and $D(x)$ can have zeros outside $(a, b)$. To finish the proof, we show that $P(x)$ and $Q(x)$ are in fact constants.
    
    It follows that $D(x)$ must be a factor of $(A'(x)B(x) - A(x)B'(x))D(x) + B^{2}(x)D'(x)$ and, consequently, a factor of $B^{2}(x)D'(x)$. Suppose $P(x)$ is a prime factor of $D(x)$ and that $D(x) = P^{k}(x) S(x)$ where $k\geq 1$ and $S(x)$ is a polynomial relatively prime to $p(x)$. Then we have 
    $$D'(x) = kP^{k-1}(x)P'(x) S(x) + P^{k}(x) S'(x).$$ 
    Given the only prime factors in real polynomials are linear functions and quadratic functions that do not have real zeros, it is evident that $P(x)$ and $P'(x)$ are relatively prime. Hence it follows from the divisibility of $D(x)$ by $P^{k}(x)$ that $B^{2}(x)\cdot kP^{k-1}(x)P'(x) S(x)$ is divisible by $P^{k}(x)$. As a result, $P(x)$ must be a factor of $B^{2}(x)$, and consequently, a factor of $B(x)$. This in turn implies that $D(x)$ is a factor of $B(x)D'(x)$. Similarly, it follows from (\ref{eq:bezout}) that $C(x)$ must be a factor of $B^{2}(x)C'(x)$. Using the same argument, we see that each prime factor of $C(x)$ must be a factor of $B(x)$ as well, and thus $C(x)$ is a factor of $B(x)C'(x)$.

    Now, we rewrite (\ref{eq:rationalpoly}) as
    $$\frac{A'(x)B(x) - A(x)B'(x)}{B(x)} = B(x)\frac{C'(x)D(x) - C(x)D'(x)}{C(x)D(x)} = \frac{B(x)C'(x)}{C(x)} - \frac{B(x)D'(x)}{D(x)}.$$
    Since the right-hand side is a polynomial, we see that $A'(x)B(x) - A(x)B'(x)$ is divisible by $B(x)$. Therefore, $B(x)$ is a factor of $A(x)B'(x)$. However, using the same argument as above, we see that each prime factor of $B(x)$ must be a factor of $A(x)$, implying that $B(x)$ has no non-degenerate prime factors since it is relatively prime to $A(x)$.

    We can conclude that $B(x)$ is a constant. It follows that $C(x)$ and $D(x)$ are both constants. Therefore, $A(x)$ must also be a constant, which completes the proof. \halmos
\end{proof} %

\begin{proof}{Proof of Proposition \ref{prop:unique}.}
    We claim that if $\ell$ is even, the optimality condition (\ref{eq:conditions}) admits at most one solution. We argue by contradiction. Suppose $\bm{\bar{p}}:=(\bar{p}_{1},\dots,\bar{p}_{k})$ and $\tilde{\bm{p}}:=(\tilde{p}_{1},\dots,\tilde{p}_{k})$ are two distinct solutions to (\ref{eq:conditions}). It follows from the argument above that $\bm{p}$ and $\tilde{\bm{p}}$ are both optimal solutions to (\ref{opt:opt1}). For any $\theta\in\mathbb{R}$, let $\bm{p}_{\theta}:=(p_{\theta,1},\dots,p_{\theta,k})$ denote the convex combination $\theta\bm{\tilde{p}} + (1-\theta)\bm{\bar{p}}$. It is obvious that $\bm{p}_{\theta}$ is a feasible solution to (\ref{opt:opt1}), i.e., the entries of $\bm{p}_{\theta}$ add up to one, as long as $\bm{p}_{\theta}\geq 0$. Moreover, for $\theta\in(0,1)$, 
    $$V_{\ell}(\bm{\bar{p}}) \leq V_{\ell}(\bm{p}_{\theta}) = V_{\ell}(\theta\bm{\tilde{p}} + (1-\theta)\bm{\bar{p}}) \leq \theta V_{\ell}(\bm{\tilde{p}}) + (1-\theta)V_{\ell}(\bm{\bar{p}}) = V_{\ell}(\bm{\bar{p}}),$$
    where the second inequality follows from Lemma \ref{lem:concavity}. Therefore, $V_{\ell}(\bm{p}_{\theta}) = V_{\ell}(\bm{\bar{p}})$, and thus $\bm{p}_{\theta}$ is an optimal solution to (\ref{opt:opt1}) and satisfies the necessary optimality conditions. Specifically, for $\theta\in(0, 1)$,
    \begin{equation}\label{eq:condition2fort}
        \frac{p_{\theta,1}^{2}}{\sigma_{1}^{2}} = \sum_{j\in[k]_{-}}\frac{p_{\theta,j}^{2}}{\sigma_{j}^{2}}.
    \end{equation}
    By definition, $p_{\theta,i} = \theta\tilde{p}_{i}+(1-\theta)\bar{p}_{i}$ for $i\in[k]$. It follows from the second equality in (\ref{eq:conditions}) and the above equation that
    $$\frac{\bar{p}_{1}\tilde{p}_{1}}{\sigma_{1}^{2}} = \sum_{j\in[k]_{-}}\frac{\bar{p}_{j}\tilde{p}_{j}}{\sigma_{j}^{2}}.$$
    It turns out that, for $\theta\in\mathbb{R}$, equality (\ref{eq:condition2fort}) still holds.

    On the other hand, for $\theta\in(0, 1)$ and $i\neq j\in[k]_{-}$, we have
    \begin{equation}\label{eq:condition1fort}
        U_{\ell}'(R_{i}(p_{\theta,1}, p_{\theta,i}))\cdot \frac{\partial}{\partial p_{i}}R_{i}(p_{\theta,1}, p_{\theta,i}) = U_{\ell}'(R_{j}(p_{\theta,1}, p_{\theta,j}))\cdot \frac{\partial}{\partial p_{j}}R_{j}(p_{\theta,1}, p_{\theta,j}).
    \end{equation}
    With simple calculations, we see that 
    $$U_{\ell}'(x) = \exp\left\{-\frac{1}{2}Tx - \frac{1}{2}\ln{x}\right\} \cdot \left( -\frac{1}{2}T + \frac{1}{2}{\frac{(-1)^{\ell+1}(2\ell+1)!!}{x^{\ell+1}}}\frac{1}{T^{\ell}} \right).$$
    Therefore, equality (\ref{eq:condition1fort}) is equivalent to
    \begin{equation}\label{eq:rearrangement}
        \begin{aligned}
            & \exp\left\{-T(R_{i}(p_{\theta,1}, p_{\theta,i}) - R_{j}(p_{\theta,1}, p_{\theta,j}))\right\} \\
            = & \frac{R_{i}(p_{\theta,1}, p_{\theta,i})}{R_{j}(p_{\theta,1}, p_{\theta,j})} \cdot \frac{\left( -\frac{1}{2}T + \frac{1}{2}{\frac{(-1)^{\ell+1}(2\ell+1)!!}{R_{j}^{\ell+1}(p_{\theta,1}, p_{\theta,j})}}\frac{1}{T^{\ell}} \right)^{2}}{\left( -\frac{1}{2}T + \frac{1}{2}{\frac{(-1)^{\ell+1}(2\ell+1)!!}{R_{i}^{\ell+1}(p_{\theta,1}, p_{\theta,i})}}\frac{1}{T^{\ell}} \right)^{2}} \cdot \frac{\left(\frac{\partial}{\partial p_{j}}R_{j}(p_{\theta,1}, p_{\theta,j})\right)^{2}}{\left(\frac{\partial}{\partial p_{i}}R_{i}(p_{\theta,1}, p_{\theta,i})\right)^{2}}.
        \end{aligned}
    \end{equation}
    Because the exponent on the left-hand side and the term on the right-hand side are both rational polynomials in $\theta$, it follows from Lemma \ref{lem:extrapolation} that the equality above holds for $\theta\in\mathbb{R}$. This in turn implies that (\ref{eq:condition1fort}) holds for $\theta\in\mathbb{R}$. To partially conclude, $\bm{p}_{\theta}$ satisfies the optimality condition (\ref{eq:conditions}). Moreover, if $\bm{p}_{\theta}\geq 0$, then it is an optimal solution to (\ref{opt:opt1}) and thus $V_{\ell}(\bm{p}_{\theta}) = V_{\ell}(\bm{\bar{p}})$. 
    
    However, this is impossible. Define $\theta_{0}:= \max\{\bar{p}_{i}/(\bar{p}_{i} - \tilde{p}_{i}) : \bar{p}_{i} - \tilde{p}_{i}>0, i\in [k]\}$ as the greatest number such that $\bm{p}_{\theta}\geq 0, \forall 0\leq \theta\leq \theta_{0}$. It follows from the fact $\bm{\tilde{p}} \neq \bm{\bar{p}}$ and the constraint $\sum_{i\in[k]}{\bar{p}_{i}} = \sum_{i\in[k]}{\tilde{p}_{i}} = 1$ that there exists $i\in[k]$ such that $\tilde{p}_{i}<\bar{p}_{i}$. Therefore, the set in the definition of $\theta_{0}$ is non-empty and $\theta_{0}$ is well-defined. In addition, there exists $i\in[k]$ such that $p_{\theta_{0},i} = 0$ by definition. Recall that $R_{j}(p_{1}, p_{j}) = (m_{1} - m_{j})^{2}/(\sigma_{1}^{2}/p_{1} + \sigma_{j}^{2}/p_{j}) \rightarrow 0$ as $p_{1}\rightarrow 0$ or as $p_{j}\rightarrow 0$. It follows that $\lim_{\theta \uparrow \theta_{0}}R_{j}(p_{\theta,1}, p_{\theta,j}) = 0$. Moreover, it turns out that $\lim_{\theta \uparrow \theta_{0}}V_{\ell}(\bm{p}_{\theta}) = \infty$ since $\lim_{x\downarrow 0}U_{\ell}(x) = \infty$. This contradicts the equality $V_{\ell}(\bm{p}_{\theta}) = V_{\ell}(\bm{\bar{p}})$ for $1\leq \theta\leq \theta_{0}$ from the above argument. Consequently, the optimality condition (\ref{eq:conditions}) admits at most one solution.
    
    Since the optimization problem (\ref{opt:opt1}) admits at least one feasible solution and the feasible region is bounded, it admits at least one optimal solution satisfying the necessary condition above. Therefore, the unique solution to the necessary conditions must be the unique optimal solution to (\ref{opt:opt1}). \halmos
\end{proof}

To prove Theorem \ref{prop:asympallocation}, we need the following lemma which states the boundedness of $\bm{p}$.
\begin{lemma}\label{lem:theoreticalallocationbound}
    Given $\ell\geq 0$ even, there exists $\varepsilon>0$ such that for any $i\in[k]$,
    $$\varepsilon<\liminf_{T\rightarrow\infty}p_{i}^{(\ell,T)} \leq \limsup_{T\rightarrow\infty}p_{i}^{(\ell,T)} < 1- \varepsilon.$$
\end{lemma}
\begin{proof}{Proof.}
    Suppose $\ell\geq 0$ is even. We will drop $\ell$ from the superscript of $\bm{p}^{(\ell,T)}$ in this proof for simplicity. It follows from the second equation in (\ref{eq:conditions}) that, for any $j\in[k]_{-}$, 
    $$(p_{j}^{(T)})^{2} \leq \sigma_{j}^{2}\cdot\sum_{q\in[k]_{-}}\frac{(p_{q}^{(T)})^{2}}{\sigma_{q}^{2}} = \sigma_{j}^{2}\cdot\frac{(p_{1}^{(T)})^{2}}{\sigma_{1}^{2}}.$$
    Equivalently, we have $p_{j}^{(T)} \leq \frac{\sigma_{j}}{\sigma_{1}}p_{1}^{(T)}$. Then, it follows from the fact $\bm{p}^{(T)}$ is a feasible allocation that
    $$1 = \sum_{i\in[k]}{p_{i}^{(T)}} \leq \frac{\sum_{i\in[k]}{\sigma_{i}}}{\sigma_{1}} p_{1}^{(T)}.$$
    Therefore, $\liminf_{T\rightarrow\infty}p_{1}^{(T)} \geq \sigma_{1}/\sum_{i\in[k]}{\sigma_{i}} > 0$. \\
    \indent Now we claim that, for any $j\in[k]_{-}$, $\liminf_{T\rightarrow\infty}p_{j}^{(T)} > 0$. To see this, note that
    \begin{equation}\label{eq:comparequalallocation}
        U_{\ell}(R_{j}(p_{1}^{(T)}, p_{j}^{(T)})) \leq \sum_{q\in[k]_{-}}U_{\ell}(R_{q}(p_{1}^{(T)}, p_{q}^{(T)})) \leq \sum_{q\in[k]_{-}}U_{\ell}(R_{q}(\frac{1}{k}, \frac{1}{k})) \leq (k-1)U_{\ell}(\min_{q\in[k]_{-}}R_{q}(\frac{1}{k}, \frac{1}{k}))
    \end{equation}
    Consider an auxiliary function $f_{T}(a) = aU_{\ell}(a\min_{q\in[k]_{-}}R_{q}(\frac{1}{k}, \frac{1}{k}))$ for $a\in[\frac{1}{k-1}, \infty)$. Recall that $f_{T}(a)$ is contingent on simulation budget $T$ as $U_{\ell}$ is. After some calculations, we see that
    $$\begin{aligned}
        f'_{T}(a) & = U_{\ell}(a\min_{q\in[k]_{-}}R_{q}(\frac{1}{k}, \frac{1}{k})) + a\min_{q\in[k]_{-}}R_{q}(\frac{1}{k}, \frac{1}{k})U'_{\ell}(a\min_{q\in[k]_{-}}R_{q}(\frac{1}{k}, \frac{1}{k})) \\
        & = \exp\left\{-\frac{1}{2}y - \frac{1}{2}\ln\frac{y}{T}\right\} \cdot \left(-\frac{1}{2}y + 1 + \sum_{l=1}^{\ell - 1}\frac{(-1)^{l}(2l-1)!!}{y^{l}} + \frac{(\frac{1}{2}-\ell)(2\ell-1)!!}{y^{\ell}}\right),
    \end{aligned}$$
    where $y$ is an abbreviation for $T\cdot a\min_{q\in[k]_{-}}R_{q}(\frac{1}{k}, \frac{1}{k})$. Since $a\min_{q\in[k]_{-}}R_{q}(\frac{1}{k}, \frac{1}{k})$ is lower bounded by $\frac{1}{k-1}\min_{q\in[k]_{-}}R_{q}(\frac{1}{k}, \frac{1}{k}) > 0$, we see that $y\rightarrow \infty$ as long as $T\rightarrow\infty$. Then, it follows from the fact that the bracket term above tends to $-\infty$ as $y\rightarrow \infty$ that, there exists some $T_{0}\in \mathbb{N}$,
    \begin{equation*}
        f'_{T}(a) < 0, \quad\forall a\in\left[\frac{1}{k-1}, \infty\right),~\forall T\geq T_{0}. 
    \end{equation*}
    Consequently, for any $T\geq T_{0}$, $f_{T}(1) \leq f_{T}(\frac{1}{k-1})$. Equivalently, 
    \begin{equation}\label{ineq:upperbound}
        (k-1)U_{\ell}(\min_{q\in[k]_{-}}R_{q}(\frac{1}{k}, \frac{1}{k})) \leq U_{\ell}(\frac{1}{k-1}\min_{q\in[k]_{-}}R_{q}(\frac{1}{k}, \frac{1}{k})).
    \end{equation}
    It follows from the monotonicity of $U_{\ell}(x)$ by (\ref{eq:comparequalallocation}) and (\ref{ineq:upperbound}) that, for any $T\geq T_{0}$, 
    $$\frac{1}{k-1}\min_{q\in[k]_{-}}R_{q}(\frac{1}{k}, \frac{1}{k}) \leq R_{j}(p_{1}^{(T)}, p_{j}^{(T)}) = \frac{(m_{1} - m_{j})^{2}}{\sigma_{1}^{2}/p_{1}^{(T)} + \sigma_{j}^{2}/p_{j}^{(T)}} \leq \frac{(m_{1} - m_{j})^{2}}{\sigma_{1}^{2} + \sigma_{j}^{2}/p_{j}^{(T)}}.$$
    By rearranging the terms, we see that
    $$p_{j}^{(T)} \geq \left[\left(\frac{1}{k-1}\min_{q\in[k]_{-}}R_{q}(\frac{1}{k}, \frac{1}{k})\right)^{-1}(m_{1} - m_{j})^{2} - \sigma_{1}^{2} \right]^{-1} \sigma_{j}^{2} > 0.$$
    Therefore, for $j\in[k]_{-}$, $\liminf_{T\rightarrow\infty}p_{j}^{(T)} > 0.$
    \indent To conclude, for all $i\in[k]$, $\liminf_{T\rightarrow\infty}p_{i}^{(T)} > 0.$ Hence it completes the proof by choosing $\varepsilon = \frac{1}{2}\min_{i\in[k]}\liminf_{T\rightarrow\infty}p_{i}^{(T)} > 0$.
    \halmos
\end{proof}

Then Theorem \ref{prop:asympallocation} follows.
\begin{proof}{Proof of Theorem \ref{prop:asympallocation}.}
    Suppose $\ell\geq 0$ is even. We will drop $\ell$ from the superscript of $\bm{p}^{(\ell,T)}$ in this proof for simplicity. Let $\{\bm{p^{(T_{m})}}\}_{m\geq 1}$ be a subsequence of $\{\bm{p^{(T)}}\}_{T\geq 1}$. Since $\{\bm{p^{(T_{m})}}\}_{m\geq 1}\subseteq \mathbb{R}^{k}$ is bounded, there exists a further convergent subsequence $\{\bm{p^{(T_{m_{n}})}}\}_{n\geq 1}$. For simplicity, we will denote $\bm{p}^{(n)}$ in short for $\bm{p}^{(T_{m_{n}})}$. Now, it remains to show that $\lim_{n\rightarrow\infty}\bm{p}^{(n)} = \bm{p}^{*}$. \\
    \indent According to Lemma \ref{lem:theoreticalallocationbound}, there exists $\varepsilon>0$ such that $\varepsilon<p_{i}^{(n)}<1-\varepsilon$ holds for $n$ sufficiently large, $\forall i\in[k]$. Since $R_{j}(\cdot,\cdot)$ are continuous functions on $[\varepsilon, 1-\varepsilon]^{2}\subseteq\mathbb{R}^{2}$, there exist two reals $0<L<U<\infty$ such that $L\leq R_{j}(p_{1}^{(n)}, p_{j}^{(n)})\leq U$, $\forall n\geq 1$, $\forall j\in[k]_{-}$. It follows similarly from the proof of (\ref{eq:rearrangement}) that
    \begin{equation*}
        \begin{aligned}
            & \exp\left\{-T_{m_{n}}(R_{i}(p^{(n)}_{1}, p^{(n)}_{i}) - R_{j}(p^{(n)}_{1}, p^{(n)}_{j}))\right\} \\
            = & \frac{R_{i}(p^{(n)}_{1}, p^{(n)}_{i})}{R_{j}(p^{(n)}_{1}, p^{(n)}_{j})} \cdot \frac{\left( \frac{1}{2}T_{m_{n}} + \frac{1}{2}{\frac{(2\ell+1)!!}{R_{j}^{\ell+1}(p^{(n)}_{1}, p^{(n)}_{j})}}\frac{1}{T_{m_{n}}^{\ell}} \right)^{2}}{\left( \frac{1}{2}T_{m_{n}} + \frac{1}{2}{\frac{(2\ell+1)!!}{R_{i}^{\ell+1}(p^{(n)}_{1}, p^{(n)}_{i})}}\frac{1}{T_{m_{n}}^{\ell}} \right)^{2}} \cdot \frac{\left(\frac{\partial}{\partial p_{j}}R_{j}(p^{(n)}_{1}, p^{(n)}_{j})\right)^{2}}{\left(\frac{\partial}{\partial p_{i}}R_{i}(p^{(n)}_{1}, p^{(n)}_{i})\right)^{2}} \\
            = &\frac{R_{j}(p^{(n)}_{1}, p^{(n)}_{j})}{R_{i}(p^{(n)}_{1}, p^{(n)}_{i})} \cdot \frac{\left( \frac{1}{2}T_{m_{n}} + \frac{1}{2}{\frac{(2\ell+1)!!}{R_{j}^{\ell+1}(p^{(n)}_{1}, p^{(n)}_{j})}}\frac{1}{T_{m_{n}}^{\ell}} \right)^{2}}{\left( \frac{1}{2}T_{m_{n}} + \frac{1}{2}{\frac{(2\ell+1)!!}{R_{i}^{\ell+1}(p^{(n)}_{1}, p^{(n)}_{i})}}\frac{1}{T_{m_{n}}^{\ell}} \right)^{2}} \cdot \frac{\left(\dfrac{\sigma_{j}^{2}/(p_{j}^{(n)})^{2}}{\sigma_{1}^{2}/p_{1}^{(n)} + \sigma_{j}^{2}/p_{j}^{(n)}}\right)^{2}}{\left(\dfrac{\sigma_{i}^{2}/(p_{i}^{(n)})^{2}}{\sigma_{1}^{2}/p_{1}^{(n)} + \sigma_{i}^{2}/p_{i}^{(n)}}\right)^{2}} \\
            \leq & \frac{U}{L}\cdot \frac{\left( \frac{1}{2}T_{m_{n}} + \frac{1}{2}{\frac{(2\ell+1)!!}{L^{\ell+1}}}\frac{1}{T_{m_{n}}^{\ell}} \right)^{2}}{\left( \frac{1}{2}T_{m_{n}} + \frac{1}{2}{\frac{(2\ell+1)!!}{U^{\ell+1}}}\frac{1}{T_{m_{n}}^{\ell}} \right)^{2}} \cdot \frac{\sigma_{j}^{4}(\sigma_{1}^{2} + \sigma_{i}^{2})}{\sigma_{i}^{4}(\sigma_{1}^{2} + \sigma_{j}^{2})}\frac{(1-\varepsilon)^{6}}{\varepsilon^{6}}.
        \end{aligned}
    \end{equation*}
    The right-hand side is evidently a rational polynomial in $T_{m_{n}}$. Hence, it follows by taking logarithms on both sides of this inequality and dividing both sides by $-T_{m_{n}}$ that
    $$\limsup_{n\rightarrow \infty}R_{i}(p^{(n)}_{1}, p^{(n)}_{i}) - R_{j}(p^{(n)}_{1}, p^{(n)}_{j}) \leq 0.$$
    Similarly, we have
    $$\liminf_{n\rightarrow \infty}R_{i}(p^{(n)}_{1}, p^{(n)}_{i}) - R_{j}(p^{(n)}_{1}, p^{(n)}_{j}) \geq 0.$$
    These two inequalities combine to yield that
    $$\lim_{n\rightarrow \infty}R_{i}(p^{(n)}_{1}, p^{(n)}_{i}) - R_{j}(p^{(n)}_{1}, p^{(n)}_{j}) = 0.$$
    Therefore, since $R_{i}(\cdot, \cdot)$ and $R_{j}(\cdot, \cdot)$ are continuous, we see that
    $$R_{i}(\lim_{n\rightarrow \infty}p^{(n)}_{1}, \lim_{n\rightarrow \infty}p^{(n)}_{i}) = R_{j}(\lim_{n\rightarrow \infty}p^{(n)}_{1}, \lim_{n\rightarrow \infty}p^{(n)}_{j}).$$
    \indent On the other hand, it follows naturally from the definition of $\bm{p}^{(n)}$ that
    $$\frac{\lim_{n\rightarrow \infty}(p^{(n)}_{1})^{2}}{\sigma_{1}^{2}} = \sum_{j\in[k]_{-}}\frac{\lim_{n\rightarrow \infty}(p^{(n)}_{j})^{2}}{\sigma_{j}^{2}}.$$
    Since the system (\ref{eq:roacondition}) admits a unique solution $\bm{p}^{*}$, we conclude that
    $$\lim_{n\rightarrow\infty}p_{i}^{(n)} = p_{i}^{*},$$
    which completes the proof.
    \halmos 
\end{proof}
\end{APPENDICES}

\end{document}